\documentclass[11pt,letterpaper]{article}

\usepackage[english]{babel}

\usepackage[top=1in,bottom=1in,left=1in,right=1in]{geometry}
\usepackage{amsmath}
\usepackage{graphicx}
\usepackage[colorlinks=true, allcolors=blue]{hyperref}
\usepackage{authblk}

\usepackage{mathtools}
\usepackage[singlelinecheck=off]{caption}
\usepackage{floatrow}

\newlength\Myfigwd

\usepackage[utf8]{inputenc} 
\usepackage[T1]{fontenc}    
\usepackage{hyperref}       
\usepackage{url}            
\usepackage{booktabs}       
\usepackage{amsfonts}       
\usepackage{nicefrac}       
\usepackage{microtype}      
\usepackage{xcolor}         
\usepackage{amsthm}
\usepackage{comment}

\usepackage{microtype}
\usepackage{graphicx}
\usepackage{subfigure}
\usepackage{booktabs} 
\usepackage{amsfonts}
\usepackage{amsmath}

\DeclareMathOperator*{\argmin}{arg\,min}
\usepackage{amssymb}

\usepackage{euscript}
\usepackage[colorlinks=true, allcolors=blue]{hyperref}

\usepackage{mathtools}
\usepackage{multicol}
\usepackage{booktabs}
\newcommand{\vertiii}[1]{{\left\vert\kern-0.25ex\left\vert\kern-0.25ex\left\vert #1 
    \right\vert\kern-0.25ex\right\vert\kern-0.25ex\right\vert}}

\usepackage{makecell}

\theoremstyle{plain}
\newtheorem{theorem}{Theorem}[section]
\newtheorem{proposition}[theorem]{Proposition}
\newtheorem{lemma}[theorem]{Lemma}

\theoremstyle{definition}
\newtheorem{definition}[theorem]{Definition}
\newtheorem{assumption}[theorem]{Assumption}

\DeclareUnicodeCharacter{2113}{\ensuremath{\ell}}

\title{Meta Learning for High-dimensional Ising Model Selection \\ Using $\ell_1$-regularized Logistic Regression}

\author[1]{Huiming Xie}
\author[1, 2]{Jean Honorio}

\affil[1]{Department of Statistics, Purdue University}
\affil[2]{Department of Computer Science, Purdue University}

\date{}
\begin{document}
\maketitle

\begin{abstract}
In this paper, we consider the meta learning problem for estimating the graphs associated with high-dimensional Ising models, using the method of $\ell_1$-regularized logistic regression for neighborhood selection of each node. Our goal is to use the information learned from the auxiliary tasks in the learning of the novel task to reduce its sufficient sample complexity. To this end, we propose a novel generative model as well as an improper estimation method. In our setting, all the tasks are \emph{similar} in their \emph{random} model parameters and supports. By pooling all the samples from the auxiliary tasks to \emph{improperly} estimate a single parameter vector, we can recover the true support union, assumed small in size, with a high probability with a sufficient sample complexity of $\Omega(1) $ per task, for $K = \Omega(d^3 \log p ) $ tasks of Ising models with $p$ nodes and a maximum neighborhood size $d$. Then, with the support for the novel task restricted to the estimated support union, we prove that consistent neighborhood selection for the novel task can be obtained with a reduced sufficient sample complexity of $\Omega(d^3 \log d)$. 
\end{abstract}

\section{Introduction}

Markov random fields (MRF) are an important class of probability models that find its applications in a wide variety of fields spanning statistical physics \cite{ising1925beitrag} , social network analysis \cite{snell1980markov}, computer vision \cite{geman1993stochastic} and natural language processing \cite{manning1999foundations}. A Markov random field is an undirected graph where each node represents a random variable, and the graph structure carries certain assumptions about the conditional independence of these random variables. A prototypical example of Markov random field is the \textbf{Ising model}, where the random variables are discrete, and in particular, binary. Several other types of Markov random fields can be viewed as a general setting of the Ising model \cite{snell1980markov}. Detecting statistical dependencies, which boils down to estimating the graph structure in the Ising model, is therefore a fundamentally significant problem to solve. Many efforts have been made, among which Ravikumar et al. \cite{ravikumar2010high} proved that with relatively low computational complexity, consistent model selection can be achieved with a sample size of $n = \Omega(d^3 \log p)$ for a graph of $p$ nodes with a maximum neighborhood size $d$ by neighborhood selection for each node using $\ell_1$-regularized logistic regression. This simple method with a theoretically supported reasonable performance has received considerable attention. 

In practice, however, one may not be able to obtain as many samples for the high-dimensional settings where both $p$ and $d$ can be large. The more common situation in reality is that one has only a few samples for a task; nonetheless, there are usually many tasks (each with few samples) of a similar kind. In this regime, the learning machine is embedded in an environment of related tasks \cite{thrun1998learning, nichol2018first}. This challenge is also commonly met in the application of other machine learning algorithms \cite{vilalta2002perspective, finn2017model, vanschoren2018meta}. One general way to tackle this kind of difficulty is through \textbf{meta learning} \cite{vanschoren2018meta}, or \emph{learning to learn} \cite{lake2015human}, where we learn from  multiple
learning episodes that oftentimes cover a distribution of related tasks --- a \emph{family} of tasks, in order to gain some experience for the learning of a novel task (in that family), in the hope of reducing the sample complexity for the latter \cite{hospedales2020meta}. 

Meta learning has been widely used in machine learning problems to help increase sample efficiency, but a majority of prior works are experimental in nature, without theoretical guarantees \cite{lemke2015metalearning, finn2017model,hospedales2020meta}. For Markov random fields, there have been some theoretical results for the models associated with Gaussian or sub-Gaussian random vectors, which is a continuous type of Markov random field \cite{zhang2021meta}. For the discrete type, or Ising model as its prototype, efforts have been made for proving theoretical sample complexity for a single task using different methods. In addition to the work of Ravikumar et al. \cite{ravikumar2010high} that we previously mentioned, Santhanam and Wainwright \cite{santhanam2012information} derived
information-theoretic lower bounds on Ising model selection showing that any method will fail at least half the time with sample size $n=\mathcal{O}(d^2 \log p)$. Yang et al. \cite{yang2015graphical} treated the Ising model as a special case of the univariate exponential family distribution model, and showed a sufficient sample complexity of $n = \Omega(d^2  (\log p)^3)$.

There has been some work on the multi-task learning problem on the Ising model, but mainly experimental without theoretical guarantees \cite{gonccalves2015multi}. Also note that while both have many tasks, multi-task learning is learning several models for the different tasks simultaneously, which is intrinsically different than the method and goal in meta learning where we learn a single model from different tasks for the easier learning of a novel task. The challenging situation of having many tasks but only a few samples per task would also render the multi-task learning method meaningless. Sihag and Tajer \cite{sihag2019structure} and Varici \cite{varici2021learning} considered the model selection of two partially identical Ising models with a known subset of nodes that have identical structures in both graphs and showed that the sufficient sample complexity is $\Omega(d^3 \log p)$  for a bounded degree $d$ for estimating the intersection of edges. A generalization of partially identical Ising model pairs was considered by Guo et al. \cite{guo2015estimating} to include multiple graphs. However, this kind of setting for joint graph learning is overly specific and fixed for the meta learning on Ising models we are referring to, especially in the similarity assumptions. Indeed, there has not been any effort in the theory of meta learning for discrete Markov random field, or particularly the Ising model, the building block of the discrete type.

To the best of our knowledge, we are the first to theoretically prove the sufficient sample complexity for the meta learning problem of the Ising model selection. For the Ising model, we follow the practice of Ravikumar et al. \cite{ravikumar2010high} in using $\ell_1$-regularized logistic regression and converting the model selection problem into one of neighborhood selection. Based on this method, for the meta learning problem on Ising models, we propose a novel generative model by introducing \emph{randomness} to the parameter vectors of the Ising models with reasonable and flexible assumptions for similarity among different tasks, which serves as a good generalization of the metadata for the Ising model, and hence our theoretical results can be applied to a wide range of distributions for the parameters in the Ising models under some mild conditions. We also propose an \emph{improper estimation} method in the meta learning problem for Ising model selection where we pool all the samples from the auxiliary tasks together to estimate a single \emph{common parameter vector} (see Definition \ref{family}), rather than estimate an individual parameter vector for each auxiliary task, and then we recover the \emph{support union} from the single \emph{common parameter vector}. Next, we estimate the parameter vector for the novel task by restricting  its support in the estimated \emph{support union}.  We have successfully shown that to learn the support union of all the tasks requires each task to have a sample complexity of only $\Omega(1)$ as long as we have many auxiliary tasks in the sense that number of tasks $K= \Omega ( d^3\log p )$, and in fact the more tasks the better for solving this problem. Here $d$ is the maximum neighborhood size in the graphs of the Ising models for all tasks, and $p$ is the number of nodes in a graph,  which is also the same for all tasks. Furthermore, We have proved that with the knowledge from the auxiliary tasks, learning on a novel task for  neighborhood selection requires a sufficient sample complexity of only  $\Omega(d^3 \log d)$, independent of $p$, which is far less than that without restriction,  $\Omega(d^3 \log p)$ \cite{ravikumar2010high}, since we have assumed sparse graphical models in which $d \ll p$. In this sense, we have a smaller sufficient sample complexity for both auxiliary and novel tasks than those in all previous works, which have a magnitude depending on $p$, particularly in the form of $\log p$.
\section{Preliminaries}\label{prelim}
This section provides the background on the Ising model selection. The notations to be used throughout the paper is summarized in Table \ref{notation} in Appendix \ref{sumnot}.

\subsection{Ising Model and Model Selection}

In this paper, we focus on the Ising model, i.e., the binary pairwise Markov random fields, for which we provide a definition here.

\begin{definition}\label{isingdef}
    Let $X = (X_1, X_2, ..., X_p) \in \{-1,+1\}^p$ denote a $p$-dimensional binary random vector, and each random variable  $X_s$ is associated with a vertex $s \in V$ of an undirected graph $G$ with vertex set  $V=\{1,...,p\}$ and edge set $E = V \times V.$ 
    $X$ is said to form an Ising model associated with $G$ if the distribution takes the form
    \begin{equation}
     \mathbb{P}_{\theta^*}(x) = \frac{1}{Z(\theta^*)} \exp \big\{\sum_{(s,t) \in E}\theta^*_{st} x_s x_t\big\},
    \end{equation}
    for some parameter $\theta^*_{st} \in \mathbb{R}$, where the partition function $Z(\theta^*)$ ensures that the distribution sums to one.
\end{definition}

A typical graphical model selection aims to infer the edge set $E$. A stronger criterion is  \emph{signed edge recovery}, i.e., to infer the edge sign vector 
\begin{equation}
    E_{\pm} :=\begin{cases}
            \text{sign}(\theta^*_{st}), & \text{if $(s,t) \in E$,}\\
            0, & \text{otherwise.}
              \end{cases} 
\end{equation}
\subsection{Neighborhood-based Logistic Regression}
The signed edge recovery is equivalent to recovering, for each vertex $r\in V$, its neighborhood set $\mathcal{N}(r):=\{t \in V|(r,t) \in E\}$, together with the correct signs sign$(\theta^*_{rt})$ for all $t\in \mathcal{N}(r)$ \cite{ravikumar2010high}. Such information can be summarized as the \emph{signed neighborhood set}
\begin{equation}
    \mathcal{N}_{\pm}(r):=\{\text{sign}(\theta^*_{rt})t|t \in \mathcal{N}(r)\},
\end{equation}
which can be recovered naturally from the sign-sparsity pattern of the $(p-1)$-dimensional sub-vector of parameters $$\theta^*_{\setminus r} := \{\theta^*_{ru}, u \in V \setminus r \}$$ associated with each vertex $r$. To estimate $\theta^*_{\setminus r}$, we can make use of the easily derived conditional distribution of $X_r$ given the other variables $X_{\setminus r} = \{X_t| t\in V \setminus r\}$: 
\begin{equation}
    \mathbb{P}_{\theta^*}(x_r|x_{\setminus r}) = \frac{\exp(2x_r\sum_{t\in V \setminus r}\theta^*_{rt}x_t)}{\exp(2x_r\sum_{t\in V \setminus r}\theta^*_{rt}x_t)+1}.
\end{equation}
The problem can thus be viewed as a logistic regression where $X_r$ is the response variable and all other variables $X_{\setminus r}$ act as the covariates. For a sparse problem, it is also natural to use the $\ell_1$-regularizer \cite{tibshirani1996regression}. Formally stated, given $\mathfrak{X}_1^n = \{x^{(1)}, x^{(2)}, ..., x^{(n)}\}$, a set of $n$ i.i.d. samples,  the regularized regression problem is a convex program of the form  \cite{ravikumar2010high}
\begin{equation}\label{isingreg}
    \min_{\theta_{\setminus r} \ \in \mathbb{R}^{p-1}} \ell (\theta_{\setminus r};\mathfrak{X}_1^n) + \lambda \|\theta_{\setminus r}\|_1 , 
\end{equation}
where $\ell (\theta_{\setminus r};\mathfrak{X}_1^n):= - \frac{1}{n}\sum^n_{i=1} \log \mathbb{P}_{\theta_{\setminus r}}(x^{(i)}_r|x^{(i)}_{\setminus r}) $ is the re-scaled negative log likelihood and $\lambda$ is a regularization parameter to be specified by the user.
We can then use the estimate $\hat{\theta}_{\setminus r}$ to estimate the signed neighborhood according to 
\begin{equation}
    \mathcal{\hat{N}}_{\pm}(r):=\{\text{sign}(\hat{\theta}_{ru})u|u \in V \setminus r, \hat{\theta}_{ru} \neq 0 \}.
\end{equation}
\section{Our Novel Generative Model and Improper Estimation Method}\label{method}
In this section, we introduce our novel generative model as well as our novel improper estimation method for the meta learning problem on Ising models.
\subsection{Our Novel Generative Model for Meta Learning on Ising Models}

We consider multiple Ising models whose parameters are generated randomly, which is more reasonable and flexible than the deterministic settings in previous works.

Formally, we define the following family of Ising models with random parameters:

\begin{definition}\label{family}
     Let $X_1^{(k)}, X_2^{(k)}, ..., X_{n^{(k)}}^{(k)} \in \{-1,+1\}^p$ be $n^{(k)}$ i.i.d. $p$-dimensional random vectors for each $k \in \{1,2,...,K\}$. Let $X_{i,s}^{(k)}$ be the $s$-th entry of the p-dimensional vector $X_i^{(k)}$ for $1 \leq s \leq p$. Say we are given $K$ undirected graphs with the same number of nodes $p$, i.e. $G^{(k)} = (V,E^{(k)})$ with the same vertex set $V = \{1,2,...,p\}$ and potentially different  edge sets $E^{(k)} \subset V \times V$  for $1 \leq k \leq K$. Each random variable $X_{i,s}^{(k)}$ is associated with a vertex $s \in V$ in the $k$-th graph $G^{(k)}$.  We say $\big\{ X_i^{(k)}\big\}_{1 \leq i \leq n^{(k)}, 1\leq k \leq K}$ forms a family of $p$-dimensional random Ising models of size $K$ if, for $1 \leq k \leq K$, 
     
     (i) the samples $X^{(k)}_1, X^{(k)}_2, ..., X^{(k)}_{n^{(k)}} \overset{\text{i.i.d.}}{\sim} \mathbb{P}_{\bar{\theta}^{(k)}}$, where
    \begin{equation}\label{isingdistk}
        \mathbb{P}_{\bar{\theta}^{(k)}}(x^{(k)}_i) =\frac{1}{Z(\bar{\theta}^{(k)})} \exp \big\{\sum_{(s,t) \in E^{(k)}}\bar{\theta}^{(k)}_{st}x^{(k)}_{i,s} x^{(k)}_{i,t} \big\},
    \end{equation}
    for some parameter $\bar{\theta^{(k)}}_{st} \in \mathbb{R}$ and the partition function $Z(\bar{\theta}^{(k)})$ ensuring that the distribution sums to one;
    
    (ii) \begin{equation}\label{decompose}
            \bar{\theta}^{(k)} = \bar{\theta} + \Delta^{(k)},
        \end{equation}
        with $\bar{\theta},\Delta^{(k)} \in \mathbb{R}^{\binom{p}{2}}$, $\bar{\theta}$ deterministic and $\Delta^{(k)}, 1 \leq k \leq K $ are i.i.d random vectors drawn from distribution $P$; and
        
    (iii) we have for $ 1 \leq k \leq K,$
    \begin{equation}\label{Delta}
        \mathbb{P}_{\Delta^{(k)} \sim P}[\text{supp}(\Delta^{(k)}) \subseteq \text{supp}(\bar{\theta})] = 1.
    \end{equation}
\end{definition}

For our meta learning problem, we have $K$ auxiliary tasks and one novel task, forming a family of $p$-dimensional random Ising models of size $K+1$. We can refer to $\bar{\theta}$ as the \emph{true common parameter vector} and 
    $S := \text{supp}(\bar{\theta})$
is what we call the \emph{true support union}. We can then understand the maximum neighborhood size $d$ for all graphs as
    $d := \max\{|S_1|, |S_2|, ... , |S_p|\}$,
where $S_r$ is the neighborhood set of each node $r \in V$ in the \emph{latent} deterministic graph parametrized by $\bar{\theta}$, defined as:
    $S_r := \{t \in V|(r,t) \in S\}$.
We assume $d \ll p$ ($d$ is small in size compared to $p$), so that the graphs are fairly sparse. By restricting the parameter estimation of the novel task to the true support union that can be estimated using the auxiliary tasks, we can potentially reduce the sample complexity for the novel task by a large margin.

\emph{Remark.}
Note that condition (\ref{Delta}) restricts the support of the randomness in the parameter of each task to $\text{supp}(\bar{\theta})$, which guarantees that the support of each task $\text{supp}(\bar{\theta}^{(k)}) \subseteq \text{supp}(\bar{\theta})$ for $1 \leq k \leq K+1$, with probability 1. For instance, for an arbitrary entry $(s,t) \in \text{supp}(\bar{\theta})$, we have two cases: for task k, if $\Delta^{(k)}_{st} = - \bar{\theta}_{st}$, then from equation (\ref{decompose}) we have $\bar{\theta}^{(k)}_{st} = 0$, so that $(s,t) \not\in \text{supp}(\bar{\theta}^{(k)})$ and $\text{supp}(\bar{\theta}^{(k)}) \subset \text{supp}(\bar{\theta})$ , i.e., we get a proper subset; else if $\Delta^{(k)}_{st} \neq - \bar{\theta}_{st}$, then by the same token, we have $(s,t) \in \text{supp}(\bar{\theta}^{(k)})$. Either way we arrive at $\text{supp}(\bar{\theta}^{(k)}) \subseteq \text{supp}(\bar{\theta})$. Suppose on the contrary we do not impose condition  (\ref{Delta}), then it will be hard for us to estimate a common parameter or a support union useful for all the tasks. On the other hand, there is still great flexibility in the family of distributions since graphs from different tasks can have edge structures with no intersection with arbitrary probability, and we do not assume entries in $\Delta^{(k)}$ to be small in absolute value.

\subsection{Our Improper Estimation Method for Meta Learning on Ising Models}
Our estimation procedure can be divided into two steps. The first step is to recover $S$ from the $K$ auxiliary tasks by estimating $\bar{\theta}$. The second step is the signed edge recovery for task $K+1$ with its support restricted to the estimated support union. Considering the sparsity of the problem, we will use the $\ell_1$-regularized logistic regression in both steps.

\paragraph{Estimating the Support Union from $K$ Tasks.}

Specifically, for the first step, we pool all the samples from the $K$ tasks and estimate $\bar{\theta}$ by minimizing the $\ell_1$-regularized logistic loss between $\bar{\theta}$ and the estimate. For a clearer presentation, we assume that each auxiliary task has the same number of samples, i.e., $n^{(k)} = n$ for $1 \leq k \leq K$. Note that we do not assume that $n^{(K+1)} = n$ for the novel task. Then, for each node $r \in V$, given the samples from all the $K$ auxiliary tasks $ \big\{ X_i^{(k)}\big\}_{1 \leq i \leq n, 1\leq k \leq K}$, for which we use a shorthand notation $\{\mathfrak{X}^n_1\}^K_1$, this regularized regression problem is a convex program of the form 
\begin{equation}\label{improper}
 \hat{\theta}_{\setminus r} = \argmin_{\theta_{\setminus r} \in \mathbb{R}^{p-1}} \ell(\theta_{\setminus r}; \{\mathfrak{X}^n_1\}^K_1) + \lambda {\|{\theta_{\setminus r}}\|}_1, 
\end{equation}
where
\begin{equation}\label{avgloss}
    \ell(\theta_{\setminus r}; \{\mathfrak{X}^n_1\}^K_1) = - \frac{1}{K}\sum^K_{k=1}  \frac{1}{n}\sum_{i=1}^{n} \log \mathbb{P}_{\theta_{\setminus r}} (x_{i,r}^{(k)}|x_{i,\setminus r}^{(k)})
\end{equation}
is the averaged re-scaled negative log likelihood of all the auxiliary tasks and $\lambda > 0$ is a regularization parameter to be specified by the user, which potentially depends on $n, p, d$ and $K$. Note that (\ref{improper}) is an improper estimation  as we estimate a single parameter vector using data from different distributions. We can further write $\ell(\theta_{\setminus r}; \{\mathfrak{X}^n_1\}^K_1) = \frac{1}{K}\sum^K_{k=1} \ell^{(k)}(\theta_{\setminus r}; \{\mathfrak{X}^n_1\}^{(k)})$, where $\{\mathfrak{X}^{n}_1\}^{(k)}$ is another shorthand notation for $\big\{ X_i^{(k)}\big\}_{1 \leq i \leq n}$, and $ \ell^{(k)}(\theta_{\setminus r}; \{\mathfrak{X}^n_1\}^{(k)}) := -\frac{1}{n}\sum_{i=1}^{n} \log \mathbb{P}_{\theta_{\setminus r}} (x_{i,r}^{(k)}|x_{i,\setminus r}^{(k)})$ for $1\leq k \leq K$. With our estimate $\hat{\theta}$ for the true common parameter vector $\bar{\theta}$, we can use $\text{supp}(\hat{\theta})$ as the estimate of the true support union $S = \text{supp}(\bar{\theta}).$

\paragraph{Estimating the Support of the Novel Task.}

For the second step of our estimation, we aim to estimate the true parameter vector $\bar{\theta}^{(K+1)}$ with knowledge from auxiliary tasks. Suppose we have successfully recovered the true support union $S$ from the first step. Since the true support for the novel task $\text{supp}(\theta^{K+1}) \subseteq S$ is assumed in the problem setting, we can perform a regularized logistic regression with an additional restriction: 
\begin{equation}\label{restrict}
\begin{aligned}
   \hat{\theta}^{(K+1)}_{\setminus r} = \argmin_{{\theta_{\setminus r} \in \mathbb{R}^{p-1}}} \Big\{ \ell^{(K+1)}(\theta_{\setminus r}; \{\mathfrak{X}^{n^{(K+1)}}_1\}^{(K+1)}) 
      + \lambda^{(K+1)} {\|{\theta_{\setminus r}}\|}_1 \Big\} \\
     \textrm{s.t.}  \quad \text{supp}(\theta_{\setminus r}) \subseteq \text{supp}(\hat{\theta}_{\setminus r}),
\end{aligned}
\end{equation}
where 
\begin{equation}
    \ell^{(K+1)}(\theta_{\setminus r}; \{\mathfrak{X}^{n^{(K+1)}}_1\}^{(K+1)}) = -\frac{1}{n^{(K+1)}}\sum_{i=1}^{n^{(K+1)}} \log \mathbb{P}_{\theta_{\setminus r}} (x_{i,r}^{(K+1)}|x_{i,\setminus r}^{(K+1)}).
\end{equation}
Here $\{\mathfrak{X}^{n^{(K+1)}}_1\}^{(K+1)} = \big\{ X_i^{(K+1)}\big\}_{1 \leq i \leq n^{(K+1)}}$ denotes the $n^{(K+1)}$ samples from the $(K+1)$-th task, and $\lambda^{(K+1)}$ is the regularization parameter for the novel task, which potentially depends on $n^{(K+1)}, p, d$.

\section{Theoretical Results}\label{res}
\subsection{Assumptions}
The success of our method requires some assumptions on the structure of the logistic regression, most of which are the dependency and incoherence conditions in the work by Ravikumar et al. \cite{ravikumar2010high} generalized to our multiple tasks setting (see Assumptions \ref{auxa1}, \ref{auxa2}, \ref{novela1}, \ref{novela2}). We also make assumptions regarding the randomness in the parameters for each task, which, intuitively speaking, make the tasks similar in some sense (see Assumption \ref{auxa3}).

\subsubsection{Assumptions in Auxiliary Tasks}\label{sectionauxass}
The assumptions for the support union estimation are stated in terms of the Hessian of the likelihood function $\mathbb{E}\{-\ell(\theta_{\setminus r}; \{\mathfrak{X}^n_1\}^K_1)\}$ evaluated at the true common parameter $\bar{\theta}_{\setminus r}$ for each node $r \in V$. More specifically, the Hessian for any fixed node $r \in V$ is a $(p-1) \times (p-1)$ matrix of the form $\bar{Q}_r:= \mathbb{E}\{\frac{1}{K} \sum^K_{k=1}\nabla^2 \log\mathbb{P}_{\bar{\theta}}[X_r^{(k)}|X^{(k)}_{\setminus r}]\}$,
which has an explicit expression 
\begin{equation}\label{calcQ}
    \bar{Q}_r = \frac{1}{K} \sum^K_{k=1}\mathbb{E}[\eta(X^{(k)}; \bar{\theta})X^{(k)}_{\setminus r} {(X^{(k)}_{\setminus r})}^T], \;\;\;\; \eta(u;\theta) := \frac{4\exp(2u_r\sum_{t\in V \setminus r}\theta_{rt}u_t)}{(\exp(2u_r\sum_{t\in V \setminus r}\theta_{rt}u_t)+1)^2}.
\end{equation}
Here $\eta(u;\theta)$ is the variance function. Note that our expectation is taken with respect to the joint distribution of the data $\big\{ X_i^{(k)}\big\}_{1 \leq i \leq n, 1\leq k \leq K}$ and the random variables $\{\Delta^{(k)}\}^K_{k=1}$ in the parameters.

In the following, we simplify the notation $\bar{Q}_r$ as $\bar{Q}$ since the reference node $r$ is used throughout the analysis and should be understood implicitly. With this notation, we also write  $\bar{Q}_{SS}$ to denote $\bar{Q}_{S_r S_r }$. We follow the same practice for similar shorthand notations in most part of the paper to lighten the notations a little bit. This should not cause confusion since the elements in $S$ are pairs of nodes (two dimensional), while those in $S_r$ are individual nodes (one dimensional).

\emph{Dependency Condition.} Following the dependency condition imposed by Ravikumar et al. \cite{ravikumar2010high}, we assume that the subset of the Fisher information matrix corresponding to the relevant covariates in the true support union has bounded eigenvalues. We have:
\begin{assumption}\label{auxa1}
There exist constants $C_{\min} > 0 $  and $D_{\max} > 0 $ such that
\begin{equation}
     \Lambda_{\min}(\bar{Q}_{SS}) \geq C_{\min}, \;\;\;\;\Lambda_{\max}(\frac{1}{K} \sum_{k=1}^K \mathbb{E} [  X_{\setminus r}^{(k)} (X_{\setminus r}^{(k)})^T] )\leq D_{\max}.
\end{equation}
\end{assumption}
These assumptions make sure that the covariates are not excessively dependent.

\emph{Incoherence Condition.}
To prevent the large number of irrelevant covariates (the ones outside the support) having too strong an effect on the relevant covariates (the ones in the support), as pointed out by Ravikumar et al. \cite{ravikumar2010high}, the following assumption is  required:
\begin{assumption}\label{auxa2}
 There exists an $\alpha \in (0,1]$ such that
\begin{equation}
    \vertiii{\bar{Q}_{S^c S} (\bar{Q}_{S S})^{-1}}_{\infty} \leq 1-\alpha.
\end{equation}
\end{assumption}

\emph{Additional Assumptions on $\{\Delta^{(k)}\}_{1\leq k \leq K}$}.
The success of our method also relies on some reasonable and flexible assumptions on the centering of the random variables $\{\Delta^{(k)}\}_{1\leq k \leq K}$ underlying the parameters of each task --- reasonable in the sense that the tasks are similar enough to provide useful information, and flexible so that there is as little inductive bias as possible.

For simplicity, without writing down the samples for the auxiliary tasks, we use a shorthand notation $\nabla \ell(\bar{\theta}_{\setminus r})$ to denote the gradient of the loss function for the improper estimation, evaluated at the true common parameter vector $\bar{\theta}$; similarly $\nabla \ell^{(k)}(\bar{\theta}^{(k)}_{\setminus r})$ means the gradient of the loss function for the $k$-th auxiliary task evaluated at the true parameter for that particular task, i.e. $\bar{\theta}^{(k)}$. Then we have

\begin{assumption}\label{auxa3} For any $\varepsilon > 0$,
\begin{equation}\label{decayassump}
\mathbb{P}\Big(\|\mathbb{E}[\nabla \ell(\bar{\theta}_{\setminus r}) - \frac{1}{K} \sum_{k=1}^K \nabla \ell^{(k)}(\bar{\theta}^{(k)}_{\setminus r})]\|_{\infty} > \sqrt{\frac{8 \log(2 p/ \varepsilon)}{nK}} \Big) \leq \varepsilon,
\end{equation}
for all $r \in V$.
\end{assumption}
The constants 8 and 2 above are just for ease of calculation, and can be substituted with other constants without harm.  Note that  $\frac{1}{K}\sum_{k=1}^K \nabla \ell^{(k)}(\bar{\theta}^{(k)}_{\setminus r})$ is just the counterpart of $\nabla \ell(\bar{\theta}_{\setminus r}),$ with the gradient of loss for each task evaluated at their own true parameters $\bar{\theta}^{(k)}_{\setminus r}$, which depend on each random $\Delta^{(k)}$, as opposed to all evaluated at the common parameter vector $\bar{\theta}_{\setminus r}$ in $\nabla \ell(\bar{\theta}_{\setminus r})$.

\emph{Remark.} In calculating the expectation  over the joint distribution of $X$ and $\Delta$ in this assumption,  $n$ and $K$ are eliminated due to the i.i.d. $\{\Delta^{(k)}\}_{1\leq k \leq K}$, and conditioned on which, the i.i.d. property of samples $\{X^{(k)}_i\}_{1\leq i \leq n, 1 \leq k \leq K}$ follows. Hence, the requirement on $K$ tasks is just one on the distribution $P$  of $\{\Delta^{(k)}\}_{1\leq k \leq K}$ for the family of Ising models. The assumption helps specify some symmetry for the family of distributions and is compatible with the primal-dual witness. Here we provide an illustrative example to help with intuitive comprehension: the quantity in the assumption can be written explicitly as $\|\mathbb{E}_{\Delta \sim P} \big[ \mathbb{E}_{X \sim \bar{\theta}+\Delta} \big[X_{\setminus r}(\mathbb{E}_{X \sim \bar{\theta}}[X_r|X_{\setminus r}]-\mathbb{E}_{X \sim \bar{\theta}+\Delta}[X_r|X_{\setminus r}])|\Delta \big] \big]\|_{\infty}$. It can be checked that with a common latent graph with 3 nodes, 3 edges and $\bar{\theta} = (1,1,1)$, a setting of $\Delta$ resulting in the real tasks to have only 2 edges each with  values $\bar{\theta}+\Delta \in \{(1.75, 1.75, 0), (1.75, 0, 1.75), (0,1.75, 1.75)\}$  with equal probabilities will fulfill our condition with the desired quantity around $0$. More details of calculation and illustration can be found in Appendix \ref{toyeg}.
\subsubsection{Assumptions in Novel Task}\label{sectionnovelass}
Our assumptions for the novel task is analogous to those for the improper estimation using auxiliary tasks, but with the parameters and the random variables restricted to the true support union $S$. We base the assumption on the Hessian of the likelihood function $\mathbb{E}\{-\ell^{(K+1)}(\theta_S; \{\mathfrak{X}^{n^{(K+1)}}_1\}^{(K+1)}_S)\}$ evaluated at the true parameter for the $(K+1)$-th task, $\bar{\theta}_S^{(K+1)}$: $$\bar{Q}^{(K+1)}_r:= \mathbb{E}\{\nabla^2 \log\mathbb{P}_{\bar{\theta}_S^{(K+1)}}[X_r^{(K+1)}|X^{(K+1)}_S]\}.$$
This is given as the explicit expression
\begin{equation}\label{calcQK1}
    \bar{Q}^{(K+1)}_r = \mathbb{E}[\eta(X_S^{(K+1)}; \bar{\theta}_S^{(K+1)})X^{(K+1)}_S {(X^{(K+1)}_S)}^T].
\end{equation}
Similarly, with the subscript $r$ dropped for simplicity of notation, and by using $S^{(K+1)} := \text{supp}(\bar{\theta}^{(K+1)})$ to denote the true support of the $(K+1)$-th task which satisfies $S^{(K+1)} \subseteq S$, we again follow the assumptions made by Ravikumar et al. \cite{ravikumar2010high}:

\begin{assumption}\label{novela1}
\emph{Dependency Condition.}
There exist constants $C_{\min}^{(K+1)} > 0 $  and ${D_{\max}^{(K+1)}} > 0 $ such that
\begin{equation}
     \Lambda_{\min}(\bar{Q}^{(K+1)}_{S^{(K+1)} S^{(K+1)}}) \geq C^{(K+1)}_{\min}, \;\;\;\; \Lambda_{\max}( \mathbb{E} [  X_S^{(K+1)} (X_S^{(K+1)})^T] )\leq D^{(K+1)}_{\max}.
\end{equation}
\end{assumption}

\begin{assumption}\label{novela2}
\emph{Incoherence Condition.}
There exists an $\alpha^{(K+1)} \in (0,1]$ such that
\begin{equation}
    \vertiii{\bar{Q}^{(K+1),S}_{{[S^{(K+1)}]}^c S^{(K+1)}} (\bar{Q}^{(K+1)}_{S^{(K+1)} S^{(K+1)}})^{-1}}_{\infty} \leq 1-\alpha^{(K+1)},
\end{equation}
where for compactness, we use the notation $\bar{Q}^{(K+1),S}_{{[S^{(K+1)}]}^c S^{(K+1)}}$ to denote $\bar{Q}^{(K+1)}_{{([S^{(K+1)}]}^c \cap S) S^{(K+1)}}$.
\end{assumption}

\subsection{Main Theorems}
The following of our two main theorems together show that the overall sufficient sample complexity for the sign-consistency of the estimators in the two steps of our meta learning approach is $\Omega(1) $ for each of the $K = \Omega(d^3 \log p)$ auxiliary tasks and $\Omega(d^3 \log d)$ for the novel task.
\subsubsection{Support Union Recovery}

Our first theorem demonstrates that the sufficient sample complexity for the recovery of the true support union $S$ by our estimator in (\ref{improper}) is $n = \Omega(d^3 \log p/K)$ per task for $K$ tasks. This means that for the situation with numerous tasks $K = \Omega(d^3 \log p)$, the sufficient sample complexity per task is as small as $\Omega(1)$. From the condition we obtained on the regularizing parameter $\lambda$, we can also see that having more tasks will give a good estimate of the support union with less penalty, without having to increase the number of samples per task --- the more tasks the better in this case.

\begin{theorem}\label{thm: thm1}
    For a family of $p$-dimensional random Ising models of size $K$ described in Definition \ref{family} with $n^{(k)} = n$ for $1 \leq k \leq K$, suppose Assumptions \ref{auxa1}, \ref{auxa2}, \ref{auxa3} are satisfied.   
    Let $\{\mathfrak{X}^n_1\}^{(k)}$ be a set of $n$ $i.i.d.$ samples from the model specified by $\bar{\theta}^{(k)}$, and $\{\mathfrak{X}^n_1\}^K_1$ denote all the samples from the $K$ tasks. Suppose that the regularization parameter $\lambda$ is selected to satisfy $\lambda \geq \beta \sqrt{\frac{\log p}{nK}}$ 
    for some constant $\beta> 0$, then there exists a positive constant $L$, independent of $(n,p,d, K)$, such that if $nK > L d^3 \log p$,
    then, for estimating the Ising model with the true common parameter $\bar{\theta}, $ for some constant $c>0$, the following properties hold with probability at least $1 - \mathcal{O}( \exp(-c\lambda^2 nK)) $.
    
    (a) For each node $r \in V $, the $\ell_1$-regularized logistic regression (\ref{improper}), given data $\{\mathfrak{X}^n_1\}^K_1$, has a unique solution, and so uniquely specifies a signed neighborhood $\mathcal{\hat{N}}_{\pm}(r)$.
    
    (b) For each $r \in V$, the estimated signed neighborhood $\mathcal{\hat{N}}_{\pm}(r)$ correctly excludes all edges not in the true neighborhood, so that supp$(\hat{\theta}) \subseteq $ supp$(\bar{\theta})$. Moreover, it correctly includes all edges $(r,t)$ for which $|\bar{\theta}_{rt}| \geq \frac{10}{C_{\min}}\sqrt{d} \lambda.$
\end{theorem}

\begin{proof}[Sketch of Proof for Theorem \ref{thm: thm1}]
We use the primal-dual witness approach \cite{wainwright2006sharp, ravikumar2010high} and the proof can be divided into two parts. The first part shows that imposing the dependence and incoherence assumptions (Assumptions \ref{auxa1} and  \ref{auxa2}) on the \emph{population version} of the Fisher information matrix $\bar{Q}$ guarantees (with high probability) that analogous conditions hold for the \emph{sample Fisher information matrix} $Q^N := \hat{\mathbb{E}}[-\nabla^2 \ell(\bar{\theta}_{\setminus r};\{\mathfrak{X}^n_1\}^K_1)]$ (see (\ref{defQN}) in Appendix); the second part of the proof is devoted to show that if the dependence condition and incoherence condition are imposed on the \emph{sample Fisher information matrix} $Q^N$, then the growth condition and choice of $\lambda$  from Theorem \ref{thm: thm1} are sufficient to ensure that the graph associated with the true common parameter vector is recovered with high probability \cite{ravikumar2010high}. 

The first part of the proof mainly use techniques such as norm inequalities and union bounds \cite{hoeffding1994probability, ravikumar2010high} to get a high probability  $1 - \mathcal{O}(\exp (-b\frac{nK}{d^3}+ \log p))$ for some constant $b > 0$, which in turn yields the growth condition in on number of samples per task $n$ and number of tasks $K$. 

In the second part, the key is to verify the strict dual feasibility \cite{rockafellar2015convex}. Using some norm inequalities and following a method used in a different context \cite{rothman2008sparse, ravikumar2010high}, we show that it suffices to bound the random term $\|\nabla \ell(\bar{\theta};\{\mathfrak{X}^n_1\}^K_1) \|_{\infty} = \|\frac{1}{K} \sum^K_{k=1} \nabla \ell^{(k)}(\bar{\theta}; \{\mathfrak{X}^n_1\}^K_1 ) \|_{\infty}$, which we decompose into two parts as follows 
\begin{multline}
    \|\nabla \ell(\bar{\theta};\mathfrak{X}^n_1\}^K_1) \|_{\infty}    \\\leq
       \|\underbrace{ \frac{1}{K} \sum^K_{k=1} \Big\{ \nabla \ell^{(k)}(\bar{\theta}; \{\mathfrak{X}^n_1\}^{(k)}) - \nabla \ell^{(k)}(\bar{\theta}^{(k)}; \{\mathfrak{X}^n_1\}^{(k)})\Big\}}_\text{$Y_1$}  \|_{\infty} 
      + \| \underbrace{\frac{1}{K} \sum^K_{k=1} \nabla \ell^{(k)}(\bar{\theta}^{(k)}; \{\mathfrak{X}^n_1\}^{(k)}) }_\text{$Y_2$}\|_{\infty}.
\end{multline}

Using Hoeffding's Inequality with latent conditional independence (LCI) \cite{ke2019exact} with latent variables $\{\Delta^{(k)}\}_{k=1}^K$, $\|Y_2\|_{\infty}$ can be bounded with high probability in the sense that
\begin{equation}
    \mathbb{P}[\|Y_2\|_{\infty} > \delta ] \leq 2 \exp \left(- \frac{\delta^2 nK}{8} + \log p \right).
\end{equation}
Also note that $Y_1 = \nabla \ell(\bar{\theta}_{\setminus r}) - \frac{1}{K}\sum_{k=1}^K \nabla \ell^{(k)}(\bar{\theta}^{(k)}_{\setminus r})$ using the shorthand notations in (\ref{decayassump}). We can then bound $\|Y_1\|_{\infty}$ by writing 
\begin{equation}
\|Y_1\|_{\infty}  = \|Y_1 - \mathbb{E}(Y_1) +\mathbb{E}(Y_1)\|_{\infty}  \leq \|Y_1 - \mathbb{E}(Y_1) \|_{\infty} + \|\mathbb{E}(Y_1)\|_{\infty}.
\end{equation}
Using Assumption (\ref{auxa3}), it is not hard to derive that 
\begin{equation}
    \mathbb{P}(\|\mathbb{E}(Y_1)\|_{\infty} \geq \delta) \leq 2 \exp \left(- \frac{\delta^2 nK}{8} + \log p \right).
\end{equation}
For the term $\|Y_1 - \mathbb{E}(Y_1) \|_{\infty}$, we can get the same rate of decay with respect to $(n,p,d,K)$ by using LCI Hoeffding's inequality \cite{ke2019exact} again. Then applying union bounds and setting $\delta$ to be $\lambda$ times a constant, we get the rate $\mathcal{O} (\exp (-c\lambda^2 n K))$ as in Theorem \ref{thm: thm1}, as well as the condition on $\lambda$. The detailed proof is in Appendix \ref{thm1pf}.
\end{proof}

\subsubsection{Support Recovery for Novel Task}

For the novel task, the following theorem provides the sufficient conditions and a probability lower bound for the sign-consistency of the estimate, from which we can conclude that using the knowledge learned from the auxiliary tasks, the consistent signed neighborhood selection for the novel task can be achieved in a sufficient sample complexity of $n^{(K+1)} = \Omega(d^3 \log d)$. 
\begin{theorem}\label{thm: thm2}
    Suppose we have recovered the true support union $S$ of a family of $p$-dimensional random Ising models of size K described in Definition \ref{family}. For a novel task of Ising model selection with parameter $\bar{\theta}^{(K+1)}$ such that $\text{supp}(\bar{\theta}^{(K+1)}) \subseteq S$ and satisfying Assumptions \ref{novela1}, \ref{novela2}, suppose the regularization parameter is chosen such that $\lambda^{(K+1)} \geq \beta \sqrt{\frac{\log  d}{n^{(K+1)} }}$
    for some constant $\beta > 0$, then there exists a positive constant $L$, independent of $(n^{(K+1)} ,p,d)$, such that if 
        $n^{(K+1)} >  L d^3 \log d$,
    then for some constant $c>0$, the following properties hold with probability at least $1-\mathcal{O}(\exp(- c {(\lambda^{(K+1)})}^2 n^{(K+1)}))$.
     
    (a) For each node $r \in V$, the $\ell_1$-regularized logistic regression for estimating $\bar{\theta}^{(K+1)}$ in the novel task, given data $\{\mathfrak{X}^{n^{(K+1)}}_1\}^{(K+1)}$has a unique solution $\hat{\theta}^{(K+1)}$, and so uniquely specifies a signed neighborhood $\mathcal{\hat{N}}^{(K+1)}_{\pm}(r) :=\{\text{sign}(\hat{\theta}^{(K+1)}_{ru})u|u \in V \setminus r, \hat{\theta}^{(K+1)}_{ru} \neq 0 \}$
    
    (b) For each $r \in V$, the estimated signed neighborhood vector $\mathcal{\hat{N}}^{(K+1)}_{\pm}(r)$ correctly excludes all edges not in the true neighborhood $\mathcal{N}^{(K+1)}_{\pm}(r) :=\{\text{sign}(\bar{\theta}^{(K+1)}_{ru})u|u \in V \setminus r, \bar{\theta}^{(K+1)}_{ru} \neq 0 \}$. Moreover, it correctly includes all edges with $| \bar{\theta}^{(K+1)}_{rt}| \geq \frac{10}{C^{(K+1)}_{\min}}\sqrt{d} \lambda^{(K+1)}.$
\end{theorem}
\emph{Remark.} Note that the constants $\beta, L$ and $c$ we use in the theorems are just some general constants for convenience of notation. The ones used in Theorem \ref{thm: thm1} are not related to those in Theorem \ref{thm: thm2}.
\begin{proof}[Sketch of Proof for Theorem \ref{thm: thm2}]
We use the primal-dual witness approach again \cite{wainwright2006sharp, ravikumar2010high}. We have supposed that we have recovered the true support union $S$ from our estimate for the true common parameter, $\hat{\theta}.$ The constraint in (\ref{restrict}) then enables us to convert the problem into one without the restriction and with a parameter and data of dimension $|S_r|$ with $|S_r| \leq d$ for all $r \in V$, for we can combine the constraint straightforward into the minimization problem. With some abuse of notation using $S$ to denote $S_r$ as before, we can write
\begin{equation}\label{restrict0}
    \hat{\theta}^{(K+1)}_{S} = \argmin_{{\theta_{S} \in \mathbb{R}^{|S_r|}}} \Big\{ \ell^{(K+1)}(\theta_{S}; \{\mathfrak{X}^{n^{(K+1)}}_{1,S}\}^{(K+1)})
    + \lambda^{(K+1)} {\|{\theta_{S}}\|}_1 \Big\},
\end{equation}
and $\hat{\theta}^{(K+1)}_{S^c} = 0.$
In this way, we solve a convex program analogous to (\ref{isingreg}), which is learning on a single Ising model, but with dimension reduced from $p$ to $|S_r|$, and consequently to $d$. The detailed proof is in Appendix \ref{thm2pf}.
\end{proof}

\section{Experiments}\label{experiment}
To help illustrate and validate our theories, we conduct a group of synthetic experiments and report the success rates  for recovery of  the true support union. We run the experiments with three different graph sizes $p \in \{6,12,24\}$, where we fix the maximum neighborhood size $d$ to be 3, and set the number of tasks scaling as $K=d^3\log p$, with sample size for each auxiliary task $n= Cd^3\log p/K = C$ for $C$ ranging from 5 to upwards of 200.  Then based on the estimated support union using $C=50$, we use different sample sizes $n^{(K+1)}=C^*d^3\log(d)$ for the novel task when $C^*$ changes from $5$ to $200$ and calculate the success rates for signed edge recovery of the novel task. We plot the success rates against $C$ and $C^*$ for the two steps respectively in Figure \ref{fig: fig1}. The curves approximately lie on top of one another as the success rates tend to 1, as predicted by Theorem \ref{thm: thm1} and \ref{thm: thm2}.

\begin{figure}[ht]
\setlength\Myfigwd{8cm}
\floatbox[{\capbeside
\thisfloatsetup{capbesidesep=quad,
justification=justified,
capbesideposition= {right,center},
capbesidewidth=\dimexpr\linewidth-\Myfigwd-1em\relax}}]
{figure}[\FBwidth]
{\includegraphics[width=\Myfigwd]{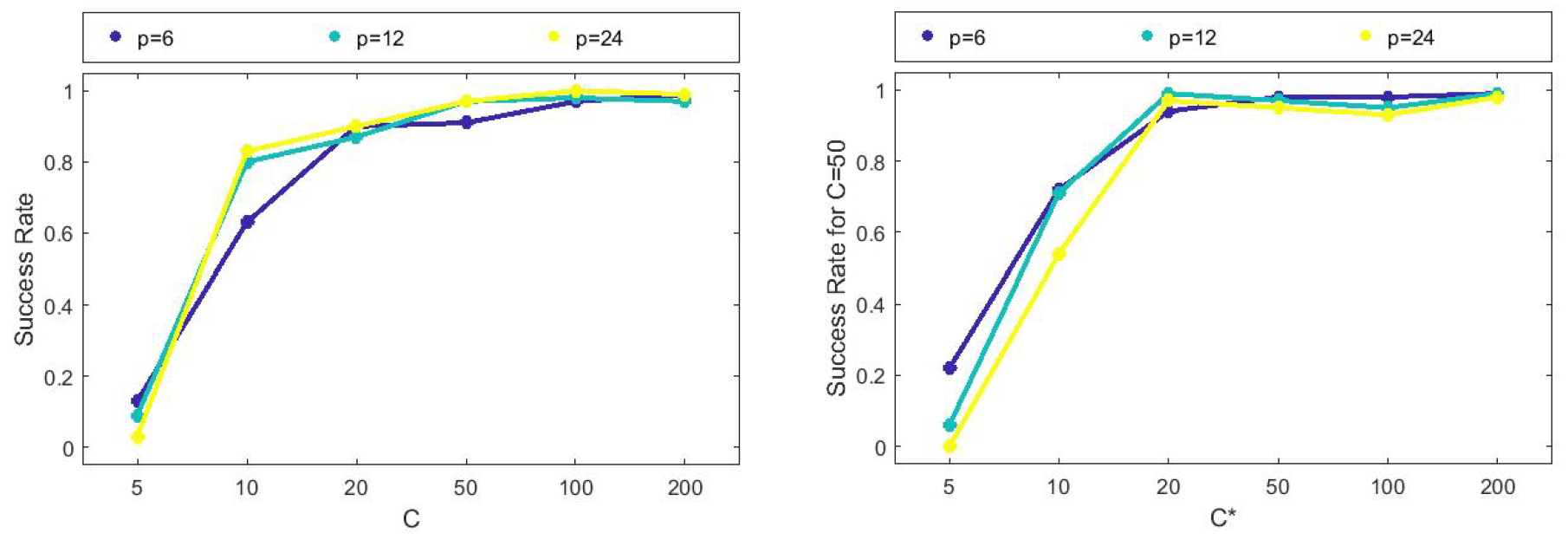}}{\caption{(Left) The success rate for support union recovery vs. the choice of $C$. (Right) The success rates for signed edge recovery for novel task restricted to the estimated support union based on $C = 50$ vs. the choice of $C^*$.}\label{fig: fig1}}
\end{figure}

As another motivation and validation of our method, we used the real-world dataset “1000 Functional Connectomes” at \url{http://www.nitrc.org/projects/fcon_1000/} from 1128 subjects, $K = 41$ sites worldwide, and $p = 157$ brain regions. We estimated the support union with precision 0.8837, recall 0.9007 and F1-score 0.8921. More experimental details are in Appendix \ref{expdetail}.

\paragraph{Concluding Remarks.} Our method and analysis in this paper can be extended to more general cases of Markov random fields. Since logistic regression can be generalized to multi-class logistic regression, the analysis performed on the meta learning problem for Ising model can find its analog in multi-class discrete Markov random fields \cite{ravikumar2010high}. Some other interesting directions for future work based on our method include but not limited to meta learning for general continuous Markov random fields, meta learning for graphical models with hyper-edges that can connect multiple nodes instead only two, or time-varying graphical models, etc. We believe our results can provide a solid foundation and open a novel perspective for meta learning in Markov random fields and related graphical models.

\bibliographystyle{plain}
\bibliography{sample}

\newpage
\appendix
\section{Summary of Notations}\label{sumnot}

\begin{table}[hbt!]\caption{Notation used in this paper.}\label{notation}

\begin{center}

\begin{small}
\centering

\begin{tabular}{ll}
\toprule
Notation & Description \\
\midrule
sign($x$)   & Sign of $x \in \mathbb{R}$\\ 
\parbox{1cm}{$\|a\|_1$} & $\ell_1$-norm of vector $a \in \mathbb{R}^n,$  i.e., $\sum^n_{i=1}|a_i|$\\
\parbox{1cm}{$\|a\|_2$} & $\ell_2$-norm of vector $a \in \mathbb{R}^n,$  i.e., $\sqrt{\sum^n_{i=1}a_i^2}$\\
\parbox{1cm}{$\|a\|_{\infty}$} & $\ell_{\infty}$-norm of vector $a \in \mathbb{R}^n,$  i.e., $\max^n_{i=1}|a_i|$\\
\parbox{1cm}{$\vertiii{A}_{\infty}$ } & $\ell_{\infty}$-operator-norm of matrix $A \in \mathbb{R}^{m \times n}$, i.e., $\max_{1 \leq i \leq m}\sum^n_{j=1}|A_{ij}|$\\
$\Lambda_{\min}(A)$  & Minimum eigenvalue of matrix $A \in \mathbb{R}^{m \times m}$ \\
$\Lambda_{\max}(A)$  & Maximum eigenvalue of matrix $A \in \mathbb{R}^{m \times m}$ \\
\parbox{1cm}{$\vertiii{A}_2$ } & $\ell_2$-operator-norm of matrix $A \in \mathbb{R}^{m \times n}$, i.e., $\sqrt{\Lambda_{\max}(A^T A)}$\\
$\text{supp}(a)$ & Support set of vector $a \in \mathbb{R}^{p}$, i.e., $\{i|a_i \neq 0, 1 \leq i \leq p\}$ \\
$|S|$ & Number of elements in set $S$ \\
$S^c$ & Complement of set $S$ \\
\parbox{1cm}{$a_S$  } & Sub-vector of vector $a \in \mathbb{R}^n$ indexed by the entries in set $S$ , i.e., $(a_i)_{i\in S}$\\
\parbox{1cm}{$A_{S_1 S_2}$  } & Sub-matrix of matrix $A^{m \times n}$ indexed by elements in $S_1 \times S_2$,  i.e., $(A_{(i,j)})_{i \in S_1, j \in S_2 }$\\
$\mathcal{O}(g(n))$ & $f(n) = \mathcal{O}(g(n))$ if $f(n) \leq Kg(n)$ for some constant $K < \infty$ \\
$\Omega(g(n))$ & $f(n) = \Omega(g(n))$ if $f(n) \geq K'g(n)$ for some constant $K' > 0$ \\
\bottomrule
\end{tabular}
\end{small}
\end{center}
\end{table}

\section{Details of Experiments}\label{expdetail}
Given fixed values of $p$ and $d$, we simulate sparse random graphs by first randomly choosing whether an edge exists or not with a probability of $\frac{d}{p-1}$. At the end we check if the maximum neighborhood size $d$ is satisfied; if not, we redo the generating process until we get a random graph with maximum degree $d$. For the non-zero edge values, we use \emph{mixed couplings} \cite{ravikumar2010high}, that is, each existent edge (edge in the true support union in our case) has value $\bar{\theta}_{st} = \pm 0.5$ with equal probability. Then, to generate the random parameter of each task: for $1 \leq k \leq K+1$ and $(s,t) \in S$, we set $\bar{\theta}^{(k)}_{st} = \bar{\theta}_{st}X^{(k)}_{st}$ with $X^{(k)}_{st} \overset{i.i.d.}{\sim} \text{Bernoulli}(0.9)$. For the samples, we use Gibbs sampling \cite{casella1992explaining} with 10 iterations to generate each $p$-dimensional data sample for the binary node values according to the specific distributions of Ising models (see (\ref{isingdistk})) using our simulated parameter values.  Under each setting of the $(p, C)$ pair, we run the experiment 100 times to record whether or not it successfully recovers the  neighborhood sets, and take the average of these 100 trials to calculate the success rate $\hat{\mathbb{P}}[\hat{\mathcal{N}}(r) = \mathcal{N}(r)]$. The regularization parameter $\lambda$ is set to be a constant factor of $\sqrt{\frac{\log p}{nK}}$ as suggested by Theorem  \ref{thm: thm1}. Here the constant factor is set to 1 by default, which works well. With $\lambda^{(K+1)}$ a constant factor (1) of $\sqrt{\frac{\log{d}}{n^{(K+1)}}}$, we then perform restricted novel task estimation 100 times for $n^{(K+1)} = C^*d^3\log(d)$ with different values for $C^*$, where the success rate for the novel task include sign information, i.e. it is calculated as $\hat{\mathbb{P}}[\hat{\mathcal{N}}^{(K+1)}_{\pm}(r) = \mathcal{N}^{(K+1)}_{\pm}(r)]$.

For the real-world data experiment, the sample sizes for each individual task range from $700$ to $8748$, with an average size of around $3156$ and standard deviation $1858$.  Due to the limited sample sizes for each individual task, we cannot be certain to retrieve the true signed support for each individual task or to evaluate the novel task performance. Yet we do have an independent set with $68259$ samples to retrieve the true support union. When running the algorithm for support union recovery, we used all 41 tasks, and we further down-sampled the data in each task by half randomly. The constant factor in $\lambda$ was tuned to be 2 to get reasonably sparse graphs $ d = 19$ compared to the number of nodes $p = 157$. 

\section{Illustrative Example}\label{toyeg}
To verify that Assumption \ref{auxa3} can be satisfied for a large family of distributions, we provide an illustrative example to demonstrate its viability. The infinity norm in the assumption can be written explicitly as 
\begin{equation}\label{rephrase}
   \|\mathbb{E}_{\Delta \sim P} \big[ \mathbb{E}_{X \sim \bar{\theta}+\Delta} \big[X_{\setminus r}(\mathbb{E}_{X \sim \bar{\theta}}[X_r|X_{\setminus r}]-\mathbb{E}_{X \sim \bar{\theta}+\Delta}[X_r|X_{\setminus r}])|\Delta \big] \big]\|_{\infty}. 
\end{equation}
For a simple undirected graph with 3 nodes and 3 potential edges, we let the latent underlying graph have the parameter vector $\bar{\theta} = (\bar{\theta}_{12}, \bar{\theta}_{13}, \bar{\theta}_{23}) = (1,1,1)$. See Figure \ref{fig: fig2} for a graph illustration. Then we let the randomness in the parameter for the observable graphs to have the following pattern
\begin{equation*}
    \Delta =\begin{cases}
            (a-1,a-1,-1), & \text{with probability } \frac{1}{3}\\
            (a-1,-1,a-1), & \text{with probability } \frac{1}{3}\\
            (-1,a-1,a-1), & \text{with probability } \frac{1}{3},
              \end{cases} 
\end{equation*}
resulting in potentially 3 kinds of graphs, each with 2 edges with the same edge value $a$ (see Figure \ref{fig: fig3}).

\begin{figure}[ht]
\setlength\Myfigwd{4cm}
{\includegraphics[width=\Myfigwd]{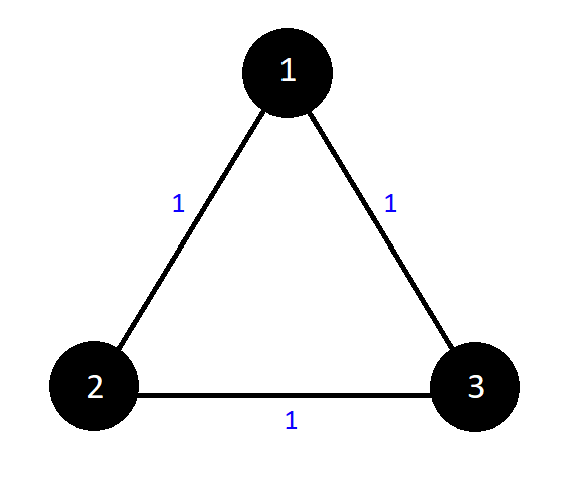}}{\caption{Latent common graph with deterministic edge vector $\bar{\theta}= (\bar{\theta}_{12}, \bar{\theta}_{13}, \bar{\theta}_{23}) = (1,1,1).$}\label{fig: fig2}}
\end{figure}

\begin{figure}[ht]
\setlength\Myfigwd{4cm}
{\includegraphics[width=\Myfigwd]{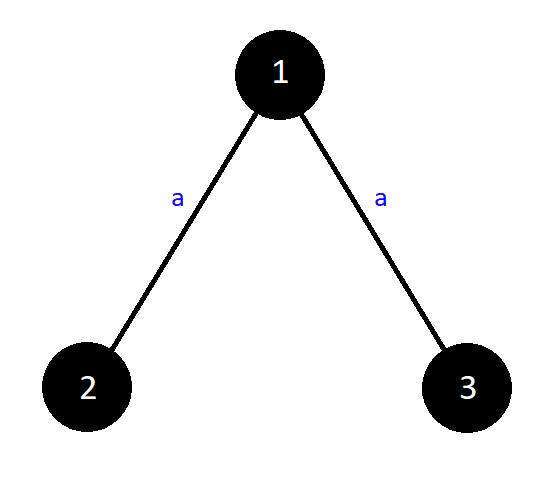}
\includegraphics[width=\Myfigwd]{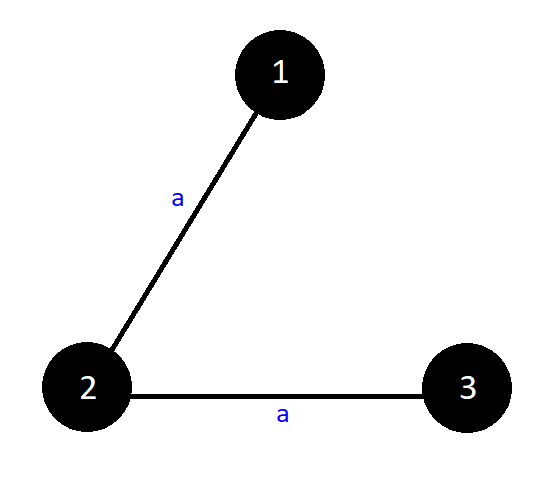}
\includegraphics[width=\Myfigwd]{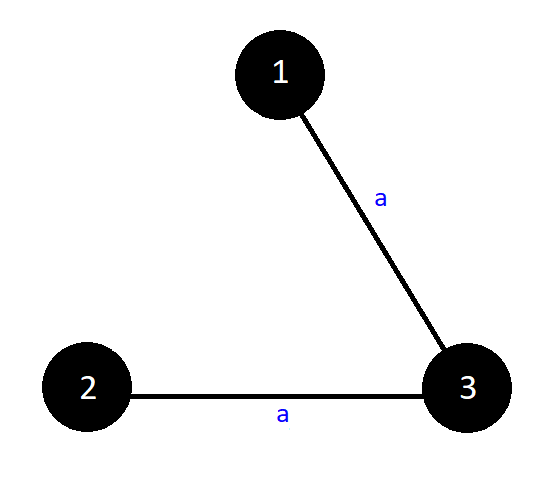}}{\caption{Observable graphs with edge vector $\bar{\theta}+\Delta$, each with 2 edges with the same value $a$.}\label{fig: fig3}}
\end{figure}

Next, we need to find a value of $a$ that can fulfill our condition. Notice that the condition involves the expectation over both $\Delta$ and $X$, and thus we need to find out explicitly the probabilities of all combinations of the 3 binary states under each of the three graph settings.

In the following, we use $\mathbb{P}(x_1,x_2,x_3)$ to denote $\mathbb{P}(X_1 = x_1,X_2 = x_2,X_3 = x_3)$ for simplicity. Now consider node $r=1$. In particular we will need the conditional distribution of $X_1$ given $X_2$ and $X_3$. For the first setting (the graph on the left in Figure \ref{fig: fig3}), we have the joint distributions
\begin{equation*}
    \mathbb{P}(1,1,1) =  \mathbb{P}(-1,-1,-1) = \frac{e^{2a}}{Z},
\end{equation*}
\begin{equation*}
    \mathbb{P}(1,1,-1) =  \mathbb{P}(1,-1,1) = \mathbb{P}(-1,1,-1) = \mathbb{P}(-1,-1,1) = \frac{1}{Z},
\end{equation*}
\begin{equation*}
    \mathbb{P}(1,-1,-1) =  \mathbb{P}(-1,1,1) = \frac{e^{-2a}}{Z},
\end{equation*}
where 
\begin{equation*}
    Z = \frac{1}{4+2(e^{2a}e^{-2a})}
\end{equation*}
is the normalizing term.  The joint distribution of $X_2$ and $X_3$ can be found to be 
\begin{equation*}
    \mathbb{P}_{X_2, X_3}(1,1) =  \mathbb{P}_{X_2, X_3}(-1,-1) = \frac{e^{2a}+e^{-2a}}{Z},
\end{equation*}
\begin{equation*}
    \mathbb{P}_{X_2, X_3}(1,-1) =  \mathbb{P}_{X_2, X_3}(-1,1) = \frac{2}{Z}.
\end{equation*}

Then we can derive that the conditional expectation of $X_1$ given $X_2$ and $X_3$ are
\begin{equation*}
    \mathbb{E}[X_1|X_2 = 1, X_3 = 1] = \frac{e^{2a} - e^{-2a}}{e^{2a}+e^{-2a}},
\end{equation*}
\begin{equation*}
    \mathbb{E}[X_1|X_2 = 1, X_3 = -1] = \mathbb{E}[X_1|X_2 = -1, X_3 = 1] = 0,
\end{equation*}
\begin{equation*}
    \mathbb{E}[X_1|X_2 = -1, X_3 = -1] = \frac{e^{-2a} - e^{2a}}{e^{2a}+e^{-2a}}.
\end{equation*}

For the other 3 graph structures (middle and right in Figure \ref{fig: fig3}), we can derive the probabilities and expectations similarly. Also note that for node $r = 1$, these two graph structures are symmetric to $X_1$. Finally with all these values we have, plugging them into the infinity norm in (\ref{rephrase}) and setting it to be small (e.g., $0$ in this illustrative example), we have that $a \approx 1.75$.

Since our setting design is symmetric for all $X_1$, $X_2$, $X_3$, the same result hold when $r = 2$
 or $3$.

\section{Proof of Theorem \ref{thm: thm1}}\label{thm1pf}
\subsection{Primal-dual Witness for Recovery of the Latent Common Graph}
The main technique we use throughout the theoretical proof is the primal-dual witness approach \cite{wainwright2006sharp, ravikumar2010high} that relies on the Karush-Kuhn Tucker conditions in optimization and concentration inequalities in learning theory. Essentially, it constructs a \emph{primal-dual pair}, i.e. a primal solution $\hat{\theta} \in \mathbb{R}^{p-1}$ and an associated sub-gradient vector $\hat{z} \in  \in \mathbb{R}^{p-1}$ as a dual solution so that the sub-gradient optimality conditions in the convex program (\ref{improper}) are satisfied. We show that under the conditions on $(n,p,d,K)$ stated in the theorem, the primal-dual pair $(\hat{\theta}, \hat{z})$ can be constructed to act as a $witness$ that guarantees the method correctly recovers the structure of the graph parametrized by the true common parameter $\bar{\theta}$.

For the convex program (\ref{improper}), the zero sub-gradient optimality condition \cite{rockafellar2015convex} has the form of
\begin{equation}\label{optimal_cond}
    \nabla \ell(\hat{\theta}) + \lambda \hat{z} = 0,
\end{equation}
where the dual (the sub-gradient vector) $\hat{z} \in \mathbb{R}^{p-1}$ must satisfy
\begin{equation}\label{lessthan1_con}
    \text{sign}(\hat{z}_{rt}) = \text{sign}(\hat{\theta}_{rt}) \quad  \text{if }  \hat{\theta}_{rt} \neq 0 \quad \text{and} \quad |\hat{z}_{rt}| \leq 1 \text{ otherwise}.
\end{equation}
By convexity, a pair $(\hat{\theta}, \hat{z}) \in  \mathbb{R}^{p-1} \times  \mathbb{R}^{p-1}$ is a primal-dual optimal solution to the convex program \emph{if and only if} the two conditions (\ref{optimal_cond}) and (\ref{lessthan1_con}) are satisfied. Furthermore, this optimal primal-dual pair correctly specifies the signed neighborhood of node $r$ \emph{if and only if} 
\begin{equation}\label{samesign_cond}
    \text{sign}(\hat{z}_{rt}) = \text{sign}(\bar{\theta}_{rt}) \quad \forall (r,t) \in S,
\end{equation}
and 
\begin{equation}\label{cond0}
    \hat{\theta}_{rt} = 0 \quad \forall (r,t) \in S^c.
\end{equation}
The $\ell_1$-regularized logistic regression problem (\ref{improper}) is convex. The following lemma provides sufficient conditions for it to be \emph{strictly} convex and hence the uniqueness of the optimal solution, as well as the shared sparsity among optimal solutions.
\begin{lemma} (A generalization of Lemma 1 in \cite{ravikumar2010high}).\label{uniquelemma}
Suppose that there exists an optimal primal solution $\hat{\theta}$ with associated optimal dual vector $\hat{z}$ such that $\|\hat{z}_{S^c}\|_{\infty} < 1$. Then any optimal primal solution $\Tilde{\theta}$ must have $\Tilde{\theta}_{S^c} = 0$. Moreover, if the Hessian sub-matrix $[\nabla^2 \ell(\hat{\theta};\{\mathfrak{X}^n_1\}^K_1)]_{SS}$ is strictly positive definite for the loss function defined in the paper, then $\hat{\theta}$ is the unique optimal solution.
\end{lemma}

\begin{proof}
The proof follows exactly the same logic as that for Lemma 1 in \cite{ravikumar2010high}, except that the loss function in our case is one more generalized --- the average of the losses in each task, which does not change the property of strict convexity when it is present. To see this, note that the loss function in \cite{ravikumar2010high} corresponds to $\ell^{(k)}(\theta)$ we defined in the paper, the loss for each task in our case.
\end{proof} 

Based on Lemma \ref{uniquelemma}, we construct a primal-dual witness pair $(\hat{\theta}, \hat{z})$ with the following steps.

\paragraph{Step 1} We set $\hat{\theta}_S$ as the minimizer of the $\ell_1$-penalized likelihood
\begin{equation}\label{defhatthetas}
    \hat{\theta}_S = \argmin_{(\hat{\theta}_S,0)}\{\ell(\theta;\{\mathfrak{X}^n_1\}^K_1) + \lambda \|{\theta}_S\|_1\},
\end{equation}
and set $\hat{z}_S = \text{sign}(\hat{\theta}_S)$.

\paragraph{Step 2} We set $\hat{\theta}_{S^c} = 0$ so that condition (\ref{cond0}) holds.
\paragraph{Step 3} We obtain $\hat{z}_{S^c}$ from (\ref{optimal_cond}) by substituting in the values of $\hat{\theta}_S$ and $\hat{z}_S$. At this point, our construction satisfies conditions (\ref{optimal_cond}) and (\ref{cond0}).
\paragraph{Step 4} We need to show that the stated scaling of $(n,p,d,K)$ in Theorem \ref{thm: thm1} implies that, with high probability, the remaining conditions (\ref{lessthan1_con}) and (\ref{samesign_cond}) are satisfied. 

The last step is most challenging and is the goal of the majority of our proof. Our analysis guarantees that $\|\hat{z}_{S^c}\|_{\infty} < 1$ with high probability. Another condition to be satisfied is the positive definiteness stated in Lemma \ref{uniquelemma}, for which  by Assumptions \ref{auxa1} and \ref{auxa2}, we prove that the sub-matrix of the sample Fisher information matrix is strictly positive definite with high probability, so that the primal solution $\hat{\theta}$ is guaranteed to be unique. The next two subsections contribute exactly to these two parts of the proof.

\subsection{Uniform Convergence of Sample Information Matrices in Auxiliary Tasks}\label{uniconv}
To satisfy the condition of positive definiteness in Lemma \ref{uniquelemma} and to lay the foundation for the analysis under the assumptions of the sample information matrix of having bounded eigenvalues in the next subsection \ref{from_sample}, we aim to prove here that if the dependency and incoherence conditions from Assumptions \ref{auxa1} and \ref{auxa2} are imposed on the \emph{population} Fisher information matrix then under the specified scaling of $(n,p,d,K)$, analogous bounds hold for the \emph{sample} Fisher information matrix with probability converging to one.

Recall the definition of the \emph{population} Fisher information matrix (dropping the subscript $r$) from Section \ref{sectionauxass}, we have (see (\ref{calcQ})):
\begin{equation}
    \bar{Q} = \frac{1}{K} \sum^K_{k=1}\mathbb{E}[\eta(X^{(k)}; \bar{\theta})X^{(k)}_{\setminus r} {(X^{(k)}_{\setminus r})}^T],
\end{equation}

and its sample counterpart, i.e., the \emph{sample} Fisher information matrix is defined as 
\begin{equation}\label{defQN}
    Q^N := \hat{\mathbb{E}}[-\nabla^2 \ell(\bar{\theta}_{\setminus r};\{\mathfrak{X}^n_1\}^K_1)] = \frac{1}{K} \sum^K_{k=1} \frac{1}{n}\sum^n_{i=1}\eta(x^{(k)}_i;\bar{\theta}) x^{(k)}_{i,\setminus r} {(x^{(k)}_{i,\setminus r})}^T.
\end{equation}
Here the $\mathbb{E}$ in $\bar{Q}$ is the population expectation under the joint distribution of the randomness in the model parameters $\{\Delta^{(k)}\}^K_{k=1}$ and the random samples $\{\mathfrak{X}^n_1\}^K_1$ for the $K$ auxiliary tasks, while $\hat{\mathbb{E}}$ in $Q^N$ denotes the empirical expectation, and the variance function is defined in (\ref{calcQ}).

\subsubsection{Uniform Convergence for Dependence Assumption}
For the dependence assumption, we show that the eigenvalue bounds in Assumptions \ref{auxa1}  hold with high probability for sample Fisher information matrix and sample covariance matrices in the following two lemmas:
\begin{lemma}\label{mineigen}
Suppose that Assumption \ref{auxa1} holds for the population Fisher information matrix $\bar{Q}$ and the \emph{pooled} population covariance matrix $\mathbb{E}(\frac{1}{K} \sum^K_{k=1}X^{(k)} {(X^{(k)})}^T)$. For any $\delta > 0$ and some fixed constants A and B, we have 

\begin{equation}\label{bdmineigen}
    \mathbb{P}[\Lambda_{\min} (Q^N_{SS}) \leq C_{\min} - \delta] \leq 2 \exp \left( -A \frac{\delta^2 nK}{d^2} + B \log(d) \right),
\end{equation}
and
\begin{equation}\label{bdmaxeigen}
    \mathbb{P} \left[\Lambda_{\max} \left[\frac{1}{K} \sum^K_{k=1}\frac{1}{n} \sum^n_{i=1} x^{(k)}_{i, \setminus r} {(x^{(k)}_{i, \setminus r})}^T \right] \geq D_{\max} - \delta \right] \leq 2 \exp \left( -A \frac{\delta^2 nK}{d^2} + B \log(d) \right).
\end{equation}
\end{lemma}

The proof of this lemma is in Section \ref{pfmineigen}.

\subsubsection{Uniform Convergence for Incoherence Assumption}
The following lemma is the analog for the incoherence assumption in Assumption \ref{auxa2} showing that the scaling of $(n,p,d,K)$ given in Theorem \ref{thm: thm1} guarantees that population incoherence implies sample incoherence.

\begin{lemma}\label{incoh}
If the \emph{pooled} population covariance satisfies $\vertiii{\bar{Q}_{S^c S} (\bar{Q}_{S S})^{-1}}_{\infty} \leq 1-\alpha$ with parameter $\alpha \in (0,1]$, then the sample matrix satisfies an analogous version, with high probability in the sense that 
\begin{equation}\label{QTbound}
    \mathbb{P}\left[ \vertiii{Q^N_{S^c S}{(Q^N_{SS})}^{-1}}_{\infty} \geq 1- \frac{\alpha}{2} \right] \leq  \exp\left(-B\frac{nK}{d^3}+ \log (p)\right),
\end{equation}
for some fixed constant $B$.
\end{lemma}
The proof of this lemma is in Section \ref{pfincoh}.

\subsection{Analysis under Assumptions of Sample Information Matrices in Auxiliary Tasks}\label{from_sample}

With the incoherence and dependence conditions guaranteed with high probability (proved in Section \ref{uniconv}), we then begin to establish model selection consistency when assumptions are imposed directly on the sample Fisher information matrix $Q^N$ as opposed to $\bar{Q}$. Recalling the definition (\ref{defQN}) of the sample Fisher information matrix $Q^N$, we define the "good event" 
\begin{equation}
    \mathcal{M}(\{\mathfrak{X}^n_1\}^K_1):= \{\{\mathfrak{X}^n_1\}^K_1 \in \{-1,+1\}^{K \times n \times p}|Q^N \text{ satisfies Assumptions } \ref{auxa1} \text{ and } \ref{auxa2} \}.
\end{equation} 
As in the statement of Theorem \ref{thm: thm1}, the quantities $L$ and $c_1$ refer to constants independent of $(n,p,d, K).$ With this notation, we have the following:

\begin{proposition}\label{prop1} (Fixed design for auxiliary tasks).
If the event $\mathcal{M}(\{\mathfrak{X}^n_1\}^K_1)$ holds, the sample size per task and number of tasks satisfy $nK > Ld^2 \log p$, and the regularization parameter is chosen such that $\lambda \geq \beta \sqrt{\frac{\log p}{nK}}$ for some fixed constant $\beta > 0$, then for recovering the true common parameter vector $\bar{\theta}$ of the latent common graph, with probability at least $1-6 \exp(-c_\lambda^2 nK) \rightarrow 1$ for some constant $c > 0$, the following properties hold,

(a) For each node $r \in V$, the $\ell_1$-regularized logistic regression for the improper estimation of $\bar{\theta}$ has a unique solution, and so uniquely specifies a signed neighborhood $\hat{N}_{\pm}(r).$

(b) For each $r \in V$, the estimated signed neighborhood vector $\hat{N}_{\pm}(r)$ correctly excludes all edges not in the true support union. Moreover, it correctly includes all edges with $| \bar{\theta}_{rt}| \geq \frac{10}{C_{\min}}\sqrt{d} \lambda.$
\end{proposition}

Intuitively, this result guarantees that if the sample Fisher information matrix is "good", then the probability of success for the recovery of the underlying latent graph parametrized by the true common parameter $\bar{\theta}$ converges to 1 at the specified rate. The following subsection is devoted to the proof of Proposition \ref{prop1}.

\subsubsection{Key Technical Results in the Proof of Proposition \ref{prop1}}\label{dualfeas}
We follow the steps of primal-dual witness as stated at the beginning of Section \ref{thm1pf}. Since the key is to guarantee the strict dual feasibility $\|\hat{z}_{S^c}\|_{\infty} < 1$ with high probability in \textbf{Step 4}, we make a series of deliberate constructions to find out the explicit expression of $\|\hat{z}_{S^c}\|_{\infty}$ and try to bound it.

Starting from the stationarity condition in (\ref{optimal_cond}): $\nabla \ell(\hat{\theta}; \{\mathfrak{X}^n_1\}^K_1) + \lambda \hat{z} = 0,$ adding to both sides 
\begin{equation}\label{defWN}
    W^N := -\nabla \ell(\bar{\theta}; \{\mathfrak{X}^n_1\}^K_1) , 
\end{equation}
we get
\begin{equation}\label{addWN}
    \nabla \ell(\hat{\theta}; \{\mathfrak{X}^n_1\}^K_1) - \nabla \ell(\bar{\theta}; \{\mathfrak{X}^n_1\}^K_1)=  W^N - \lambda \hat{z}.
\end{equation}
Note that $W^N$ is just a shorthand notation for the $(p-1)$-dimensional score function. Then, applying the mean-value theorem coordinate-wise to the expansion (\ref{addWN}) gives
\begin{equation}
    \nabla^2 \ell(\bar{\theta}; \{\mathfrak{X}^n_1\}^K_1) [\hat{\theta} - \bar{\theta}] = W^N - \lambda \hat{z} + R^N,
\end{equation}
where the remainder term takes the form 
\begin{equation}
    R_j^N=-[\nabla^2 \ell(\theta^{(j)}; \{\mathfrak{X}^n_1\}^K_1) - \nabla^2 \ell(\bar{\theta}; \{\mathfrak{X}^n_1\}^K_1)]^T_j(\hat{\theta} - \bar{\theta}),
\end{equation}
with $\theta^{(j)}$ being a parameter vector on the line between $\bar{\theta}$ and $\hat{\theta}$, and with $[\cdot ]^T_j$ denoting the $j$-th row of the matrix.

Recalling our shorthand notation $Q^N = -\nabla^2 \ell(\bar{\theta}; \{\mathfrak{X}^n_1\}^K_1) $ and the fact that we have set $\hat{\theta}_{S^c} = 0$ in our primal-dual construction:
\begin{equation}
    \begin{cases} -Q^N_{S^c S}[\hat{\theta}_S - \bar{\theta}_S] = W^N_{S^c} -\lambda \hat{z}_{S^c} +R^N_{S^c} \\
-Q^N_{S S}[\hat{\theta}_S - \bar{\theta}_S] = W^N_{S} -\lambda \hat{z}_{S} +R^N_{S} \end{cases}.
\end{equation}
Since the matrix $Q^N_{SS}$ is invertible by assumption, it can be re-written as 
\begin{equation}
    Q^N_{S^cS}(Q^N_{SS})^{-1}[W^N_{S} -\lambda \hat{z}_{S} +R^N_{S}] = W^N_{S^c} -\lambda \hat{z}_{S^c} +R^N_{S^c},
\end{equation}
by using the common parts $\hat{\theta}_S - \bar{\theta}_S$ in the equations. Rearranging yields:
\begin{equation}
    \hat{z}_{S^c} =\frac{1}{\lambda} [W^N_{S^c} + R^N_{S^c}]- \frac{1}{\lambda} Q^N_{S^cS}(Q^N_{SS})^{-1}[W^N_{S} + R^N_{S}] + Q^N_{S^cS}(Q^N_{SS})^{-1}\hat{z}_S
\end{equation}
By the assumptions $\vertiii{Q^N_{S^c S} (Q^N_{SS})^{-1}}_{\infty} \leq 1 - \alpha$, and the fact that $\|{\hat{z}_{S}}\|_{\infty} = 1$, we have
\begin{equation}
    \|{\hat{z}_{S^c}}\|_{\infty}  \leq (1-\alpha) + (2-\alpha)\left[\frac{\|R^N\|_{\infty}}{\lambda} + \frac{\|W^N\|_{\infty}}{\lambda}\right].
\end{equation}

\paragraph{Strict Dual Feasibility.}
Now, to satisfy the \emph{strict dual feasibility} $\|\hat{z}_{S^c}\|_{\infty} < 1$, we need to bound $\frac{\|W^N\|_{\infty}}{\lambda}$ and $\frac{\|R^N\|_{\infty}}{\lambda}$. The following two lemmas show that $\frac{\|W^N\|_{\infty}}{\lambda}$ decays to $0$ at an exponential rate and $\|R^N\|_{\infty}$ can be bounded deterministically accordingly under some conditions.

\begin{lemma}\label{WN} (Decaying behavior of $W^N$).
For the specified mutual incoherence parameter  $\alpha \in (0,1]$ and a fixed constant $c$, we have 
\begin{equation}\label{WNeq}
    \mathbb{P}\left(\frac{2-\alpha}{\lambda} \|W^N\|_{\infty} > \frac{\alpha }{4} \right) \leq 6 \exp \left(- \frac{\alpha^2 \lambda^2}{c(2-\alpha)^2} nK + log(p) \right),
\end{equation}
which converges to 0 at rate $\exp(-c'\lambda^2 nK )$ for some fixed constant $c'$, as long as  $\lambda \geq \frac{\sqrt{2c}(2-\alpha)}{\alpha} \sqrt{\frac{\log p}{nK}}.$
\end{lemma}
The proof of this lemma is in Section \ref{pfWN}
\begin{lemma}\label{RN} (Control on the remainder term $R^N$).
If $\lambda d \leq \frac{C_{\min}^2}{100 D_{\max}} \frac{\alpha}{2-\alpha}$ and $\|W^N\|_{\infty} \leq \frac{\lambda}{4}$, then 
\begin{equation}\label{RNeq}
    \frac{\|R^N\|_{\infty}}{\lambda} \leq \frac{25D_{\max}}{C_{\min}^2} \lambda d \leq \frac{\alpha}{4 (2-\alpha)}.
\end{equation}
\end{lemma}
The proof of this lemma is in Section \ref{pfRN}.

Next, applying Lemmas \ref{WN} and \ref{RN}, we have the \emph{strict dual feasibility} as

\begin{equation*}
    \|{\hat{z}_{S^c}}\|_{\infty} \leq (1-\alpha) + \frac{\alpha}{4} + \frac{\alpha}{4} =  1 - \frac{\alpha}{2},
\end{equation*}
with probability converging to one.

\paragraph{Correct Sign Recovery.}
For the statement of \emph{correct sign recovery} in Proposition \ref{prop1}, we show here that our primal sub-vector $\hat{\theta}_S$ defined by (\ref{defhatthetas}) satisfies sign consistency $\text{sign}(\hat{\theta}_S) = \text{sign}(\bar{\theta}_S)$, which suffices to show that 
\begin{equation*}
    \|\hat{\theta}_S - \bar{\theta}_S\|_{\infty} \leq \frac{\bar{\theta}_{\min}}{2},
\end{equation*}
where $\bar{\theta}_{\min} := \min_{(r,t) \in S}|\bar{\theta}_{rt}|$. The following lemma is used in the proof here, which establishes that the sub-vector $\hat{\theta}_S$ is an $\ell_2$-consistent estimate of the true common sub-vector $\bar{\theta}_S$.

\begin{lemma}\label{l2cons}
($\ell_2$-consistency of primal sub-vector). If $\lambda d \leq \frac{C^2_{\min}}{10 D_{\max}}$ and $\|W^N\|_{\infty} \leq \frac{\lambda}{4}$, then 
\begin{equation}
    \|\hat{\theta}_S - \bar{\theta}_S \|_2 \leq \frac{5}{C_{\min}}\sqrt{d}\lambda
\end{equation}

\end{lemma}
The proof of this lemma is in Section \ref{pfl2cons}.

By Lemma \ref{l2cons}, we can write 
\begin{equation*}
\begin{split}
    \frac{2}{\bar{\theta}_{\min}}\|\hat{\theta}_S - \bar{\theta}_S\|_{\infty} &\leq \frac{2}{\bar{\theta}_{\min}}\|\hat{\theta}_S - \bar{\theta}_S\|_2\\
    & \leq \frac{2}{\bar{\theta}_{\min}}\frac{5}{C_{\min}}\sqrt{d}\lambda,    
\end{split}
\end{equation*}
which is less than 1 as long as $| \bar{\theta}_{rt}| \geq \frac{10}{C_{\min}}\sqrt{d} \lambda$.

Now it is clear that the uniform convergence of sample information matrices (in Section \ref{uniconv}) together with Proposition \ref{prop1} (from Section \ref{from_sample}) completes the proof of Theorem \ref{thm: thm1}.

\section{Proof of Theorem \ref{thm: thm2}}\label{thm2pf}
 We have supposed that we have recovered the true support union $S$ from our estimate for the true common parameter, $\hat{\theta}.$ The constraint in (\ref{restrict}) then enables us to convert the problem into one without the restriction and with a parameter of dimension $|S_r|$ with $|S_r| \leq d$ for all $r \in V$, for we can combine the constraint straightforward into the minimization problem. With some abuse of notation using $S$ to denote $S_r$ as before, we can write
\begin{equation}\label{restrict1}
     \hat{\theta}^{(K+1)}_{S} = \argmin_{{\theta_{S} \in \mathbb{R}^{p-1}}} \Big\{ \ell^{(K+1)}(\theta_{S}; \{\mathfrak{X}^{n^{(K+1)}}_{1,S}\}^{(K+1)})+ \lambda^{(K+1)} {\|{\theta_{S}}\|}_1 \Big\},
\end{equation}
and $\hat{\theta}^{(K+1)}_{S^c} = 0$, since we know that 
\begin{equation}\label{ssubsets}
    S^{(K+1)} \subseteq S.
\end{equation} This simplifies the problem to a great extent, and our proof henceforth takes on a similar pattern as the proof without restriction in \cite{ravikumar2010high}, but with reduced dimensions.

\subsection{Primal-dual Witness for Graph Recovery in the Novel Task}
We again use the primal-dual witness approach \cite{wainwright2006sharp, ravikumar2010high} as stated in the proof of Theorem \ref{thm: thm1}. See Section \ref{thm1pf}. With the loss function, parameter and data changed for only one task --- the novel task.

For the convex program (\ref{restrict1}), the zero sub-gradient optimality condition \cite{rockafellar2015convex} has the form of

\begin{equation}\label{2optimal_cond}
    \nabla \ell^{(K+1)}(\hat{\theta}^{(K+1)}_S) + \lambda^{(K+1)} \hat{z}^{(K+1)}_S = 0,
\end{equation}

where the dual (the sub-gradient vector) $\hat{z}^{(K+1)}_S \in \mathbb{R}^{|S_r|}$ must satisfy

\begin{equation}\label{2lessthan1_con}
    \text{sign}(\hat{z}^{(K+1)}_{rt}) = \text{sign}(\hat{\theta}^{(K+1)}_{rt}) \quad  \text{if }  \hat{\theta}_{rt} \neq 0 \quad \text{and} \quad |\hat{z}_{rt}| \leq 1 \text{ otherwise}.
\end{equation}

By convexity, a pair $(\hat{\theta}^{(K+1)}_S, \hat{z}^{(K+1)}_S) \in  \mathbb{R}^{|S_r|} \times  \mathbb{R}^{|S_r|}$ is a primal-dual optimal solution to the convex program \emph{if and only if} the two conditions (\ref{2optimal_cond}) and (\ref{2lessthan1_con}) are satisfied. Furthermore, this optimal primal-dual pair correctly specifies the signed neighborhood of node $r$ \emph{if and only if} 
\begin{equation}\label{2samesign_cond}
    \text{sign}(\hat{z}^{(K+1)}_{rt}) = \text{sign}(\bar{\theta}^{(K+1)}_{rt}) \quad \forall (r,t) \in S^{(K+1)},
\end{equation}
and 
\begin{equation}\label{2cond0}
    \hat{\theta}^{(K+1)}_{rt} = 0 \quad \forall (r,t) \in [S^{(K+1)}]^c.
\end{equation}

For this restricted problem, we have a similar lemma as \ref{uniquelemma} to for the uniqueness of the solution and shared sparsity.
\begin{lemma} (Lemma 1 in \cite{ravikumar2010high} with reduced dimensions).\label{uniquelemma2}
Suppose that there exists an optimal primal solution $\hat{\theta}^{(K+1)}_S$ with associated optimal dual vector $\hat{z}^{(K+1)}_S$ such that $\|\hat{z}^{(K+1)}_{[S^{(K+1)}]^c}\|_{\infty} < 1$. Then any optimal primal solution $\Tilde{\theta}^{(K+1)}_S$ must have $\Tilde{\theta}^{(K+1)}_{[S^{(K+1)}]^c} = 0$. Moreover, if the Hessian sub-matrix $[\nabla^2 \ell^{(K+1)}(\hat{\theta}^{(K+1)}_S; \{\mathfrak{X}^{n^{(K+1)}}_{1,S}\}^{(K+1)})]_{S^{(K+1)}S^{(K+1)}}$ is strictly positive definite, then $\hat{\theta}^{(K+1)}_S$ is the unique optimal solution.
\end{lemma}
\begin{proof}
See proof of Lemma 1 in \cite{ravikumar2010high}). The case in this convex program has a loss function $\ell^{(K+1)}$ carrying the same meaning as those in \cite{ravikumar2010high}), only with the dimensions of the parameter vector and our samples reduced since they are restricted to the true support union $S$ (see (\ref{ssubsets})).
\end{proof}

Based on Lemma \ref{uniquelemma2}, we construct a primal-dual witness pair $(\hat{\theta}^{(K+1)}_S, \hat{z}^{(K+1)}_S)$ with the following steps.

\paragraph{Step 1} We set $\hat{\theta}^{(K+1)}_{S^{(K+1)}}$ as the minimizer of the $\ell_1$-penalized likelihood
\begin{equation}\label{2defhatthetas}
    \hat{\theta}_{S^{(K+1)}} = \argmin_{(\hat{\theta}_{S^{(K+1)}},0)}\{\ell(\theta_S;\{\mathfrak{X}^n_{1,S}\}^K_1) + \lambda^{(K+1)} \|{\theta}_{S^{(K+1)}}\|_1\},
\end{equation}
and set $\hat{z}^{(K+1)}_{S^{(K+1)}} = \text{sign}(\hat{\theta}^{(K+1)}_{S^{(K+1)}})$.

\paragraph{Step 2} We set $\hat{\theta}^{(K+1)}_{[S^{(K+1)}]^c} = 0$ so that condition (\ref{2cond0}) holds.
\paragraph{Step 3} We obtain $\hat{z}^{(K+1)}_{[S^{(K+1)}]^c}$ from (\ref{optimal_cond}) by substituting in the values of $\hat{\theta}^{(K+1)}_{S^{(K+1)}}$ and $\hat{z}^{(K+1)}_{S^{(K+1)}}$ so that our construction satisfies conditions (\ref{2optimal_cond}) and (\ref{2cond0}).
\paragraph{Step 4} We need to show that the stated scaling of $(n^{(K+1)},d)$ in Theorem \ref{thm: thm1} implies that, with high probability, the remaining conditions (\ref{2lessthan1_con}) and (\ref{2samesign_cond}) are satisfied. 

 Our analysis in the last step guarantees that $\|\hat{z}^{(K+1)}_{[S^{(K+1)}]^c}\|_{\infty} < 1$ with high probability. Another condition to be satisfied is the positive definiteness stated in Lemma \ref{uniquelemma2}, for which  by Assumptions \ref{novela1} and \ref{novela2}, we prove that the sub-matrix of the sample Fisher information matrix is strictly positive definite with high probability, so that the primal solution $\hat{\theta}^{(K+1)}_S$ is guaranteed to be unique. The next two subsections contribute to these two parts of the proof.

\subsection{Uniform Convergence of Sample Information Matrices in Novel Task}\label{2uniconv}
To satisfy the condition of positive definiteness in Lemma \ref{uniquelemma2} and to prepare for the analysis under the assumptions of the sample information matrix of having bounded eigenvalues in the next subsection \ref{2from_sample}, we will prove in this subsection that if the dependency and incoherence conditions from Assumptions \ref{novela1} and \ref{novela2} are imposed on the \emph{population} Fisher information matrix then under the specified scaling of $(n^{(K+1)},d)$, analogous bounds hold for the \emph{sample} Fisher information matrix with probability converging to one.

Recall the definition of the \emph{population} Fisher information matrix (dropping the subscript $r$) from (\ref{sectionnovelass}), we have (see (\ref{calcQK1}):
\begin{equation}
    \bar{Q}^{(K+1)} = \mathbb{E}[\eta(X_S^{(K+1)}; \bar{\theta}_S^{(K+1)})X^{(K+1)}_S {(X^{(K+1)}_S)}^T].
\end{equation}

and its sample counterpart, i.e., the \emph{sample} Fisher information matrix is defined as 
\begin{equation}\label{2defQN}
    Q^{(K+1)}  := \hat{\mathbb{E}}[-\nabla^2 \ell^{(K+1)}(\bar{\theta}^{(K+1)}_S;\{\mathfrak{X}^n_{1,S}\}^K_1)] = \frac{1}{n^{(K+1)}}\sum^{n^{(K+1)}}_{i=1}\eta(x^{(K+1)}_{i,S};\bar{\theta}_S^{(K+1)}) x^{(K+1)}_{i,S} {(x^{(K+1)}_{i,S})}^T.
\end{equation}

Here the $\mathbb{E}$ in $\bar{Q}^{(K+1)}$ is the population expectation under the joint distribution of the randomness in the model parameter $\Delta^{(K+1)}$ and the random samples $\{\mathfrak{X}^n_1\}^{(K+1)}$ for the the novel task. $\hat{\mathbb{E}}$ in $Q^{(K+1)}$ denotes the empirical expectation, and the variance function is defined in (\ref{calcQ}).

\subsubsection{Uniform Convergence for Dependence Assumption}
For the dependence assumption, we show that the eigenvalue bounds in Assumptions \ref{novela1}  hold with high probability for sample Fisher information matrix and sample covariance matrices in the following two lemmas:
\begin{lemma}\label{2mineigen}
Suppose that Assumption \ref{novela1} holds for the population Fisher information matrix $\bar{Q}^{(K+1)}$ and population covariance matrix $\mathbb{E}(X^{(K+1)}_S {(X^{((K+1)}_S)}^T)$. For any $\delta > 0$ and some fixed constants A and B, we have 
\begin{equation}
    \mathbb{P} \left[\Lambda_{\min}(Q^{(K+1)}_{S^{(K+1)}S^{(K+1)}}) \leq C_{\min}^{(K+1)} - \delta \right] \leq 2 \exp \left( -A \frac{\delta^2 n^{(K+1)}}{d^2} + B \log(d) \right),
\end{equation}
and 
\begin{equation}\label{2maxeigneq}
    \mathbb{P} \left[\Lambda_{\max}(\frac{1}{n^{(K+1)}} \sum^{n^{(K+1)}}_{i=1} x^{(K+1)}_{i, S} {(x^{(K+1)}_{i, S})}^T) \geq D_{\max}^{(K+1)} - \delta \right]  \leq 2 \exp \left( -A \frac{\delta^2 n^{(K+1)}}{d^2} + B \log(d) \right)
\end{equation}

\end{lemma}

The proof of this lemma is in Section \ref{2pfmineigen}.

\subsubsection{Uniform Convergence for Incoherence Assumption}
The following lemma is the analog for the incoherence assumption in Assumption \ref{auxa2} showing that the scaling of $(n,p,d,K)$ given in Theorem \ref{thm: thm1} guarantees that population incoherence implies sample incoherence.

\begin{lemma}\label{2incoh}
If the population covariance satisfies $\vertiii{\bar{Q}^{(K+1),S}_{{[S^{(K+1)}]}^c S^{(K+1)}} (\bar{Q}^{(K+1)}_{S^{(K+1)} S^{(K+1)}})^{-1}}_{\infty} \leq 1-\alpha$ with parameter $\alpha \in (0,1]$, then the sample matrix satisfies an analogous version, with high probability in the sense that 
\begin{equation}\label{2incoheq}
\mathbb{P}[\vertiii{Q^{(K+1),S}_{{[S^{(K+1)}]}^c S^{(K+1)}}{(Q^{(K+1)}_{S^{(K+1)}S^{(K+1)}})}^{-1}}_{\infty} \geq 1- \alpha^{(K+1)}/2] \leq  \exp\left(-B\frac{n^{(K+1)}}{{d}^3}+ \log ({d})\right)
\end{equation}
for some fixed constant $B$.
\end{lemma}
The proof of this lemma is in Section \ref{2pfincoh}.

\subsection{Analysis under Assumptions of Sample Information Matrices}\label{2from_sample}

With the incoherence and dependence conditions guaranteed with high probability (proved in Section \ref{2uniconv}), we can begin to establish model selection consistency when assumptions are imposed directly on the sample Fisher information matrix $Q^{(K+1)}$ as opposed to $\bar{Q}^{(K+1)}$. Recalling the definition (\ref{calcQK1}) of the sample Fisher information matrix $Q^{(K+1)}$, we define a "good event" for the novel task
\begin{multline}
\mathcal{M}^{(K+1)}(\{\mathfrak{X}^{n^{(K+1)}}_{1,S}\}^{(K+1)})\\:= \{\{\mathfrak{X}^{n^{(K+1)}}_{1,S}\}^{(K+1)} \in \{-1,+1\}^{n^{(K+1)} \times |S_r|} |Q^{(K+1)} \text{ satisfies Assumptions } \ref{novela1} \text{ and } \ref{novela2} \}.
\end{multline}
 
As in the statement of Theorem \ref{thm: thm1}, the quantities $L$ and $c_2$ refer to constants independent of $(n^{(K+1)},p,d).$ With this notation, we have the following proposition:

\begin{proposition}\label{prop2} (Fixed design for novel task).
Suppose we have recovered the true support union $S$. If the event $\mathcal{M}^{(K+1)}(\{\mathfrak{X}^{n^{(K+1)}}_{1,S}\}^{(K+1)})$ holds, the sample size satisfy $n^{(K+1)} > Ld^2 \log d$, and the regularization parameter is chosen such that $\lambda \geq 16 \frac{(2-\alpha)}{\alpha} \sqrt{\frac{\log d}{n^{(K+1)}}}$, then for recovering the true common parameter vector $\bar{\theta}^{(K+1)}$ of the latent common graph, with probability at least $1-2 \exp(-c\lambda^2 n^{(K+1)}) \rightarrow 1$ for some constant $c>0$, the following properties hold,

(a) For each node $r \in V$, the $\ell_1$-regularized logistic regression for estimating $\bar{\theta}^{(K+1)}_S$ in the novel task, given data $\{\mathfrak{X}^{n^{(K+1)}}_1\}^{(K+1)}$has a unique solution $\hat{\theta}^{(K+1)}_S$, and so uniquely specifies a signed neighborhood $\mathcal{\hat{N}}^{(K+1)}_{\pm}(r) :=\{\text{sign}(\hat{\theta}^{(K+1)}_{ru})u|u \in V \setminus r, \hat{\theta}^{(K+1)}_{ru} \neq 0 \}$
    
(b) For each $r \in V$, the estimated signed neighborhood vector $\mathcal{\hat{N}}^{(K+1)}_{\pm}(r)$ correctly excludes all edges not in the true neighborhood $\mathcal{N}^{(K+1)}_{\pm}(r) :=\{\text{sign}(\bar{\theta}^{(K+1)}_{ru})u|u \in V \setminus r, \bar{\theta}^{(K+1)}_{ru} \neq 0 \}$. Moreover, it correctly includes all edges with $| \bar{\theta}^{(K+1)}_{rt}| \geq \frac{10}{C_{\min}^{(K+1)}}\sqrt{d} \lambda^{(K+1)}.$
\end{proposition}

Loosely stated, this result guarantees that if the sample Fisher information matrix  is "good", then the probability of success for the recovery graph by converges to 1 at the specified rate. The following subsection is devoted to the proof of Proposition \ref{prop2}.

\subsubsection{Key Technical Results in the Proof of Proposition \ref{prop2}}\label{2dualfeas}
We follow the steps of primal-dual witness as stated at the beginning of Section \ref{thm2pf}. Since the key is to guarantee the strict dual feasibility $\|\hat{z}^{(K+1), S}_{{[S^{(K+1)}]}^c}\|_{\infty} < 1$ with high probability in \textbf{Step 4}, we first try to find out the explicit expression of $\|\hat{z}^{(K+1), S}_{{[S^{(K+1)}]}^c}\|_{\infty}$ and try to bound it.

Starting from the stationarity condition in (\ref{optimal_cond}): $\nabla \ell(\hat{\theta}^{(K+1)}_S ) + \lambda \hat{z}^{(K+1)}_S = 0,$ adding to both sides 
\begin{equation}
    W^{(K+1)} := -\nabla \ell^{(K+1)}(\bar{\theta}^{(K+1)}_S ), 
\end{equation}
noticing that $\mathbb{E}[W^{(K+1)}]= 0,$ and skipping writing down the sample $\{\mathfrak{X}^{n^{(K+1)}}_{1,S}\}^{(K+1)}$ in the loss function, we get

\begin{equation}\label{2addWN}
    \nabla \ell^{(K+1)}(\hat{\theta}^{(K+1)}_S ) - \nabla \ell^{(K+1)}(\bar{\theta}^{(K+1)}_S )=  W^{(K+1)} - \lambda^{(K+1)} \hat{z}^{(K+1)}_S.
\end{equation}
Note that $W^{(K+1)}$ is just a shorthand notation for the $|S_r|$-dimensional score function. Then, applying the mean-value theorem coordinate-wise to the expansion (\ref{2addWN}) gives
\begin{equation}
    \nabla^2 \ell^{(K+1)}(\bar{\theta}^{(K+1)}_S ) [\hat{\theta}^{(K+1)}_S - \bar{\theta}^{(K+1)}_S] = W^{(K+1)} - \lambda^{(K+1)} \hat{z}^{(K+1)}_S + R^{(K+1)},
\end{equation}
where the remainder term takes the form 
\begin{equation}
    R_j^{(K+1)}=-[\nabla^2 \ell^{(K+1)}(\theta^{(K+1)j}_S) - \nabla^2 \ell^{(K+1)}(\bar{\theta}^{(K+1)}_S)]^T_j(\hat{\theta}^{(K+1)}_S - \bar{\theta}^{(K+1)}_S),
\end{equation}
with $\theta^{(K+1)j}_S$ a parameter vector on the line between $\bar{\theta}^{(K+1)}_S$ and $\hat{\theta}^{(K+1)}_S$, and with $[\cdot ]^T_j$ denoting the $j$-th row of the matrix.

Recalling our shorthand notation $Q^{(K+1)} = -\nabla^2 \ell^{(K+1)}(\bar{\theta}^{(K+1)}_S; \{\mathfrak{X}^{n^{(K+1)}}_1\}^{(K+1)}) $ and the fact that we have set $\hat{\theta}^{(K+1),S}_{{[S^{(K+1)}]}^c} = 0$ in our primal-dual construction:
\begin{equation}
    \begin{cases} -Q^{(K+1)}_{{[S^{(K+1)}]}^c S^{(K+1)}}[\hat{\theta}_{S^{(K+1)}} - \bar{\theta}_{S^{(K+1)}}] = W^{(K+1)}_{{[S^{(K+1)}]}^c} -\lambda^{(K+1)}  \hat{z}^{(K+1), S}_{{[S^{(K+1)}]}^c} +R^{(K+1)}_{{[S^{(K+1)}]}^c} \\
-Q^{(K+1)}_{S^{(K+1)} S^{(K+1)}}[\hat{\theta}_{S^{(K+1)}} - \bar{\theta}_{S^{(K+1)}}] = W^{(K+1)}_{S^{(K+1)}} -\lambda^{(K+1)}  \hat{z}^{(K+1)}_{S^{(K+1)}} +R^{(K+1)}_{S^{(K+1)}} \end{cases}
\end{equation}
Since the matrix $Q^{(K+1)}_{S^{(K+1)} S^{(K+1)}}$ is invertible by assumption, it can be re-written as 
\begin{equation}
    \begin{split}
    Q^{(K+1)}_{{[S^{(K+1)}]}^c S^{(K+1)}}(Q^{(K+1)}_{S^{(K+1)} S^{(K+1)}})^{-1}[W^{(K+1)}_{S^{(K+1)}} -\lambda^{(K+1)}  \hat{z}^{(K+1)}_{S^{(K+1)}} +R^{(K+1)}_{S^{(K+1)}}] \\ 
    = W^{(K+1)}_{{[S^{(K+1)}]}^c} -\lambda^{(K+1)}  \hat{z}^{(K+1), S}_{{[S^{(K+1)}]}^c} +R^{(K+1)}_{{[S^{(K+1)}]}^c},
\end{split}
\end{equation}
by using the common parts $\hat{\theta}^{(K+1)}_{S^{(K+1)}} - \bar{\theta}^{(K+1)}_{S^{(K+1)}}$ in the equations. Rearranging yields:
\begin{align}
     \hat{z}^{(K+1), S}_{{[S^{(K+1)}]}^c} & =\frac{1}{\lambda^{(K+1)}} [W^{(K+1)}_{{[S^{(K+1)}]}^c} + R^{(K+1)}_{{[S^{(K+1)}]}^c}] \\
     & - \frac{1}{\lambda^{(K+1)}} Q^{(K+1)}_{{[S^{(K+1)}]}^c S^{(K+1)}}(Q^{(K+1)}_{S^{(K+1)} S^{(K+1)}})^{-1} [W^{(K+1)}_{S^{(K+1)}} + R^{(K+1)}_{S^{(K+1)}}] 
    \\ & + Q^{(K+1)}_{{[S^{(K+1)}]}^c S^{(K+1)}}(Q^{(K+1)}_{S^{(K+1)} S^{(K+1)}})^{-1} \hat{z}^{(K+1)}_{S^{(K+1)}}.
\end{align}
By the assumptions $\vertiii{Q^{(K+1)}_{{[S^{(K+1)}]}^c S^{(K+1)}}(Q^{(K+1)}_{S^{(K+1)} S^{(K+1)}})^{-1}}_{\infty} \leq 1 - \alpha^{(K+1)}$, and using the fact that $\|{\hat{z}^{(K+1)}_{S^{(K+1)}}}\|_{\infty} = 1$, we have
\begin{equation}
   \| \hat{z}^{(K+1), S}_{{[S^{(K+1)}]}^c}\|_{\infty}  \leq (1-\alpha^{(K+1)}) + (2-\alpha^{(K+1)})\left[\frac{\|R^{(K+1)}\|_{\infty}}{\lambda^{(K+1)}} + \frac{\|W^{(K+1)}\|_{\infty}}{\lambda^{(K+1)}}\right].
\end{equation}

\paragraph{Strict Dual Feasibility.}
Now, to satisfy the \emph{strict dual feasibility} $ \| \hat{z}^{(K+1), S}_{{[S^{(K+1)}]}^c}\|_{\infty} < 1$, we need to bound $\frac{\|W^{(K+1)}\|_{\infty}}{\lambda^{(K+1)}}$ and $\frac{\|R^{(K+1)}\|_{\infty}}{\lambda^{(K+1)}}$. The following two lemmas show that $\frac{\|W^{(K+1)}\|_{\infty}}{\lambda^{(K+1)}}$ decays to $0$ at an exponential rate and $\frac{\|R^{(K+1)}\|_{\infty}}{\lambda^{(K+1)}}$ can be bounded deterministically accordingly under some conditions.

\begin{lemma}\label{2WN} (Decaying behavior of $W^{(K+1)}$).
For the specified mutual incoherence parameter  $\alpha^{(K+1)} \in (0,1]$, we have 

\begin{equation}
    \mathbb{P}\left[\frac{2-\alpha^{(K+1)}}{\lambda_n} \|W^{(K+1)}\|_{\infty} > \frac{\alpha^{(K+1)} }{4} \right] \leq 2 \exp \left(- \frac{{(\alpha^{(K+1)})}^2 {(\lambda^{(K+1)})}^2}{128 (2-\alpha^{(K+1)})^2}n^{(K+1)} + \log(d) \right),
\end{equation}
which converges to 0 at rate $\exp(-c{(\lambda^{(K+1)})}^2 n^{(K+1)} )$  for some constant $c$, as long as  $\lambda^{(K+1)} \geq \frac{16(2-\alpha^{(K+1)})}{\alpha^{(K+1)}}\sqrt{\frac{\log(d)}{n^{(K+1)}}}.$
\end{lemma}
The proof of this lemma is in Section \ref{2pfWN}
\begin{lemma}\label{2RN} (Control on the remainder term $R^{(K+1)}$).
If $\lambda^{(K+1)} d \leq \frac{{(C_{\min}^{(K+1)})}^2}{100 D_{\max}^{(K+1)}} \frac{\alpha^{(K+1)}}{2-\alpha^{(K+1)}}$ and $\|W^{(K+1)}\|_{\infty} \leq \frac{\lambda^{(K+1)}}{4}$, then 
\begin{equation}
    \frac{\|R^{(K+1)}\|_{\infty}}{\lambda^{(K+1)}} \leq \frac{25D_{\max}^{(K+1)}}{{(C_{\min}^{(K+1)})}^2} \lambda^{(K+1)} d \leq \frac{\alpha^{(K+1)}}{4 (2-\alpha^{(K+1)})}
\end{equation}
\end{lemma}
The proof of this lemma is in Section \ref{2pfRN}.

Next, applying Lemmas \ref{2WN} and \ref{2RN}, we have the \emph{strict dual feasibility} as

\begin{align*}
    \| \hat{z}^{(K+1), S}_{{[S^{(K+1)}]}^c}\|_{\infty} & \leq (1-\alpha^{(K+1)}) + \frac{\alpha^{(K+1)}}{4} + \frac{\alpha^{(K+1)}}{4} \\
     & =  1 - \frac{\alpha^{(K+1)}}{2},
\end{align*}
with probability converging to one.

\paragraph{Correct Sign Recovery.}
For the statement of \emph{correct sign recovery} in Proposition \ref{prop2}, we show here that our primal sub-vector $\hat{\theta}^{(K+1)}_{S^{(K+1)}}$ defined by (\ref{2defhatthetas}) satisfies sign consistency $\text{sign}(\hat{\theta}^{(K+1)}_{S^{(K+1)}}) = \text{sign}(\bar{\theta}^{(K+1)}_{S^{(K+1)}})$, which suffices to show that 

\begin{equation*}
    \|\hat{\theta}^{(K+1)}_{S^{(K+1)}} - \bar{\theta}^{(K+1)}_{S^{(K+1)}}\|_{\infty} \leq \frac{\bar{\theta}^{(K+1)}_{\min}}{2},
\end{equation*}
where $\bar{\theta}^{(K+1)}_{\min} := \min_{(r,t) \in S^{(K+1)}}|\bar{\theta}^{(K+1)}_{rt}|$. The following lemma is used in the proof here, which establishes that the sub-vector $\hat{\theta}^{(K+1)}_{S^{(K+1)}}$ is an $\ell_2$-consistent estimate of the true common sub-vector $\bar{\theta}^{(K+1)}_{S^{(K+1)}}$.

\begin{lemma}\label{2l2cons}
($\ell_2$-consistency of primal sub-vector). If $\lambda^{(K+1)} d \leq \frac{C^2_{\min}}{10 {D_{\max}^{(K+1)}}}$ and $\|W^{(K+1)}\|_{\infty} \leq \frac{\lambda^{(K+1)}}{4}$, then 
\begin{equation}
    \|\hat{\theta}^{(K+1)}_{S^{(K+1)}} - \bar{\theta}^{(K+1)}_{S^{(K+1)}} \|_2 \leq {5}{{C_{\min}^{(K+1)}}}\sqrt{d}\lambda^{(K+1)}
\end{equation}
\end{lemma}
The proof of this lemma is in Section \ref{2pfl2cons}.

By Lemma \ref{2l2cons}, we can write 
\begin{equation*}
\begin{split}
    \frac{2}{\bar{\theta}^{(K+1)}_{\min}}\|\hat{\theta}^{(K+1)}_{S^{(K+1)}} - \bar{\theta}^{(K+1)}_{S^{(K+1)}}\|_{\infty} &\leq \frac{2}{\bar{\theta}^{(K+1)}_{\min}}\|\hat{\theta}^{(K+1)}_{S^{(K+1)}}- \bar{\theta}^{(K+1)}_{S^{(K+1)}}\|_2\\
    & \leq \frac{2}{\bar{\theta}^{(K+1)}_{\min}}\frac{5}{{C_{\min}^{(K+1)}}}\sqrt{d}\lambda^{(K+1)},    
\end{split}
\end{equation*}
which is less than 1 as long as $| \bar{\theta}^{(K+1)}_{rt}| \geq \frac{10}{{C_{\min}^{(K+1)}}}\sqrt{d} \lambda^{(K+1)}$.

Then we can use the uniform convergence of sample information matrices (in Section \ref{2uniconv}) and Proposition \ref{prop2} (from Section \ref{2from_sample}) to finish the proof of Theorem \ref{thm: thm2}.

\section{Proof of Lemmas for Theorem \ref{thm: thm1}}
\subsection{Proof of Lemmas for Uniform Convergence of Sample Information Matrices in Auxiliary Tasks}

\subsubsection{Proof of Lemma \ref{mineigen}} \label{pfmineigen}
\begin{proof}

The $(j,l)^{th}$ element of the difference matrix $Q^N(\bar{\theta}) - \bar{Q}(\bar{\theta})$ can be written as an i.i.d. sum of the form $Z_{jl} = \frac{1}{K} \sum^K_{k=1} \frac{1}{n} \sum^n_{i=1} Z_{jl,i}^{(k)} $, where each $Z_{jl,i}^{(k)}$ is zero-mean and bounded (in particular, $|Z_{jl,i}^{(k)}| \leq 4$). By the Azuma-Hoeffding's bound \cite{hoeffding1994probability}, for any indices $j,l = 1,...,d$ and for any $\varepsilon > 0$, we have 
\begin{equation} \label{QZ}
    \mathbb{P}[(Z_{jl})^2 \geq \varepsilon^2] = \mathbb{P}\big[ |\frac{1}{K} \sum^K_{k=1} \frac{1}{n} \sum^n_{i=1} Z_{jl,i}^{(k)}| \geq \varepsilon\big] \leq 2 \exp \left( -\frac{\varepsilon^2 nK}{32}\right) .
\end{equation}
By the Courant-Fischer variational representation \cite{horn2012matrix},
\begin{align*}
    \Lambda_{\min}({\bar{Q}_{SS}}) & = \min_{\|x\|_2 = 1}x^T\bar{Q}_{SS} x 
    \\ & =  \min_{\|x\|_2 = 1} \{x^T Q^N_{SS} x + x^T ( \bar{Q}_{SS} - Q^N_{SS}) x\}
    \\ & \leq y^T Q^N_{SS} y + y^T ( \bar{Q}_{SS} - Q^N_{SS}) y,
\end{align*}

where  $y \in \mathbb{R}^d$ is a unit-norm minimal eigenvector of $Q^N_{SS}$. Therefore, we have 
\begin{equation*}
    \Lambda_{\min}(Q^N_{SS}) \geq \Lambda_{\min}({\bar{Q}_{SS}}) - \vertiii{\bar{Q}_{SS} - Q^N_{SS}}_2 \geq C_{\min} - \vertiii{\bar{Q}_{SS} - Q^N_{SS}}_2.
\end{equation*}
Observe that 
\begin{equation*}
    \vertiii{Q^N_{SS} - \bar{Q}_{SS}}_2 \leq \left( \sum^{d}_{j=1} \sum^{d}_{l=1} {(Z_{jl})}^2\right)^{1/2}.
\end{equation*}
Setting $\varepsilon^2 = \delta^2/d^2$ in (\ref{QZ}) and applying the union bound over the $d^2$ index pairs $(j,l)$ then yields 
\begin{equation} \label{usebd1}
    \mathbb{P}[\vertiii{Q^N_{SS} - \bar{Q}_{SS}}_2 \geq \delta] \leq 2 \exp \left( -\frac{\delta^2nK}{32 d^2}+2 \log (d)\right).
\end{equation}
So, we have the first concentration inequality in Lemma \ref{pfmineigen}:
\begin{equation}
    \mathbb{P}[\Lambda_{\min}(Q^N_{SS}) \leq C_{\min} - \delta] \leq 2 \exp \left(  -\frac{\delta^2nK}{32 d^2}+2 \log (d)\right).
\end{equation}
Now, for the second concentration inequality about maximum eigenvalue of the sample covariance matrix, with the same reasoning from the Courant-Fischer variational representation \cite{horn2012matrix}, we have, for $1 \leq k \leq K,$
\begin{align*}
    \Lambda_{\max}(\mathbb{E}[\frac{1}{K} \sum^K_{k=1}X^{(k)}_{\setminus r}{(X^{(k)}_{\setminus r})}^T]) & = \max_{\|v\|_2=1} v^T \mathbb{E}[\frac{1}{K} \sum^K_{k=1}X^{(k)}_{\setminus r}{(X^{(k)}_{\setminus r})}^T] v
    \\ & =\max_{\|v\|_2=1} \{ v^T (\frac{1}{K} \sum^K_{k=1}\frac{1}{n} \sum^n_{i=1} x^{(k)}_{i, \setminus r} (x^{(k)}_{i, \setminus r})^T) v 
    \\& + v^T (\mathbb{E}[\frac{1}{K} \sum^K_{k=1}X^{(k)}_{\setminus r}{(X^{(k)}_{\setminus r})}^T] - \frac{1}{K} \sum^K_{k=1}\frac{1}{n} \sum^n_{i=1} x^{(k)}_{i, \setminus r} (x^{(k)}_{i, \setminus r})^T) v \} \\& \geq u^T (\frac{1}{K} \sum^K_{k=1}\frac{1}{n} \sum^n_{i=1} x^{(k)}_{i, \setminus r} (x^{(k)}_{i, \setminus r})^T) u \\&+ u^T (\mathbb{E}[\frac{1}{K} \sum^K_{k=1}X^{(k)}_{\setminus r}{(X^{(k)}_{\setminus r})}^T]
     - \frac{1}{K} \sum^K_{k=1}\frac{1}{n} \sum^n_{i=1} x^{(k)}_{i, \setminus r} (x^{(k)}_{i, \setminus r})^T) u,
\end{align*}
where $u \in \mathbb{R}^d$ is a unit-norm maximal eigenvector of $\frac{1}{K} \sum^K_{k=1}\frac{1}{n} \sum^n_{i=1} x^{(k)}_{i, \setminus r} (x^{(k)}_{i, \setminus r})^T$. Therefore, we have 

\begin{align*}
& \Lambda_{\max} (\frac{1}{K} \sum^K_{k=1}\frac{1}{n} \sum^n_{i=1} x^{(k)}_{i, \setminus r} (x^{(k)}_{i, \setminus r})^T) \\& \leq \Lambda_{\max}(\mathbb{E}[\frac{1}{K} \sum^K_{k=1}X^{(k)}_{\setminus r}{(X^{(k)}_{\setminus r})}^T])  + u^T (\frac{1}{K} \sum^K_{k=1}\frac{1}{n} \sum^n_{i=1} x^{(k)}_{i, \setminus r} (x^{(k)}_{i, \setminus r})^T - \mathbb{E}[\frac{1}{K} \sum^K_{k=1}X^{(k)}_{\setminus r}{(X^{(k)}_{\setminus r})}^T]) u  \\& \leq D_{\max} + \vertiii{(\frac{1}{K} \sum^K_{k=1}\frac{1}{n} \sum^n_{i=1} x^{(k)}_{i, \setminus r} (x^{(k)}_{i, \setminus r})^T - \mathbb{E}[\frac{1}{K} \sum^K_{k=1}X^{(k)}_{\setminus r}{(X^{(k)}_{\setminus r})}^T])}_2.
\end{align*}

The difference matrix $\frac{1}{K} \sum^K_{k=1}\frac{1}{n} \sum^n_{i=1} x^{(k)}_{i, \setminus r} (x^{(k)}_{i, \setminus r})^T - \mathbb{E}[\frac{1}{K} \sum^K_{k=1}X^{(k)}_{\setminus r}{(X^{(k)}_{\setminus r})}^T]$ can be written as an i.i.d. sum of the form $Y_{jl} = \frac{1}{K} \sum^K_{k=1}\frac{1}{n} \sum^n_{i=1} Y_{jl,i}^{(k)}$, where each $Y_{jl,i}^{(k)}$ is zero-mean and bounded (in particular, $|Y_{jl,i}^{(k)}| \leq 4$). By the Azuma-Hoeffding's bound \cite{hoeffding1994probability}, for any indices $j,l = 1,...,d$ and for any $\varepsilon > 0$, we have
\begin{equation}\label{EXY}
    \mathbb{P}[(Y_{jl})^2 \geq \varepsilon^2] = \mathbb{P}\big[ |\frac{1}{K} \sum^K_{k=1} \frac{1}{n} \sum^n_{i=1} Y_{jl,i}| \geq \varepsilon\big] \leq 2 \exp \left( -\frac{\varepsilon^2 nK}{32}\right).
\end{equation}
Observe that $$\vertiii{\frac{1}{K} \sum^K_{k=1}\frac{1}{n} \sum^n_{i=1} x^{(k)}_{i, \setminus r} (x^{(k)}_{i, \setminus r})^T - \mathbb{E}[\frac{1}{K} \sum^K_{k=1}X^{(k)}_{\setminus r}{(X^{(k)}_{\setminus r})}^T]}_2 \leq \left(\sum^d_{j=1} \sum^d_{l=1} (Y_{jl})^2\right)^{1/2}.$$
Setting $\varepsilon^2 = \delta^2/d^2$ in (\ref{EXY}), and applying the union bound over the $d^2$ index pairs $(j,l)$ then yields $$\mathbb{P}\left[\vertiii{\frac{1}{K} \sum^K_{k=1}\frac{1}{n} \sum^n_{i=1} x^{(k)}_{i, \setminus r} (x^{(k)}_{i, \setminus r})^T - \mathbb{E}[\frac{1}{K} \sum^K_{k=1}X^{(k)}_{\setminus r}{(X^{(k)}_{\setminus r})}^T]}_2 \geq \delta\right] \leq 2 \exp \left( -\frac{\delta^2nK}{32 d^2}+2 \log (d)\right).$$
So we have $$\mathbb{P} \big[  \Lambda_{\max} \big[ \frac{1}{n} \sum^n_{i=1} x^{(k)}_{i, \setminus r} (x^{(k)}_{i, \setminus r})^T)\big] \geq D_{\max} + \delta \big] \leq 2 \exp \left( -\frac{\delta^2nK}{32 d^2}+2 \log (d)\right)$$ as stated in the lemma.
\end{proof}

\subsubsection{Proof of Lemma \ref{incoh}}\label{pfincoh}
We begin the proof of this lemma by decomposing the sample matrix as the sum $Q^N_{S^cS}{(Q^N_{SS})}^{-1} = T_1 + T_2 + T_3 + T_4,$ where we define 
\begin{subequations}
\begin{align}
T_1 &:= \bar{Q}_{S^cS}[{(Q^N_{SS})}^{-1} - {(\bar{Q}_{SS})}^{-1}],\\
T_2 &:= [Q^N_{S^cS} - \bar{Q}_{S^cS}]{(\bar{Q}_{SS})}^{-1},\\
T_3 &:= [Q^N_{S^cS} - \bar{Q}_{S^cS}][{(Q^N_{SS})}^{-1} - {(\bar{Q}_{SS})}^{-1}],\\
T_4 &:= \bar{Q}_{S^cS}{(\bar{Q}_{SS})}^{-1}.
\end{align}
\end{subequations}
The fourth term is controlled by the incoherence assumption in Assumption \ref{auxa1}: $$\vertiii{T_4}_{\infty} = \vertiii{\bar{Q}_{S^cS}{(\bar{Q}_{SS})}^{-1}}_{\infty} \leq 1-\alpha.$$ If we can show that $\vertiii{T_i}_{\infty} \leq \frac{\alpha}{6}$ for the remaining indices $i=1,2,3$, then by our four-term decomposition and the triangle inequality, the sample version can satisfy the desired bound (\ref{QTbound}). To deal with these remaining terms, we make use of the following lemma:

\begin{lemma}\label{lemma of lemmas}
For any $\delta > 0$, and constants $B, B_1, B_2$, the following bounds hold 
\begin{subequations}
\begin{align}
    \mathbb{P}[\vertiii{Q^N_{S^cS} - \bar{Q}_{S^cS}}_{\infty} \geq \delta] & \leq 2 \exp \left( -B\frac{\delta^2nK}{d^2} + \log(d) + \log(p) \right),\label{eq:subeq1}\\
    \mathbb{P}[\vertiii{Q^N_{SS} - \bar{Q}_{SS}}_{\infty} \geq \delta] &\leq 2 \exp \left( -B\frac{\delta^2 nK}{d^2} + 2\log(d) \right),\label{eq:subeq2}\\
    \mathbb{P}[\vertiii{{(Q^N_{SS})}^{-1} - {(\bar{Q}_{SS})}^{-1}}_{\infty}  \geq \delta] &\leq 4 \exp \left( -B_1\frac{nK \delta^2}{d^3} + B_2 \log (d)\right).\label{eq:subeq3}
\end{align}
\end{subequations}
\end{lemma}
See Section \ref{pflmlm} for the proof of these claims.

\paragraph{Control of the first term}

For the first term, we re-factorize it as 
\begin{equation*}
    T_1 = \bar{Q}_{S^cS}{(\bar{Q}_{SS})}^{-1}[\bar{Q}_{SS} - Q^N_{SS}]{(Q^N_{SS})}^{-1}.
\end{equation*}
Then,
\begin{align*}
    \vertiii{T_1}_{\infty} & \leq \vertiii{\bar{Q}_{S^cS}{(\bar{Q}_{SS})}^{-1}}_{\infty}\vertiii{\bar{Q}_{SS} - Q^N_{SS}}_{\infty}\vertiii{{(Q^N_{SS})}^{-1}}_{\infty}
     \\ & \leq (1-\alpha)\vertiii{\bar{Q}_{SS} - Q^N_{SS}}_{\infty} \{\sqrt{d} \vertiii{{(Q^N_{SS})}^{-1}}_2\},
\end{align*}
where we have used the incoherence assumption in Assumption \ref{auxa1}. Using the bound (\ref{bdmaxeigen}) from Lemma (\ref{mineigen}) with $\delta = C_{\min}/2$, we have $\vertiii{{(Q^N_{SS})}^{-1}}_2  = {[\Lambda_{\min}(Q^N_{SS})]}^{-1} \leq \frac{2}{C_{\min}} $ with probability greater than $ 1- 2\exp ( -BnK/d^2+2 \log (d)).$ Next, applying the bound (\ref{eq:subeq2}) with $\delta = c / \sqrt{d},$ we conclude that with probability greater than $1-2\exp(-BnKc^2/d^3+2\log(d)),$ we have $$\vertiii{\bar{Q}_{SS} - Q^N_{SS}}_{\infty} \leq c/\sqrt{d}.$$
By choosing the constant $c > 0$ sufficiently small, we are guaranteed that $$\mathbb{P}[\vertiii{T_1}_{\infty} \geq \alpha/6] \leq 2 \exp \left(-B\frac{nKc^2}{d^3} + \log(d) \right).$$
\paragraph{Control of the second term}
To bound $T_2$, we first write 
\begin{align*}
    \vertiii{T_2}_{\infty} &\leq \sqrt{d}\vertiii{{(\bar{Q}_{SS})}^{-1}}_2 \vertiii{Q^N_{S^cS} - \bar{Q}_{S^cS}}_{\infty}
    \\& \leq \frac{\sqrt{d}}{C_{\min}} \vertiii{Q^N_{S^cS} - \bar{Q}_{S^cS}}_{\infty}.
\end{align*}
Then apply the bound (\ref{eq:subeq1}) with $\delta = \frac{\alpha}{6}\frac{C_{\min}}{\sqrt{d}}$ to conclude that $$\mathbb{P}[\vertiii{T_2}_{\infty} \geq \alpha/6] \leq 2 \exp \left(-B\frac{nK}{d^3} + \log(p) \right).$$

\paragraph{Control of the third term}
Finally, in order to bound the third term $T_3$, we apply the bounds (\ref{eq:subeq1}) and (\ref{eq:subeq3}), both with $\delta = \sqrt{\alpha/6} $ and use the fact that $\log d \leq \log p$ to conclude that 
\begin{equation}
    \mathbb{P}[\vertiii{T_3}_{\infty} \geq \alpha/6] \leq 4 \exp \left(-B\frac{nK}{d^3} + \log(p) \right).
\end{equation}
Putting together the four pieces, we conclude that $$\mathbb{P}\left[\vertiii{Q^N_{S^cS}{(Q^N_{SS})}^{-1}}_{\infty} \geq 1- \alpha/2\right] = \mathcal{O}\left( \exp\left(-B\frac{nK}{d^3}+ \log (p)\right)\right)$$
\subsection{Proof of Lemmas for Proposition \ref{prop1}}
\subsubsection{Proof of Lemma \ref{WN}}\label{pfWN}
\begin{proof}
By definition of $W^N$ (see (\ref{defWN})), we have
\begin{equation}
    \|W^N\|_{\infty} = \|\nabla \ell(\bar{\theta};\{\mathfrak{X}^n_1\}^K_1) \|_{\infty} = \|\frac{1}{K} \sum^K_{k=1} \nabla \ell^{(k)}(\bar{\theta}; \{\mathfrak{X}^n_1\}^K_1 ) \|_{\infty}.
\end{equation}
which can be decompose into two parts as follows 
\begin{multline}
    \|\nabla \ell(\bar{\theta};\mathfrak{X}^n_1\}^K_1) \|_{\infty} \\  \leq \|\underbrace{ \frac{1}{K} \sum^K_{k=1} \Big\{ \nabla \ell^{(k)}(\bar{\theta}; \{\mathfrak{X}^n_1\}^{(k)}) - \nabla \ell^{(k)}(\bar{\theta}^{(k)}; \{\mathfrak{X}^n_1\}^{(k)})\Big\}}_\text{$Y_1$}  \|_{\infty}  + \| \underbrace{\frac{1}{K} \sum^K_{k=1} \nabla \ell^{(k)}(\bar{\theta}^{(k)}; \{\mathfrak{X}^n_1\}^{(k)}) }_\text{$Y_2$}\|_{\infty}.
\end{multline}
We then bound the two terms $\|Y_1\|_{\infty}$ and $\|Y_2\|_{\infty}$ respectively.

\emph{Bounding $\|Y_2\|_{\infty}$}. 

Note that the conditional expectation of $Y_2$ given $\{\Delta^{(k)}\}^K_1$ is
\begin{align*}
    \mathbb{E}[Y_2|\{\Delta^{(k)}\}^K_1] & =  \mathbb{E}[\frac{1}{K} \sum^K_{k=1}\nabla \ell^{(k)}(\bar{\theta}^{(k)}; \{\mathfrak{X}^n_1\}^{(k)}) |\{\Delta^{(k)}\}^K_1]\\
    & = \frac{1}{K} \sum^K_{k=1} \mathbb{E}[\nabla \ell^{(k)}(\bar{\theta}^{(k)}; \{\mathfrak{X}^n_1\}^{(k)}) |\Delta^{(k)}]\\
    & = \frac{1}{K} \sum^K_{k=1} 0\\
    & = 0,
\end{align*}
where the second to last line comes from the fact that the expected gradient at the true parameter of each task is 0. This property can also be checked by expanding the expression of $Y_2$. Each entry of $Y_2$, denoted by $Y_{2,u}$, for $1 \leq u \leq p-1$, can be expressed as a sum of random variables $Z^{(k)}_{i,u}$: 
\begin{equation}\label{y2u}
    Y_{2,u}= \frac{1}{K} \sum^K_{k=1} \frac{1}{n} \sum^n_{i=1} Z^{(k)}_{i,u},
\end{equation}
where 
\begin{align*}
    Z^{(k)}_{i,u} & = x^{(k)}_{i,u} \{ x^{(k)}_{i,r} - \frac{\exp(\sum_{t \in V \setminus r} \bar{\theta}^{(k)}_{rt} x_{i,t}^{(k)}) - \exp(-\sum_{t \in V \setminus r} \bar{\theta}_{rt}^{(k)} x_{i,t}^{(k)})}{\exp(\sum_{t \in V \setminus r} \bar{\theta}^{(k)}_{rt} x_{i,t}^{(k)}) + \exp(-\sum_{t \in V \setminus r} \bar{\theta}_{rt}^{(k)} x_{i,t}^{(k)} )}\} \\
    & = x^{(k)}_{i,u} \{ x^{(k)}_{i,r} - \mathbb{P}_{\bar{\theta}^{(k)}}[X^{(k)}_r = 1 | x^{(k)}_{i, \setminus r}] + \mathbb{P}_{\bar{\theta}^{(k)}}[X^{(k)}_r = -1 | x^{(k)}_{i, \setminus r}] \}.
\end{align*}
We have the conditional expectation $\mathbb{E}[Z^{(k)}_{i,u}|\Delta^{(k)}]= 0$ by applying another law of total expectation \cite{weiss2005course} with the inner conditional expectation of $X^{(k)}_r$ given $X^{(k)}_{\setminus r}$ and the outer total expectation being the marginal joint expectation of $X^{(k)}_{\setminus r}$. So we have the total expectation
\begin{equation}
    \mathbb{E}[Z^{(k)}_{i,u}] = \mathbb{E}[\mathbb{E}[Z^{(k)}_{i,u}|\Delta^{(k)}]] = \mathbb{E}[0] = 0.
\end{equation}
Also, from the expression of $Z^{(k)}_{i,u}$, since all samples are either $-1$ or $+1$, it is easy to see that $|Z^{(k)}_{i,u}| \leq 2$. On the other hand, note that the total $nK$ samples $\{X^{(k)}_i\}_{1\leq i \leq n, 1\leq k \leq K}$ are conditionally independent given $\{\Delta^{(k)}\}^K_1$ ( $\{\Delta^{(k)}\}^K_1$ are the \emph{latent random variables}). We can then apply the Hoeffding's Inequality with latent conditional independence (LCI), Corollary 1 in \cite{ke2019exact} by conditioning on the \emph{latent random variables} $\{\Delta^{k}\}_{k=1}^K$ to get
\begin{equation}
    \mathbb{P}[ \sum^K_{k=1}\sum^n_{i=1}(Z^{(k)}_{i,u}-0) \geq \delta] \leq \exp \Big( -\frac{\delta^2}{8nK}\Big),
\end{equation}
for any $\delta > 0.$
Substituting $Y_{2,u}$ in (\ref{y2u}) and by the symmetry of it (resulting from the symmetry of the binary random variables $\{X^{(k)}_i\}_{1\leq i \leq n, 1\leq k \leq K}$), we have 
\begin{align*}
    \mathbb{P}[|Y_{2,u}| \geq \delta] & = \mathbb{P}[Y_{2,u} \geq \delta \quad \text{or} \quad Y_{2,u} \leq -\delta]\\
    & \leq \mathbb{P}[Y_{2,u} \geq \delta] + \mathbb{P}[Y_{2,u} \leq -\delta]\\
    &= 2 \mathbb{P}[Y_{2,u} \geq \delta]\\
     &= 2\mathbb{P}[\frac{1}{nK}\sum^K_{k=1}\sum^n_{i=1} (Z^{(k)}_{i,u}-0) \geq \delta]\\
     &\leq 2 \exp \Big( -\frac{\delta^2}{8nK}\Big).
\end{align*}
After that, applying union bound over the indices $u$ of $Y_2$ yields
\begin{equation}
    \mathbb{P}[\|Y_2\|_{\infty} \geq \delta] \leq 2 \exp \Big( -\frac{\delta^2}{8nK}+ \log p \Big).
\end{equation}

\emph{Bounding $\|Y_1\|_{\infty}$}.

Note that $Y_1 = \nabla \ell(\bar{\theta}_{\setminus r}) - \frac{1}{K}\sum_{k=1}^K \nabla \ell^{(k)}(\bar{\theta}^{(k)}_{\setminus r})$ using the shorthand notations in (\ref{decayassump}). We can bound $\|Y_2\|_{\infty}$ by writing 
\begin{equation}
    \begin{split}
        \|Y_1\|_{\infty} & = \|Y_1 - \mathbb{E}(Y_1) +\mathbb{E}(Y_1)\|_{\infty} \\
        & \leq \|Y_1 - \mathbb{E}(Y_1) \|_{\infty} + \|\mathbb{E}(Y_1)\|_{\infty}.
    \end{split}
\end{equation}
Using Assumption (\ref{auxa3}) by setting $\delta = \sqrt{\frac{8 \log(2 p/ \varepsilon)}{nK}}$, we have
\begin{equation}
    \mathbb{P}(\|\mathbb{E}(Y_1)\|_{\infty} \geq \delta) \leq 2 \exp \left(- \frac{\delta^2 nK}{8} + \log p \right).
\end{equation}

Notice that for $Y_1$, we can also decompose it into a sum of random variables $Z'^{(k)}_{i,u}$

\begin{equation}
    Y_{1,u}= \frac{1}{K} \sum^K_{k=1} \frac{1}{n} \sum^n_{i=1} Z'^{(k)}_{i,u},
\end{equation}
where 
\begin{align*}
    Z'^{(k)}_{i,u} & = x^{(k)}_{i,u} \{ x^{(k)}_{i,r} - \frac{\exp(\sum_{t \in V \setminus r} \bar{\theta}_{rt} x_{i,t}^{(k)}) - \exp(-\sum_{t \in V \setminus r} \bar{\theta}_{rt} x_{i,t}^{(k)})}{\exp(\sum_{t \in V \setminus r} \bar{\theta}_{rt} x_{i,t}^{(k)}) + \exp(-\sum_{t \in V \setminus r} \bar{\theta}_{rt} x_{i,t}^{(k)} )}\} \\
    & = x^{(k)}_{i,u} \{ x^{(k)}_{i,r} - \mathbb{P}_{\bar{\theta}}[X_r = 1 | x^{(k)}_{i, \setminus r}] + \mathbb{P}_{\bar{\theta}}[X_r = -1 | x^{(k)}_{i, \setminus r}] \}.
\end{align*}
Here $\mathbb{P}_{\bar{\theta}}$ denotes the conditional probability of the random variable associated with node $r$ taking on $-1$ or $+1$ given a $(p-1)$-dimensional data vector values $x^{(k)}_{i,\setminus r}$, supposing the true parameter is $\bar{\theta}_{\setminus r}$. In this way, we can write
each entry of $Y_1 - \mathbb{E}(Y_1) $ as 
\begin{align*}
    Y_{1,u} - \mathbb{E}(Y_{1,u}) & = \frac{1}{K} \sum^K_{k=1} \frac{1}{n} \sum^n_{i=1} Z'^{(k)}_{i,u} - \mathbb{E}[\frac{1}{K} \sum^K_{k=1} \frac{1}{n} \sum^n_{i=1} Z'^{(k)}_{i,u}]\\
    & = \frac{1}{K} \sum^K_{k=1} \frac{1}{n} \sum^n_{i=1} Z'^{(k)}_{i,u}-\mathbb{E}[Z'^{(k)}_{i,u}].
\end{align*}
Then we define random variable
\begin{equation}
    H^{(k)}_{i,u} := Z'^{(k)}_{i,u}-\mathbb{E}[Z'^{(k)}_{i,u}]
\end{equation}
for all $1 \leq i \leq n, 1 \leq k \leq K, 1 \leq u \leq p-1$. We have
\begin{equation}
    \mathbb{E}[H^{(k)}_{i,u}] = \mathbb{E} [Z'^{(k)}_{i,u}-\mathbb{E}[Z'^{(k)}_{i,u}]] = \mathbb{E}[Z'^{(k)}_{i,u}] - \mathbb{E}[Z'^{(k)}_{i,u}] = 0.
\end{equation}
Since the expected value $\mathbb{E}[Z'^{(k)}_{i,u}]$ is deterministic, the randomness of $H^{(k)}_{i,u}$ takes on the same pattern as $Z'^{(k)}_{i,u}$, so they are conditionally independent given $\{\Delta^{(k)}\}^K_1$. In addition,  $H^{(k)}_{i,u}$ is bounded in the sense that $|H^{(k)}_{i,u}|\leq 6$. By using LCI Hoeffding's inequality \cite{ke2019exact} again, we get
\begin{equation}
    \mathbb{P}[ \sum^K_{k=1}\sum^n_{i=1}(H^{(k)}_{i,u}-0) \geq \delta] \leq \exp \Big( -\frac{\delta^2}{72nK}\Big).
\end{equation}
Using the same reasoning (symmetry and union bound) in proving the bound for $\|Y_2\|_{\infty}$, we get
\begin{equation}
    \mathbb{P}[\|Y_1 - \mathbb{E}[Y_1]\|_{\infty} \geq \delta] \leq 2 \exp \Big( -\frac{\delta^2}{72nK}+ \log p \Big).
\end{equation}
Next, putting the terms $\|\mathbb{E}[Y_1]\|_{\infty}$ and $\|Y_1 - \mathbb{E}[Y_1]\|_{\infty}$ together, we have
\begin{align*}
    \mathbb{P}(\|Y_1\|_{\infty} > 2\delta) & = 1- \mathbb{P}(\|Y_1\|_{\infty}  < 2\delta)\\
    &\leq  1- \mathbb{P}(\|\mathbb{E}[Y_1]\|_{\infty} +  \|Y_1 - \mathbb{E}[Y_1]\|_{\infty}< 2\delta)\\
    &\leq 1- \mathbb{P}(\|\mathbb{E}[Y_1]\|_{\infty}< \delta \quad  \text{and}  \quad \|Y_1 - \mathbb{E}[Y_1]\|_{\infty}< \delta)\\
    & = \mathbb{P}(\|\mathbb{E}[Y_1]\|_{\infty}\geq \delta \quad  \text{or}  \quad   \|Y_1 - \mathbb{E}[Y_1]\|_{\infty}\geq \delta)\\
    & \leq \mathbb{P}(\|\mathbb{E}[Y_1]\|_{\infty}\geq \delta) +  \|Y_1 - \mathbb{E}[Y_1]\|_{\infty}\geq \delta).
\end{align*}
By the same token, we get
\begin{align}
    \mathbb{P}(\|W^N\|_{\infty} & > 3\delta) \leq \mathbb{P}(\|Y_2\|_{\infty} > \delta) + \mathbb{P}(\|\mathbb{E}[Y_1]\|_{\infty}\geq \delta) +  \|Y_1 - \mathbb{E}[Y_1]\|_{\infty}\geq \delta)\\
    & \leq 4\exp \Big( -\frac{\delta^2}{8nK}+ \log p \Big)) + 2\exp \Big( -\frac{\delta^2}{72nK}+ \log p \Big)\\
    & \leq 6\exp \Big( -\frac{\delta^2}{72nK}+ \log p \Big).
\end{align}
Finally, setting $3 \delta = \frac{\alpha \lambda}{4(2-\alpha)}$, we obtain
\begin{equation}
    \mathbb{P}\left(\frac{2-\alpha}{\lambda} \|W^N\|_{\infty} > \frac{\alpha }{4} \right) \leq 6 \exp \left(- \frac{\alpha^2 \lambda^2}{c(2-\alpha)^2} nK + log(p) \right),
\end{equation}
for some fixed constant $c$ as in Lemma \ref{WN}.
\end{proof}
\subsubsection{Proof of Lemma \ref{RN}}\label{pfRN}
\begin{proof}
We first show that the remainder term $R^N$ satisfies the bound $\|R^N\|_{\infty}\leq D_{\max} \|\hat{\theta}_S - \bar{\theta}_S\|_2^2$. Then the result of Lemma \ref{l2cons}, namely $\|\hat{\theta}_S - \bar{\theta}_S \|_2 \leq \frac{5}{C_{\min}}\sqrt{d}\lambda$, can be used to conclude that 
\begin{equation*}
    \frac{\|R^N\|_{\infty}}{\lambda} \leq \frac{25D_{\max}}{C_{\min}^2} \lambda d 
\end{equation*}
as claimed in Lemma \ref{RN}.

Focusing on element $R^N_j$ for some index $j \in \{1,...,p\},$ we have 
\begin{align*}
    R_j^N &=-[\nabla^2 \ell(\theta^{(j)}; \mathfrak{X}) - \nabla^2 \ell(\bar{\theta}; \mathfrak{X})]^T_j (\hat{\theta} - \bar{\theta})
    \\& = \frac{1}{K} \sum^K_{k=1}  \frac{1}{n} \sum^n_{i=1} [\eta (x^{(k)}_i; \theta^{(j)}) - \eta (x^{(k)}_i; \bar{\theta}))] (\hat{\theta} - \bar{\theta}),
\end{align*}
for some point $\theta^{(j)} = \mu_j \hat{\theta} + (1-\mu_j)\bar{\theta}$ and $\mu_j \in [0,1]$. Then we set $g(t) = \frac{4e^{2t}}{(e^{2t}+1)^2}$ by noting that $\eta(\theta,x) = g(x_r\sum_{t \in V \setminus r} \theta_{rt} x_t).$ By the chain rule and another application of the mean value theorem, we then have
\begin{align*}
    R_j^N & = \frac{1}{K} \sum^K_{k=1}  \frac{1}{n} \sum^n_{i=1} g'({(\theta'^{(j)})}^T x^{(k)}_i){(x^{(k)}_i)}^T [\theta^{(j)} - \bar{\theta}]\{ x^{(k)}_{i,j} {(x^{(k)}_i)}^T [\hat{\theta} - \bar{\theta}]\}
    \\ & = \frac{1}{K} \sum^K_{k=1}  \frac{1}{n} \sum^n_{i=1} \{g'({(\theta'^{(j)})}^T x^{(k)}_i)x^{(k)}_{i,j}\} \{{[\theta^{(j)} - \bar{\theta}]}^T  x^{(k)}_{i,j} {(x^{(k)}_i)}^T [\hat{\theta} - \bar{\theta}]\},
\end{align*}

where $\theta'^{(j)}$ is another point on the line joining $\hat{\theta}$ and $\bar{\theta}.$ Setting $a_i^{(k)} := \{g'({(\theta'^{(j)})}^T x^{(k)}_i)x^{(k)}_{i,j}\}$ and $b^{(k)}_i := \{{[\theta^{(j)} - \bar{\theta}]}^T  x^{(k)}_{i,j} {(x^{(k)}_i)}^T [\hat{\theta} - \bar{\theta}]\},$ and treating $a, b$ both as $nK$-dimensional vectors, we have 
$$|R_j^N| = \frac{1}{nK} |\sum^K_{k=1} \sum^n_{i=1} a_i^{(k)}b^{(k)}_i| \leq  \frac{1}{nK} \|a\|_{\infty} \|b\|_1.$$
A calculation shows that $\|a\|_{\infty} \leq 1,$ and 
\begin{align*}
    \frac{1}{nK}\|b\|_1 & = \frac{1}{K} \sum^K_{k=1} \mu_j {[\hat{\theta} - \bar{\theta}]}^T \left\{ \frac{1}{n} \sum^n_{i=1} x^{(k)}_{i} {(x^{(k)}_i)}^T \right\} [\hat{\theta} - \bar{\theta}]
    \\& = \frac{1}{K} \sum^K_{k=1} \mu_j {[\hat{\theta}_S - \bar{\theta}_S]}^T \left\{ \frac{1}{K} \sum^K_{k=1} \frac{1}{n} \sum^n_{i=1} x^{(k)}_{i,S} {(x^{(k)}_{i,S})}^T \right\} [\hat{\theta}_S - \bar{\theta}_S]
    \\& \leq D_{\max} \|\hat{\theta}_S - \bar{\theta}_S\|_2^2,
\end{align*}
where the second line uses the fact that $\hat{\theta}_{S^c} = \bar{\theta}_{S^c} = 0$. Therefore, we have $$\|R^N\|_{\infty} \leq D_{\max} \|\hat{\theta}_S - \bar{\theta}_S\|_2^2$$
\end{proof}
\subsubsection{Proof of Lemma \ref{l2cons}}\label{pfl2cons}
\begin{proof}
Following the method of proof in \cite{ravikumar2010high} which was also previously used in another context \cite{rothman2008sparse}, we define the function $G: \mathbb{R}^d \rightarrow \mathbb{R}$ by
\begin{equation}
    G(u_S) := \ell(\bar{\theta}_S + u_S; \{\mathfrak{X}^n_1\}^K_1) - \ell( \bar{\theta}_S; \{\mathfrak{X}^n_1\}^K_1) + \lambda_n (\|\bar{\theta}_S + u_S\|_1 - \|\bar{\theta}_S\|_1).
\end{equation}
It can be seen that $\hat{u} = \hat{\theta}_S - \bar{\theta}_S$ minimizes $G$. Moreover, $G(0) = 0$ by construction; therefore, we must have $G(\hat{u}) \leq 0$. Note that $G$ is convex. Suppose we show for some radius $B > 0$, and for $u \in \mathbb{R}^d$ with $\|u\|_2 = B,$ we have $G(u) > 0,.$ we then claim that $\|\hat{u}\|_2 \leq B.$ In fact, if $\hat{u}$ lay outside the ball of radius $B$, then the convex combination $t\hat{u}+(1-t)(0)$ would lie on the boundary of the ball, for an appropriately chosen $t \in (0,1).$ By convexity, $$G(t\hat{u}+(1-t)(0)) \leq t G(\hat{u}) + (1-t)G(0) \leq 0,$$
which contradicts the assumed strict positivity of $G$ on the boundary. It thus suffices to establish strict positivity of $G$ on the boundary of the ball with radius $B=M\lambda \sqrt{d}$, where $M>0$ is a parameter to be chosen later in the proof. Let $u \in \mathbb{R}^d$ be an arbitrary vector with $\|u\|_2=B.$ Recalling the notation $W^N := -\nabla \ell(\bar{\theta}; \{\mathfrak{X}^n_1\}^K_1),$ by a Taylor series expansion of the log likelihood component of $G$, we have 
$$G(u) = -(W^N_S)^T u + u^T[\nabla^2 \ell(\bar{\theta}_S+\alpha u; \{\mathfrak{X}^n_1\}^K_1)]u + \lambda_n(\|\bar{\theta}_S+u_S\|_1 - \|\bar{\theta}_S\|_1) $$
for some $\alpha \in [0,1].$ For the first term, we have the bound 
\begin{equation}
    |(W^N_S)^T u| \leq \|W^N_S\|_{\infty} \|u\|_1 \leq  \|W^N_S\|_{\infty} \sqrt{d} \|u\|_2 \leq (\lambda_n \sqrt{d})^2 \frac{M}{4},
\end{equation}
since $\|W^N_S\|_{\infty} \leq \frac{\lambda_n}{4}$ by assumption.
For the last term, applying triangle inequality  yields $$\lambda_n(\|\bar{\theta}_S+u_S\|_1 - \|\bar{\theta}_S\|_1) \geq -\lambda_n \| u_S\|_1.$$
Since $\|u_S\|_1 \leq \sqrt{d} \|u_S\|_2, $ we have 
\begin{equation}
    \lambda_n(\|\bar{\theta}_S+u_S\|_1 - \|\bar{\theta}_S\|_1) \geq -\lambda_n \sqrt{d}\| u_S\|_2 = - M{(\sqrt{d} \lambda_n)}^2.
\end{equation}
Finally, turning to the middle Hessian term, we have 
\begin{align*}
    q^* & := \Lambda_{\min}(\nabla^2 \ell(\bar{\theta}_S+\alpha u); \{\mathfrak{X}^n_1\}^K_1)) 
    \\ & \geq \min_{\alpha \in [0,1]} \Lambda_{\min}(\nabla^2 \ell(\bar{\theta}_S+\alpha u_S); \{\mathfrak{X}^n_1\}^K_1))
    \\ & = \min_{\alpha \in [0,1]} \Lambda_{\min} \left[ \frac{1}{K} \sum^K_{k=1} \frac{1}{n} \sum^n_{i=1} \eta (x^{(k)}_i ; \bar{\theta}_S+\alpha u_S) x^{(k)}_{i,S} {(x^{(k)}_{i,S})}^T \right].
\end{align*}
By a Taylor series expansion of $\eta(x^{(k)}_i; \cdot)$, we have, for some $\alpha_0 \in [0,\alpha]$,  a lower bound of $q^*$:

\begin{multline*}
\min_{\alpha \in [0,1]} \Lambda_{\min} \{ \frac{1}{nK} \sum^K_{k=1} \sum^n_{i=1} [\eta (x^{(k)}_i ; \bar{\theta}_S) x^{(k)}_{i,S} {(x^{(k)}_{i,S})}^T \\ \cr   + \alpha g'(x^{(k)}_{i,r} \sum_{t\in S \setminus r} (\bar{\theta}_{rt}  + \alpha_0 u_{rt})x^{(k)}_{i,t} ) x^{(k)}_{i,r} (u_S^T x^{(k)}_{i,S}) x^{(k)}_{i,S} {(x^{(k)}_{i,S})}^T ]\} 
\end{multline*}
\begin{align*}
    & \geq \Lambda_{\min} \left[ \frac{1}{K} \sum^K_{k=1} \frac{1}{n} \sum^n_{i=1} \eta (x^{(k)}_i ; \bar{\theta}_S) x^{(k)}_{i,S} {(x^{(k)}_{i,S})}^T \right] 
    \\& + \min_{\alpha \in [0,1]} \alpha \Lambda_{\min} \left[ \frac{1}{K} \sum^K_{k=1} \frac{1}{n} \sum^n_{i=1} g' \left(x^{(k)}_{i,r}{(\bar{\theta}_S+\alpha_0 u_S)}^T x^{(k)}_{i,S} \right)x^{(k)}_{i,r} (u_S^T x^{(k)}_{i,S}) x^{(k)}_{i,S} {(x^{(k)}_{i,S})}^T \right]
    \\ & \geq \Lambda_{\min}(Q^N_{SS}) - \max_{\alpha \in [0,1]} \alpha \max_{\alpha_0 \in [0, \alpha]} \vertiii{\frac{1}{K} \sum^K_{k=1} \frac{1}{n} \sum^n_{i=1} g' \left(x^{(k)}_{i,r}{(\bar{\theta}_S+\alpha_0 u_S)}^T x^{(k)}_{i,S} \right) (u_S^T x^{(k)}_{i,S}) x^{(k)}_{i,S} {(x^{(k)}_{i,S})}^T}_2
    \\& \geq \Lambda_{\min}(Q^N_{SS}) - \max_{\alpha \in [0, 1]} \vertiii{\frac{1}{K} \sum^K_{k=1} \frac{1}{n} \sum^n_{i=1} g' \left(x^{(k)}_{i,r}{(\bar{\theta}_S+\alpha u_S)}^T x^{(k)}_{i,S} \right) (u_S^T x^{(k)}_{i,S}) x^{(k)}_{i,S} {(x^{(k)}_{i,S})}^T}_2
    \\& \geq C_{\min} - \max_{\alpha \in [0, 1]} \vertiii{\underbrace{\frac{1}{K} \sum^K_{k=1} \frac{1}{n} \sum^n_{i=1} g' \left(x^{(k)}_{i,r}{(\bar{\theta}_S+\alpha u_S)}^T x^{(k)}_{i,S} \right) (\langle u_S, x^{(k)}_{i,S} \rangle) x^{(k)}_{i,S} {(x^{(k)}_{i,S})}^T}_{A(\alpha)}}_2
\end{align*}

It remains to control the spectral norm of the matrix $A(\alpha),$ for $\alpha \in [0,1]$. For any fixed $\alpha \in [0,1]$, and $y \in \mathbb{R^d}$ with $\|y\|_2 = 1,$ we have 
\begin{align*}
    \langle y, A(\alpha) y \rangle & = \frac{1}{K} \sum^K_{k=1} \frac{1}{n} \sum^n_{i=1} g' \left(x^{(k)}_{i,r}{(\bar{\theta}_S+\alpha u_S)}^T x^{(k)}_{i,S} \right) [\langle u_S, x^{(k)}_{i,S} \rangle] {[\langle x^{(k)}_{i,S}, y \rangle]}^2
    \\& \leq \frac{1}{K} \sum^K_{k=1} \frac{1}{n} \sum^n_{i=1} \left|g' \left(x^{(k)}_{i,r}{(\bar{\theta}_S+\alpha u_S)}^T x^{(k)}_{i,S} \right)\right| |\langle u_S, x^{(k)}_{i,S} \rangle| {[\langle x^{(k)}_{i,S}, y \rangle]}^2
\end{align*}
Note that $ \left|g' \left(x^{(k)}_{i,r}{(\bar{\theta}_S+\alpha u_S)}^T x^{(k)}_{i,S} \right)\right| \leq 1,$ and $$|\langle u_S, x^{(k)}_{i,S} \rangle| \leq \|u_S\|_1 \leq \sqrt{d} \|u_S\|_2 = M \lambda_n d.$$
Moreover, we have $$\frac{1}{K} \sum^K_{k=1} \frac{1}{n} \sum^n_{i=1}{(\langle x^{(k)}_{i,S}, y \rangle)}^2  \leq \vertiii{\frac{1}{K} \sum^K_{k=1} \frac{1}{n} \sum^n_{i=1} x^{(k)}_{i,S} (x^{(k)}_{i,S})^T}_2 \leq D_{\max}$$
by assumption.We then obtain $$\max_{\alpha \in [0,1]}\vertiii{A(\alpha)}_2 \leq D_{\max} M \lambda_n d \leq C_{\min}/2,$$
assuming that $\lambda_n \leq \frac{C_{\min}}{2M D_{\max} d}.$ Under this condition, we have shown that
\begin{equation}
    q^* := \Lambda_{\min}(\nabla^2 \ell(\bar{\theta}_S+\alpha u); \{\mathfrak{X}^n_1\}^K_1)) \geq C_{\min}/2.
\end{equation}
Finally, combining the three terms in $G(u),$ we conclude that $$G(u_S) \geq (\lambda_n \sqrt{d})^2 \left\{ -\frac{1}{4} M + \frac{C_{\min}}{2}M^2 - M\right\},$$
which is strictly positive for $M = 5/C_{\min}.$ Therefore, as long as $$\lambda_n \leq \frac{C_{\min}}{2M D_{\max} d} = \frac{C_{\min}^2}{10 D_{\max}d },$$
we are guaranteed that $$\|\hat{u}_S\|_2 \leq M \lambda_n \sqrt{d} = \frac{5}{C_{\min}} \lambda_n \sqrt{d}.$$
\end{proof}

\section{Proof of Lemmas for Theorem \ref{thm: thm2}}
\subsection{Proof of Lemmas for Uniform Convergence of Sample Information Matrices in Novel Task}

\subsubsection{Proof of Lemma \ref{2mineigen}} \label{2pfmineigen}
\begin{proof}
The $(j,l)^{th}$ element of the difference matrix $Q^{(K+1)}(\bar{\theta}^{(K+1)}_S) - \bar{Q}^{(K+1)}(\bar{\theta}^{(K+1)}_S)$ can be written as an i.i.d. sum of the form $Z_{jl}^{(K+1)} = \frac{1}{n^{(K+1)}} \sum^{n^{(K+1)}}_{i=1} Z_{jl,i}^{(K+1)} $, where each $Z_{jl,i}^{(K+1)}$ is zero-mean and bounded (in particular, $|Z_{jl,i}^{(K+1)}| \leq 4$) By the Azuma-Hoeffding's bound \cite{hoeffding1994probability}, for any indices $j,l = 1,...,d$ and for any $\varepsilon > 0$, we have 
\begin{equation}\label{2QZ}
    \mathbb{P}[(Z_{jl}^{(K+1)})^2 \geq \varepsilon^2] = \mathbb{P}\big[ | \frac{1}{n^{(K+1)}} \sum^{n^{(K+1)}}_{i=1} Z_{jl,i}^{(K+1)}| \geq \varepsilon\big] \leq 2 \exp \left( -\frac{\varepsilon^2 n^{(K+1)}}{32}\right).
\end{equation}
By the Courant-Fischer variational representation \cite{horn2012matrix},
\begin{align*}
    \Lambda_{\min}({\bar{Q}^{(K+1)}_{S^{(K+1)}S^{(K+1)}}}) & = \min_{\|x\|_2 = 1}x^T\bar{Q}^{(K+1)}_{S^{(K+1)}S^{(K+1)}} x 
    \\ & =  \min_{\|x\|_2 = 1} \{x^T Q^{(K+1)}_{S^{(K+1)}S^{(K+1)}} x + x^T ( \bar{Q}^{(K+1)}_{S^{(K+1)}S^{(K+1)}} - Q^{(K+1)}_{S^{(K+1)}S^{(K+1)}}) x\}
    \\ & \leq y^T Q^{(K+1)}_{S^{(K+1)}S^{(K+1)}} y + y^T ( \bar{Q}^{(K+1)}_{S^{(K+1)}S^{(K+1)}} -Q^{(K+1)}_{S^{(K+1)}S^{(K+1)}}) y,
\end{align*}
where  $y \in \mathbb{R}^d$ is a unit-norm minimal eigenvector of $Q^{(K+1)}_{S^{(K+1)}S^{(K+1)}}$. Therefore, we have 
\begin{align*}
    \Lambda_{\min}(Q^{(K+1)}_{S^{(K+1)}S^{(K+1)}}) &\geq \Lambda_{\min}(\bar{Q}^{(K+1)}_{S^{(K+1)}S^{(K+1)}}) - \vertiii{\bar{Q}^{(K+1)}_{S^{(K+1)}S^{(K+1)}} - Q^{(K+1)}_{S^{(K+1)}S^{(K+1)}}}_2 
    \\& \geq C^{(K+1)}_{\min} - \vertiii{\bar{Q}^{(K+1)}_{S^{(K+1)}S^{(K+1)}} - Q^{(K+1)}_{S^{(K+1)}S^{(K+1)}}}_2.
\end{align*}
Observe that $$\vertiii{\bar{Q}^{(K+1)}_{S^{(K+1)}S^{(K+1)}} - Q^{(K+1)}_{S^{(K+1)}S^{(K+1)}}}_2 \leq \left( \sum^{d}_{j=1} \sum^{d}_{l=1} {(Z^{(K+1)}_{jl})}^2\right)^{1/2}$$
Setting $\varepsilon^2 = \delta^2/d^2$ in (\ref{2QZ}) and applying the union bound over the $d^2$ index pairs $(j,l)$ then yields
\begin{equation}\label{usebd2}
    \mathbb{P}[\vertiii{\bar{Q}^{(K+1)}_{S^{(K+1)}S^{(K+1)}} - Q^{(K+1)}_{S^{(K+1)}S^{(K+1)}}}_2 \geq \delta] \leq 2 \exp \left( -\frac{\delta^2n^{(K+1)}}{32 d^2}+2 \log (d)\right)
\end{equation}

So, we have the first concentration inequality in Lemma \ref{2mineigen}:
\begin{equation}
    \mathbb{P}[\Lambda_{\min}(Q^{(K+1)}_{S^{(K+1)}S^{(K+1)}}) \leq C^{(K+1)}_{\min} - \delta] \leq 2 \exp \left(  -\frac{\delta^2n^{(K+1)}}{32 d^2}+2 \log (d)\right).
\end{equation}
This proves the first part of the lemma.
For the second concentration inequality about maximum eigenvalue of the sample covariance matrix, with the same
reasoning from the Courant-Fischer variational representation \cite{horn2012matrix}, we have, 
\begin{align*}
    \Lambda_{\max}(\mathbb{E}[X^{(K+1)}_{S}{(X^{(K+1)}_{S})}^T]) & = \max_{\|v\|_2=1} v^T \mathbb{E}[X^{(K+1)}_{S}{(X^{(K+1)}_{S})}^T] v
    \\ & =\max_{\|v\|_2=1} \{ v^T (\frac{1}{n^{(K+1)}} \sum^{n^{(K+1)}}_{i=1} x^{(K+1)}_{i, S} (x^{(K+1)}_{i, S})^T) v  \\&+ v^T (\mathbb{E}[X^{(K+1)}_{S}{(X^{(K+1)}_{S})}^T] - \frac{1}{n^{(K+1)}} \sum^{n^{(K+1)}}_{i=1} x^{(K+1)}_{i, S} (x^{(K+1)}_{i, S})^T) v \} \\& \geq u^T (\frac{1}{n^{(K+1)}} \sum^{n^{(K+1)}}_{i=1} x^{(K+1)}_{i, S} (x^{(K+1)}_{i, S})^T) u \\& + u^T (\mathbb{E}[X^{(K+1)}_{S}{(X^{(K+1)}_{S})}^T] - \frac{1}{n^{(K+1)}} \sum^{n^{(K+1)}}_{i=1} x^{(K+1)}_{i, S} (x^{(K+1)}_{i, S})^T) u,
\end{align*}

where $u \in \mathbb{R}^d$ is a unit-norm maximal eigenvector of $\frac{1}{n^{(K+1)}} \sum^{n^{(K+1)}}_{i=1} x^{(K+1)}_{i, S} (x^{(K+1)}_{i, S})^T$. Therefore, we have 
\begin{align*}
  & \Lambda_{\max} (\frac{1}{n^{(K+1)}} \sum^{n^{(K+1)}}_{i=1} x^{(K+1)}_{i, S} (x^{(K+1)}_{i, S})^T) \\ & \leq \Lambda_{\max}(\mathbb{E}[X^{(K+1)}_{S}{(X^{(K+1)}_{S})}^T]) + u^T (\frac{1}{n^{(K+1)}} \sum^{n^{(K+1)}}_{i=1} x^{(K+1)}_{i, S} (x^{(K+1)}_{i, S})^T - \mathbb{E}[X^{(K+1)}_{S}{(X^{(K+1)}_{S})}^T]) u  \\ & \leq D^{(K+1)}_{\max} + \vertiii{(\frac{1}{n^{(K+1)}} \sum^{n^{(K+1)}}_{i=1} x^{(K+1)}_{i, S} (x^{(K+1)}_{i, S})^T - \mathbb{E}[X^{(K+1)}_{S}{(X^{(K+1)}_{S})}^T])}_2
\end{align*}

The difference matrix $\frac{1}{n^{(K+1)}} \sum^{n^{(K+1)}}_{i=1} x^{(K+1)}_{i, S} (x^{(K+1)}_{i, S})^T - \mathbb{E}[X^{(K+1)}_{S}{(X^{(K+1)}_{S})}^T]$ can be written as an i.i.d. sum of the form $Y_{jl}^{(K+1)} = \frac{1}{n^{(K+1)}} \sum^{n^{(K+1)}}_{i=1} Y_{jl,i}^{(K+1)}$, where each $Y_{jl,i}^{(K+1)}$ is zero-mean and bounded (in particular, $|Y_{jl,i}^{(K+1)}| \leq 4$). By the Azuma-Hoeffding's bound \cite{hoeffding1994probability}, for any indices $j,l = 1,...,d$ and for any $\varepsilon > 0$, we have 
\begin{equation}\label{2EXY}
    \mathbb{P}[(Y_{jl}^{(K+1)})^2 \geq \varepsilon^2] = \mathbb{P}\big[ | \frac{1}{n^{(K+1)}} \sum^{n^{(K+1)}}_{i=1} Y_{jl,i}^{(K+1)}| \geq \varepsilon\big] \leq 2 \exp \left( -\frac{\varepsilon^2 n^{(K+1)}}{32}\right).
\end{equation}
Observe that $$\vertiii{\frac{1}{n^{(K+1)}} \sum^{n^{(K+1)}}_{i=1} x^{(K+1)}_{i, S} (x^{(K+1)}_{i, S})^T - \mathbb{E}[X^{(K+1)}_{S}{(X^{(K+1)}_{S})}^T]}_2 \leq \left(\sum^d_{j=1} \sum^d_{l=1} (Y_{jl}^{(K+1)})^2\right)^{1/2}.$$
Setting $\varepsilon^2 = \delta^2/d^2$ in (\ref{2EXY}) and applying the union bound over the $d^2$ index pairs $(j,l)$ then yields 
\begin{multline*}
    \mathbb{P}[\vertiii{\frac{1}{n^{(K+1)}} \sum^{n^{(K+1)}}_{i=1} x^{(K+1)}_{i, S} (x^{(K+1)}_{i, S})^T - \mathbb{E}[X^{(K+1)}_{S}{(X^{(K+1)}_{S})}^T]}_2 \geq \delta] \\ \leq 2 \exp \left( -\frac{\delta^2n^{(K+1)}}{32 {d}^2}+2 \log ({d})\right)
\end{multline*}
\
So we have the second part of the lemma $$\mathbb{P} \big[  \Lambda_{\max} \big[ \frac{1}{n^{(K+1)}} \sum^{n^{(K+1)}}_{i=1} x^{(K+1)}_{i, S} (x^{(K+1)}_{i, S})^T)\big] \geq D^{(K+1)}_{\max} + \delta \big] \leq 2 \exp \left( -\frac{\delta^2n^{(K+1)}}{32 {d}^2}+2 \log ({d})\right)$$

\end{proof}

\subsubsection{Proof of Lemma \ref{2incoh}}\label{2pfincoh}
\begin{proof}
Decomposing the sample matrix as the sum $Q^{(K+1),S}_{{[S^{(K+1)}]}^cS^{(K+1)}}{(Q^{(K+1)}_{S^{(K+1)}S^{(K+1)}})}^{-1} = T^{(K+1)}_1 + T^{(K+1)}_2 + T^{(K+1)}_3 + T^{(K+1)}_4,$ where we define 
\begin{subequations}
\begin{align}
    T^{(K+1)}_1 &:= \bar{Q}^{(K+1),S}_{{[S^{(K+1)}]}^c S^{(K+1)}}[{(Q^{(K+1)}_{S^{(K+1)}S^{(K+1)}})}^{-1} - {(\bar{Q}^{(K+1)}_{S^{(K+1)}S^{(K+1)}})}^{-1}],\\
    T^{(K+1)}_2 &:= [Q^{(K+1),S}_{{[S^{(K+1)}]}^c S^{(K+1)}} - \bar{Q}^{(K+1),S}_{{[S^{(K+1)}]}^c S^{(K+1)}}]{(\bar{Q}^{(K+1)}_{S^{(K+1)}S^{(K+1)}})}^{-1},\\
    T^{(K+1)}_3 &:= [Q^{(K+1),S}_{{[S^{(K+1)}]}^c S^{(K+1)}} - \bar{Q}^{(K+1),S}_{{[S^{(K+1)}]}^c S^{(K+1)}}][{(Q^{(K+1)}_{S^{(K+1)}S^{(K+1)}})}^{-1} - {(\bar{Q}^{(K+1)}_{S^{(K+1)}S^{(K+1)}})}^{-1}],\\
    T^{(K+1)}_4 &:= \bar{Q}^{(K+1),S}_{{[S^{(K+1)}]}^c S^{(K+1)}}{(\bar{Q}^{(K+1)}_{S^{(K+1)}S^{(K+1)}})}^{-1}.
\end{align}
\end{subequations}
The fourth term is controlled by the incoherence assumption (A2) $$\vertiii{T^{(K+1)}_4}_{\infty} = \vertiii{\bar{Q}^{(K+1),S}_{{[S^{(K+1)}]}^cS^{(K+1)}}{(\bar{Q}^{(K+1)}_{S^{(K+1)}S^{(K+1)}})}^{-1}}_{\infty} \leq 1-\alpha^{(K+1)}.$$ If we can show that $\vertiii{T^{(K+1)}_i}_{\infty} \leq \frac{\alpha^{(K+1)}}{6}$ for the remaining indices $i=1,2,3$, then by our four term decomposition and the triangle inequality, the sample version satisfies the desired bound (\ref{2incoheq}). For the remaining three terms, the following lemma is useful in the proof:

\begin{lemma}\label{2lmlm}
For any $\delta > 0$, and constants $B,B_1,B_2,$ the following bounds hold,
\begin{subequations}
\begin{align}
    \mathbb{P}\left[\vertiii{Q^{(K+1),S}_{{[S^{(K+1)}]}^cS^{(K+1)}} - \bar{Q}^{(K+1),S}_{{[S^{(K+1)}]}^cS^{(K+1)}}}_{\infty} \geq \delta\right] & \leq 2 \exp \left( -B\frac{\varepsilon^2n^{(K+1)}}{{d}^2} + 2\log (d) \right),\label{2eq:subeq1}\\
    \mathbb{P}\left[\vertiii{Q^{(K+1)}_{S^{(K+1)}S^{(K+1)}} - \bar{Q}^{(K+1)}_{S^{(K+1)}S^{(K+1)}}}_{\infty} \geq \delta\right]& \leq 2 \exp \left( -B\frac{\varepsilon^2n^{(K+1)}}{{d}^2} + 2\log({d}) \right),\label{2eq:subeq2}\\
    \mathbb{P}\left[\vertiii{{(Q^{(K+1)}_{S^{(K+1)}S^{(K+1)}})}^{-1} - {(\bar{Q}^{(K+1)}_{S^{(K+1)}S^{(K+1)}})}^{-1}}_{\infty}  \geq \delta\right] & \leq 4 \exp \left( -B_1\frac{n^{(K+1)}\delta^2}{d^3} + B_2 \log ({d})\right).\label{2eq:subeq3}
\end{align}
\end{subequations}
\end{lemma}
See Section \ref{pf2lmlm} for the proof of these claims.
\paragraph{Control of the first term}
Turning to the first term, we re-factorize it as $$T^{(K+1)}_1 = \bar{Q}^{(K+1),S}_{{[S^{(K+1)}]}^cS^{(K+1)}}{(\bar{Q}^{(K+1)}_{S^{(K+1)}S^{(K+1)}})}^{-1}[\bar{Q}^{(K+1)}_{S^{(K+1)}S^{(K+1)}} - Q^{(K+1)}_{S^{(K+1)}S^{(K+1)}}]{(Q^{(K+1)}_{S^{(K+1)}S^{(K+1)}})}^{-1}.$$ Then, we can upper bound $\vertiii{T^{(K+1)}_1}_{\infty}$ by
\begin{align*}
    & \vertiii{\bar{Q}^{(K+1),S}_{{[S^{(K+1)}]}^cS^{(K+1)}}{(\bar{Q}^{(K+1)}_{S^{(K+1)}S^{(K+1)}})}^{-1}}_{\infty}\vertiii{\bar{Q}^{(K+1)}_{S^{(K+1)}S^{(K+1)}} - Q^{(K+1)}_{S^{(K+1)}S^{(K+1)}}}_{\infty}\vertiii{{(Q^{(K+1)}_{S^{(K+1)}S^{(K+1)}})}^{-1}}_{\infty}
     \\ & \leq (1-\alpha)\vertiii{\bar{Q}^{(K+1)}_{S^{(K+1)}S^{(K+1)}} - Q^{(K+1)}_{S^{(K+1)}S^{(K+1)}}}_{\infty} \{\sqrt{{d}} \vertiii{{(Q^{(K+1)}_{S^{(K+1)}S^{(K+1)}})}^{-1}}_2\},
\end{align*}
where we have used the incoherence assumption in Assumption \ref{novela2}. Using the bound (\ref{2maxeigneq}) in Lemma \ref{2mineigen} with $\delta = C_{\min}/2$, we have $\vertiii{{(Q^{(K+1)}_{S^{(K+1)}S^{(K+1)}})}^{-1}}_2  = {[\Lambda_{\min}(Q^{(K+1)}_{S^{(K+1)}S^{(K+1)}})]}^{-1} \leq \frac{2}{C_{\min}} $ with probability greater than $ 1- 2\exp ( -Bn^{(K+1)}/{d}^2+2 \log ({d}))$. Next, applying the bound (\ref{2eq:subeq2}) with $\delta = c / \sqrt{d},$ we conclude that with probability  greater than $1-2\exp(-Bn^{(K+1)}c^2/{d}^3+2\log(d)),$ we have $$\vertiii{\bar{Q}^{(K+1)}_{S^{(K+1)}S^{(K+1)}} - Q^{(K+1)}_{S^{(K+1)}S^{(K+1)}}}_{\infty} \leq c/\sqrt{d}.$$
By choosing the constant $c > 0$ sufficiently small, we are guaranteed that $$\mathbb{P}[\vertiii{T^{(K+1)}_1}_{\infty} \geq \alpha^{(K+1)}/6] \leq 2 \exp \left(-B\frac{n^{(K+1)}c^2}{d^3} + \log({d}) \right).$$

\paragraph{Control of the second term}
To bound $T^{(K+1)}_2$, we first write 
\begin{align*}
    \vertiii{T^{(K+1)}_2}_{\infty} &\leq \sqrt{{d}}\vertiii{{(\bar{Q}^{(K+1)}_{S^{(K+1)}S^{(K+1)}})}^{-1}}_2 \vertiii{Q^{(K+1),S}_{{[S^{(K+1)}]}^cS^{(K+1)}} - \bar{Q}^{(K+1),S}_{{[S^{(K+1)}]}^cS^{(K+1)}}}_{\infty}
    \\& \leq \frac{\sqrt{{d}}}{C_{\min}} \vertiii{Q^{(K+1),S}_{{[S^{(K+1)}]}^c S^{(K+1)}} - \bar{Q}^{(K+1),S}_{{[S^{(K+1)}]}^cS^{(K+1)}}}_{\infty}.
\end{align*}
Then we apply the bound (\ref{2eq:subeq1}) with $\delta = \frac{\alpha^{(K+1)}}{6}\frac{C_{\min}}{\sqrt{{d}}}$ to conclude that $$\mathbb{P}[\vertiii{T^{(K+1)}_2}_{\infty} \geq \alpha^{(K+1)}/6] \leq 2 \exp \left(-B\frac{n^{(K+1)}}{{d}^3} + \log({d}) \right).$$

\paragraph{Control of the third term}
We set $\delta = \sqrt{\alpha^{(K+1)}/6} $ in the bounds (\ref{2eq:subeq1}) and (\ref{2eq:subeq3}) to conclude that $$\mathbb{P}[\vertiii{T^{(K+1)}_3}_{\infty} \geq \alpha^{(K+1)}/6] \leq 4 \exp \left(-B\frac{n^{(K+1)}}{{d}^3} + \log({d}) \right).$$
Putting together, we conclude that $$\mathbb{P}[\vertiii{Q^{(K+1),S}_{{[S^{(K+1)}]}^cS^{(K+1)}}{(Q^{(K+1)}_{S^{(K+1)}S^{(K+1)}})}^{-1}}_{\infty} \geq 1- \alpha^{(K+1)}/2] = \mathcal{O}\left( \exp\left(-B\frac{n^{(K+1)}}{{d}^3}+ \log ({d})\right)\right)$$
\end{proof}
\subsection{Proof of Lemmas for Proposition \ref{prop2}}
\subsubsection{Proof of Lemma \ref{2WN}}\label{2pfWN}
\begin{proof}
Each entry of $W^{(K+1)}$, denoted by $W^{(K+1)}_u$, for $1 \leq u \leq |S(r)| \leq d$, can be expressed as a sum of independent random variables $Z^{(K+1)}_{i,u}$: $$W^{(K+1)}_u=   \frac{1}{n^{(K+1)}} \sum^{n^{(K+1)}}_{i=1} Z^{(K+1)}_{i,u},$$ where 
\begin{align*}
    Z^{(K+1)}_{i,u} & = x^{(K+1)}_{i,u} \{ x^{(K+1)}_{i,r} - \frac{\exp(\sum_{t \in S \setminus r} \bar{\theta}^{(K+1)}_{rt} x_{i,t}^{(K+1)}) - \exp(-\sum_{t \in S \setminus r} \bar{\theta}_{rt}^{(K+1)} x_{i,t}^{(K+1)})}{\exp(\sum_{t \in S \setminus r} \bar{\theta}_{rt}^{(K+1)} x_{i,t}^{(K+1)}) + \exp(-\sum_{t \in S \setminus r} \bar{\theta}_{rt}^{(K+1)} x_{i,t}^{(K+1)} )}\} \\
    & = x^{(K+1)}_{i,u} \{ x^{(K+1)}_{i,r} - \mathbb{P}_{\bar{\theta}_S^{(K+1)}}[X^{(K+1)}_r = 1 | x^{(K+1)}_{i, S}] + \mathbb{P}_{\bar{\theta}_S^{(K+1)}}[X^{(K+1)}_r = -1 | x^{(K+1)}_{i, S}] \}.
\end{align*}
Notice that the conditional expectation given the values of $\Delta^{(K+1)}$ has mean zero: $$\mathbb{E}[Z^{(K+1)}_{i,u}|\Delta^{(K+1)}] = 0.$$ Then by law of total expectation \cite{weiss2005course} we have
\begin{equation}
    \mathbb{E}[Z^{(K+1)}_{i,u}] = \mathbb{E}[\mathbb{E}[Z^{(K+1)}_{i,u}|\Delta^{(K+1)}]] = \mathbb{E}[0] =  0.
\end{equation}
(See the same logic in the proof in Section \ref{pfWN}). Also, since all the samples are either $-1$ or $+1$, we have $|Z^{(K+1)}_{i,u}| \leq 2.$ Then by Azuma-Hoeffding's inequality \cite{hoeffding1994probability}, we have, for any $\delta > 0$, $$\mathbb{P}[|W^{(K+1)}_u| > \delta] \leq 2 \exp(-\frac{n^{(K+1)} \delta^2}{8}).$$ 
Setting $\delta = \frac{\alpha^{(K+1)} \lambda^{(K+1)}}{4(2-\alpha^{(K+1)})}, $ we obtain $$\mathbb{P}[\frac{2-\alpha^{(K+1)}}{\lambda^{(K+1)}}|W^{(K+1)}_u| > \frac{\alpha^{(K+1)}}{4} ] \leq 2 \exp \left(- \frac{{(\alpha^{(K+1)})}^2 {(\lambda^{(K+1)})}^2 }{128{(2-\alpha^{(K+1)})}^2} n^{(K+1)}\right)$$

Applying a union bound over the indices $u$ of $W^{(K+1)}$ yields $$\mathbb{P}[\frac{2-\alpha^{(K+1)}}{\lambda^{(K+1)}}\|W^{(K+1)}\|_{\infty} > \frac{\alpha^{(K+1)}}{4}  ]\leq 2 \exp \left(- \frac{{(\alpha^{(K+1)})}^2 {(\lambda^{(K+1)})}^2 }{128{(2-\alpha^{(K+1)})}^2} n^{(K+1)} + \log d\right),$$

which converges to zero at rate $\exp (-c{(\lambda^{(K+1)})}^2 n^{(K+1)})$ as long as $\lambda^{(K+1)} \geq \frac{16(2-\alpha^{(K+1)})}{\alpha^{(K+1)}} \sqrt{\frac{\log d}{n^{(K+1)}}}$
\end{proof}
\subsubsection{Proof of Lemma \ref{2RN}}\label{2pfRN}
\begin{proof}
Similar to the proof for Lemma \ref{RN}. We first show that the remainder term $R^{(K+1)}$ satisfies the bound $\|R^{(K+1)}\|_{\infty}\leq D^{(K+1)}_{\max} \|\hat{\theta}^{(K+1)}_{S^{(K+1)}} - \bar{\theta}^{(K+1)}_{S^{(K+1)}}\|_2^2$. Then the result of Lemma \ref{2l2cons}, namely $\|\hat{\theta}^{(K+1)}_{S^{(K+1)}} - \bar{\theta}^{(K+1)}_{S^{(K+1)}} \|_2 \leq \frac{5}{C^{(K+1)}_{\min}}\sqrt{d}\lambda^{(K+1)}$, can be used to conclude that 
\begin{equation*}
    \frac{\|R^{(K+1)}\|_{\infty}}{\lambda^{(K+1)}} \leq \frac{25D^{(K+1)}_{\max}}{{C^{(K+1)}_{\min}}^2} \lambda^{(K+1)} d,
\end{equation*}
as claimed in Lemma \ref{RN}.
Focusing on element $R^{(K+1)}_j$ for some index $j \in \{1,...,|S_r|\},$ we have 
\begin{align*}
& R_j^{(K+1)} \\& =-[\nabla^2 \ell^{(K+1)}(\theta^{(K+1)j}_S; \{\mathfrak{X}^{n^{(K+1)}}_{1,S}\}^{(K+1)}) - \nabla^2 \ell^{(K+1)}(\bar{\theta}^{(K+1)}_S; \{\mathfrak{X}^{n^{(K+1)}}_{1,S}\}^{(K+1)})]^T_j(\hat{\theta}^{(K+1)}_S - \bar{\theta}^{(K+1)}_S) \\& =  \frac{1}{n^{(K+1)}} \sum^{n^{(K+1)}}_{i=1} [\eta (x^{(K+1)}_i; \theta^{(K+1)(j)}_S) - \eta (x^{(K+1)}_i; \bar{\theta}^{(K+1)}_S))] (\hat{\theta}^{(K+1)}_S - \bar{\theta}^{(K+1)}_S)
\end{align*}
for some point $\theta^{(K+1)(j)}_S = \mu_j \hat{\theta}^{(K+1)}_S + (1-\mu_j)\bar{\theta}^{(K+1)}_S$ with $\mu_j \in [0,1].$ Setting $g(t) = \frac{4e^{2t}}{(e^{2t}+1)^2}$ by noting that that $\eta(\theta_S,x) = g(x_r\sum_{t \in S \setminus r} \theta_{rt} x_t).$ By the chain rule and another application of the mean value theorem, we write 
\begin{multline*}
    R_j^{(K+1)} =  \frac{1}{n^{(K+1)}} \sum^{n^{(K+1)}}_{i=1} \{g'({(\theta'^{(K+1)(j)}_S)}^T x^{(K+1)}_{i,S})x^{(K+1)}_{i,j}\} \{{[\theta^{(K+1)(j)}_S - \bar{\theta}^{(K+1)}_S]}^T  \\ x^{(K+1)}_{i,j} {(x^{(K+1)}_{i,S})}^T [\hat{\theta}^{(K+1)}_S - \bar{\theta}^{(K+1)}_S]\},
\end{multline*}
where $\theta'^{(K+1)(j)}_S$ is another point on the line joining $\hat{\theta}^{(K+1)}_S$ and $\bar{\theta}^{(K+1)}_S.$

Setting $a_i^{(K+1)} := \{g'({(\theta'^{(j)}_S)}^T x^{(K+1)}_{i,S})x^{(K+1)}_{i,j}\}$ and $b^{(K+1)}_i := \{{[\theta^{(j)}_S - \bar{\theta}_S]}^T  x^{(K+1)}_{i,j} {(x^{(K+1)}_{i,S})}^T [\hat{\theta}_S - \bar{\theta}_S]\},$ 
$$|R_j^{(K+1)}| = \frac{1}{n^{(K+1)}} \left|\sum^{n^{(K+1)}}_{i=1} a_i^{(K+1)}b^{(K+1)}_i \right| \leq  \frac{1}{n^{(K+1)}} \|a^{(K+1)}\|_{\infty} \|b^{(K+1)}\|_1.$$
We have $\|a^{(K+1)}\|_{\infty} \leq 1,$ and 
\begin{align*}
    \frac{1}{n^{(K+1)}}\|b^{(K+1)}\|_1 & =  \mu_j {[\hat{\theta}^{(K+1)}_S - \bar{\theta}^{(K+1)}_S]}^T \left\{ \frac{1}{n^{(K+1)}} \sum^{n^{(K+1)}}_{i=1} x^{(K+1)}_{i,S} {(x^{(K+1)}_{i,S})}^T \right\} [\hat{\theta}^{(K+1)}_S - \bar{\theta}^{(K+1)}_S]
    \\ & =  \mu_j {[\hat{\theta}^{(K+1)}_S - \bar{\theta}^{(K+1)}_S]}^T \left\{ \frac{1}{n} \sum^n_{i=1} x^{(K+1)}_{i,S} {(x^{(K+1)}_{i,S})}^T \right\} [\hat{\theta}^{(K+1)}_S - \bar{\theta}^{(K+1)}_S]
    \\& \leq D^{(K+1)}_{\max} \|\hat{\theta}^{(K+1)}_S - \bar{\theta}^{(K+1)}_S\|_2^2
    \\& = D^{(K+1)}_{\max} \|\hat{\theta}^{(K+1)}_{S^{(K+1)}} - \bar{\theta}^{(K+1)}_{S^{(K+1)}}\|_2^2,
\end{align*}
where the last line uses the fact that $\hat{\theta}^{(K+1)}_{{[S^{(K+1)}]}^c} = \bar{\theta}^{(K+1)}_{{[S^{(K+1)}]}^c} = 0$
Therefore, we have $$\|R^{(K+1)}\|_{\infty} \leq D^{(K+1)}_{\max} \|\hat{\theta}^{(K+1)}_{S^{(K+1)}} - \bar{\theta}^{(K+1)}_{S^{(K+1)}}\|_2^2$$
\end{proof}
\subsubsection{Proof of Lemma \ref{2l2cons}}\label{2pfl2cons}
\begin{proof}
As in the proof for Lemma \ref{l2cons}, following the method of proof in \cite{ravikumar2010high} which was also previously used in another context
\cite{rothman2008sparse}, we define  the function $G^{(K+1)}: \mathbb{R}^{d} \rightarrow \mathbb{R}$ by
\begin{multline}
    G^{(K+1)}(u_{S^{(K+1)}}) := \ell^{(K+1)}(\bar{\theta}^{(K+1)}_{S^{(K+1)}} + u_{S^{(K+1)}}) \\ - \ell^{(K+1)}( \bar{\theta}^{(K+1)}_{S^{(K+1)}}) + \lambda^{(K+1)} (\|\bar{\theta}^{(K+1)}_{S^{(K+1)}} + u_{S^{(K+1)}}\|_1 - \|\bar{\theta}^{(K+1)}_{S^{(K+1)}}\|_1).
\end{multline}
It can be seen that $\hat{u}_{S^{(K+1)}} = \hat{\theta}^{(K+1)}_{S^{(K+1)}} - \bar{\theta}^{(K+1)}_{S^{(K+1)}}$ minimizes $G^{(K+1)}$. Moreover, $G^{(K+1)}(0) = 0$ by construction; therefore, we must have $G^{(K+1)}(\hat{u}_{S^{(K+1)}}) \leq 0$. Note also that $G^{(K+1)}$ is convex. Suppose that we show for some radius $B > 0$, and for $u \in \mathbb{R}^d$ with $\|u\|_2 = B,$ we have $G^{(K+1)}(u) > 0.$ We then claim that $\|\hat{u}\|_2 \leq B.$ Indeed, if $\hat{u}$ lay outside the ball of radius $B$, then the convex combination $t\hat{u}+(1-t)(0)$ would lie on the boundary of the ball, for an appropriately chosen $t \in (0,1).$ By convexity, $$G^{(K+1)}(t\hat{u}+(1-t)(0)) \leq t G^{(K+1)}(\hat{u}) + (1-t)G^{(K+1)}(0) \leq 0,$$ contradicting the assumed strict positivity of $G^{(K+1)}$ on the boundary. It thus suffices to establish strict positivity of $G^{(K+1)}$ on the boundary of the ball with radius $B=M\lambda^{(K+1)} \sqrt{d}$, where $M>0$ is a parameter to be chosen later in the proof. Let $u \in \mathbb{R}^d$ be an arbitrary vector with $\|u\|_2=B.$ Recalling the notation $W^{(K+1)} := -\nabla \ell^{(K+1)}(\bar{\theta}_S^{(K+1)}; \{\mathfrak{X}^{n^{(K+1)}}_{1,S}\}^{(K+1)}),$ by a Taylor series expansion of the log likelihood component of $G^{(K+1)}$, we have 
\begin{multline*}
    G(u) = -(W^{(K+1)}_{S^{(K+1)}})^T u + u^T[\nabla^2 \ell(\bar{\theta}^{(K+1)}_{S^{(K+1)}}+\alpha u_{S^{(K+1)}}; \{\mathfrak{X}^{n^{(K+1)}}_{1,S}\}^{(K+1)})]u \\ + \lambda^{(K+1)}(\|\bar{\theta}^{(K+1)}_{S^{(K+1)}}+u_{S^{(K+1)}}\|_1 - \|\bar{\theta}^{(K+1)}_{S^{(K+1)}}\|_1)
\end{multline*}

for some $\alpha \in [0,1].$ For the first term, we have the bound $$|(W^{(K+1)}_{S^{(K+1)}})^T u| \leq \|W^{(K+1)}_{S^{(K+1)}}\|_{\infty} \|u\|_1 \leq  \|W^{(K+1)}_{S^{(K+1)}}\|_{\infty} \sqrt{d} \|u\|_2 \leq (\lambda^{(K+1)} \sqrt{d})^2 \frac{M}{4},$$
since $\|W^{(K+1)}_{S^{(K+1)}}\|_{\infty} \leq \frac{\lambda^{(K+1)} }{4}$ by assumption.
For the last term, applying triangle inequality  yields $$\lambda^{(K+1)}(\|\bar{\theta}^{(K+1)}_{S^{(K+1)}}+u_{S^{(K+1)}}\|_1 - \|\bar{\theta}^{(K+1)}_{S^{(K+1)}}\|_1)\geq -\lambda^{(K+1)} \| u_{S^{(K+1)}}\|_1.$$
Since $\|u_{S^{(K+1)}}\|_1 \leq \sqrt{d} \|u_{S^{(K+1)}}\|_2, $ we have $$\lambda^{(K+1)}(\|\bar{\theta}^{(K+1)}_{S^{(K+1)}}+u_{S^{(K+1)}}\|_1 - \|\bar{\theta}^{(K+1)}_{S^{(K+1)}}\|_1)\geq -\lambda^{(K+1)}\sqrt{d} \| u_{S^{(K+1)}}\|_2 = - M{(\sqrt{d} \lambda^{(K+1)})}^2.$$
Finally, turning to the middle Hessian term, we have 
\begin{align*}
    q^* & := \Lambda_{\min}(\nabla^2 \ell(\bar{\theta}^{(K+1)}_{S^{(K+1)}}+\alpha^{(K+1)} u_{S^{(K+1)}}; \{\mathfrak{X}^{n^{(K+1)}}_{1,S}\}^{(K+1)})) 
    \\ & \geq \min_{\alpha^{(K+1)} \in [0,1]} \Lambda_{\min}(\nabla^2 \ell(\bar{\theta}^{(K+1)}_{S^{(K+1)}}+\alpha^{(K+1)} u_{S^{(K+1)}};\{\mathfrak{X}^{n^{(K+1)}}_{1,S}\}^{(K+1)}))
    \\ & = \min_{\alpha^{(K+1)} \in [0,1]} \Lambda_{\min} \left[  \frac{1}{n^{(K+1)}} \sum^{n^{(K+1)}}_{i=1} \eta (x^{(K+1)}_i ; \bar{\theta}^{(K+1)}_{S^{(K+1)}}+\alpha^{(K+1)} u_{S^{(K+1)}}) x^{(K+1)}_{i,S^{(K+1)}} {(x^{(K+1)}_{i,S^{(K+1)}})}^T \right].
\end{align*}

By a Taylor series expansion of $\eta(x^{(K+1)}_i; \cdot)$, we have, for some $\alpha_0 \in [0,\alpha^{(K+1)}]$,
\begin{align*}
   & q^*  \geq \min_{\alpha^{(K+1)} \in [0,1]} \Lambda_{\min} \left\{ \frac{1}{n^{(K+1)}} \sum^{n^{(K+1)}}_{i=1} \left[\eta (x^{(K+1)}_i ; \bar{\theta}^{(K+1)}_{S^{(K+1)}}) x^{(K+1)}_{i,S^{(K+1)}} {(x^{(K+1)}_{i,S^{(K+1)}})}^T \right]\right\}
    \\& +  \alpha^{(K+1)} g'\left(x^{(K+1)}_{i,r} \sum_{t\in S^{(K+1)} \setminus r} (\bar{\theta}^{(K+1)}_{rt} + \alpha_0 u_{rt})x^{(K+1)}_{i,t} \right) x^{(K+1)}_{i,r} (u_{S^{(K+1)}}^T x^{(K+1)}_{i,{S^{(K+1)}}}) x^{(K+1)}_{i,{S^{(K+1)}}} {(x^{(K+1)}_{i,{S^{(K+1)}}})}^T 
    \\& \geq \Lambda_{\min} \Big[ \frac{1}{n^{(K+1)}} \sum^{n^{(K+1)}}_{i=1} \eta (x^{(K+1)}_i ; \bar{\theta}_{S^{(K+1)}}) x^{(K+1)}_{i,{S^{(K+1)}}} {(x^{(K+1)}_{i,{S^{(K+1)}}})}^T \Big] 
    \\& \begin{multlined}[t][10.5cm]
        +\min_{\alpha^{(K+1)} \in [0,1]} \alpha^{(K+1)} \Lambda_{\min} \Big[  \frac{1}{n^{(K+1)}} \sum^{n^{(K+1)}}_{i=1} g' \left(x^{(K+1)}_{i,r}{(\bar{\theta}^{(K+1)}_{S^{(K+1)}}+\alpha_0 u_{S^{(K+1)}})}^T x^{(K+1)}_{i,{S^{(K+1)}}} \right) \\ 
        \cr x^{(K+1)}_{i,r} (u_{S^{(K+1)}}^T x^{(K+1)}_{i,{S^{(K+1)}}}) x^{(K+1)}_{i,{S^{(K+1)}}} {(x^{(K+1)}_{i,{S^{(K+1)}}})}^T \Big]
    \end{multlined}
    \\& \geq \Lambda_{\min}(Q^{(K+1)}_{{S^{(K+1)}}{S^{(K+1)}}}) - \max_{\alpha^{(K+1)} \in [0, 1]} \\&  \vertiii{\frac{1}{{n^{(K+1)}}} \sum^{n^{(K+1)}}_{i=1} g' (x^{(K+1)}_{i,r}{(\bar{\theta}^{(K+1)}_{S^{(K+1)}}+\alpha_0 u_{S^{(K+1)}})}^T x^{(K+1)}_{i,{S^{(K+1)}}} )  (u_{S^{(K+1)}}^T x^{(K+1)}_{i,{S^{(K+1)}}}) x^{(K+1)}_{i,{S^{(K+1)}}} {(x^{(K+1)}_{i,{S^{(K+1)}}})}^T}_2
    \\& \geq C_{\min} -\max_{\alpha^{(K+1)} \in [0, 1]}
    \\  & \vertiii{ \frac{1}{{n^{(K+1)}}} \sum^{n^{(K+1)}}_{i=1} g' (x^{(K+1)}_{i,r}{(\bar{\theta}^{(K+1)}_{S^{(K+1)}}+\alpha_0 u_{S^{(K+1)}})}^T x^{(K+1)}_{i,{S^{(K+1)}}} )(\langle u_{S^{(K+1)}}, x^{(K+1)}_{i,{S^{(K+1)}}} \rangle) x^{(K+1)}_{i,{S^{(K+1)}}} {(x^{(K+1)}_{i,{S^{(K+1)}}})}^T}_2
\end{align*}
It remains to control the spectral norm of the matrix , denoted as $A(\alpha^{(K+1)}) $ here, for $\alpha^{(K+1)} \in [0,1]$. For any fixed $\alpha^{(K+1)} \in [0,1]$, and $y \in \mathbb{R^d}$ with $\|y\|_2 = 1,$ we have 
\begin{align*}
    \langle y, A(\alpha^{(K+1)}) y \rangle & =  \frac{1}{n^{(K+1)}} \sum^{n^{(K+1)}}_{i=1} g' \left(\bar{\theta}^{(K+1)}_{S^{(K+1)}}+\alpha_0 u_{S^{(K+1)}}\right) [\langle u_{S^{(K+1)}}, x^{(K+1)}_{i,{S^{(K+1)}}} \rangle] {[\langle x^{(K+1)}_{i,{S^{(K+1)}}}, y \rangle]}^2
    \\& \leq \frac{1}{n^{(K+1)}} \sum^{n^{(K+1)}}_{i=1} \left|g' \left(\bar{\theta}^{(K+1)}_{S^{(K+1)}}+\alpha_0 u_{S^{(K+1)}} \right)\right| |\langle u_{S^{(K+1)}}, x^{(K+1)}_{i,{S^{(K+1)}}} \rangle| {[\langle x^{(K+1)}_{i,{S^{(K+1)}}}, y \rangle]}^2.
\end{align*}
Note that $ \left|g' \left(\bar{\theta}^{(K+1)}_{S^{(K+1)}}+\alpha_0 u_{S^{(K+1)}} \right)\right| \leq 1,$ and $$|\langle u_{S^{(K+1)}}, x^{(K+1)}_{i,{S^{(K+1)}}}  \rangle| \leq \|u_{S^{(K+1)}}\|_1 \leq \sqrt{d} \|u_{S^{(K+1)}}\|_2 = M \lambda^{(K+1)} d.$$
Moreover, we have $$\frac{1}{n^{(K+1)}} \sum^{n^{(K+1)}}_{i=1}{(\langle x^{(K+1)}_{i,{S^{(K+1)}}}, y \rangle)}^2  \leq \vertiii{\frac{1}{n^{(K+1)}} \sum^{n^{(K+1)}}_{i=1} x^{(K+1)}_{i,{S^{(K+1)}}} (x^{(K+1)}_{i,{S^{(K+1)}}})^T}_2 \leq D^{(K+1)}_{\max}$$
by assumption. We then obtain $$\max_{\alpha^{(K+1)} \in [0,1]}\vertiii{A(\alpha^{(K+1)})}_2 \leq D^{(K+1)}_{\max} M \lambda^{(K+1)} d \leq C^{(K+1)}_{\min}/2,$$
assuming that $\lambda^{(K+1)} \leq \frac{C^{(K+1)}_{\min}}{2M D^{(K+1)}_{\max} d}.$

Under this condition, we have shown that $$q^*  := \Lambda_{\min}(\nabla^2 \ell(\bar{\theta}^{(K+1)}_{S^{(K+1)}}+\alpha^{(K+1)} u_{S^{(K+1)}})) \geq C^{(K+1)}_{\min}/2.$$
Finally, combining the three terms in $G^{(K+1)}(u),$ we conclude that $$G^{(K+1)}(u_{S^{(K+1)}}) \geq (\lambda^{(K+1)} \sqrt{d})^2 \left\{ -\frac{1}{4} M + \frac{C^{(K+1)}_{\min}}{2}M^2 - M\right\},$$
which is strictly positive for $M = 5/C^{(K+1)}_{\min}.$ So as long as $$\lambda^{(K+1)} \leq \frac{C^{(K+1)}_{\min}}{2M D^{(K+1)}_{\max} d} = \frac{{(C^{(K+1)}_{\min})}^2}{10 D^{(K+1)}_{\max}d },$$
we are guaranteed that $$\|\hat{u}_{S^{(K+1)}}\|_2 \leq M \lambda^{(K+1)} \sqrt{d} = \frac{5}{C^{(K+1)}} \lambda^{(K+1)} \sqrt{d}.$$
\end{proof}

\section{Proof of Lemmas Used in Proving Other Lemmas}
\subsection{Proof of Lemma \ref{lemma of lemmas}}\label{pflmlm}
\begin{proof}
By the definition of the $\ell_{\infty}$-matrix norm, and using $Z_{jl}$ defined in Section \ref{pfmineigen} we have
\begin{align*}
    \mathbb{P}[\vertiii{Q^N_{S^cS} - \bar{Q}_{S^cS}}_{\infty} \geq \delta] & = \mathbb{P} \big[ \max_{j\in S^c} \sum_{l \in S} |Z_{jl}| \geq \delta \big]
    \\& \leq p  \mathbb{P} \big[  \sum_{l \in S} |Z_{jl}| \geq \delta \big],
\end{align*}
where the final inequality uses a union bound and the fact that $|S^c| \leq p.$
\begin{align*}
    \mathbb{P} \big[  \sum_{k \in S} |Z_{jl}| \geq \delta \big] & \leq \mathbb{P} [  \exists k \in S ||Z_{jl}| \geq \delta/d ]
    \\& \leq d \mathbb{P} [ |Z_{jl}| \geq \delta/d ].
\end{align*}
We then obtain (\ref{eq:subeq1}) by setting $\varepsilon = \delta/d$ in the Hoeffding bound (\ref{QZ}):
\begin{align*}
    \mathbb{P}[\vertiii{Q^N_{S^cS} - \bar{Q}_{S^cS}}_{\infty} \geq \delta] &\leq pd  \mathbb{P} [ |Z_{jl}| \geq \delta/d ]
    \\& \leq 2 \exp \left( -\frac{\varepsilon^2nK}{32d^2} + \log(d) + \log(p) \right)
\end{align*}
Analogously, for (\ref{eq:subeq2}), we have
\begin{align*}
    \mathbb{P}[\vertiii{Q^N_{SS} - \bar{Q}_{SS}}_{\infty} \geq \delta] & = \mathbb{P} \big[ \max_{j\in S} \sum_{k \in S} |Z_{jl}| \geq \delta \big]
    \\& \leq d  \mathbb{P} \big[  \sum_{l \in S} |Z_{jl}| \geq \delta \big]
    \\& \leq d \mathbb{P} [  \exists l \in S ||Z_{jl}| \geq \delta/d ]
    \\& \leq d^2 \mathbb{P} [ |Z_{jl}| \geq \delta/d ]
    \\& \leq 2 \exp \left( -\frac{\varepsilon^2nK}{32d^2} + 2\log(d) \right).
\end{align*}
To prove (\ref{eq:subeq3}), we can write
\begin{align*}
    \vertiii{{(Q^N_{SS})}^{-1} - {(\bar{Q}_{SS})}^{-1}}_{\infty} & = \vertiii{ {(\bar{Q}_{SS})}^{-1}[\bar{Q}_{SS} - Q^N_{SS}]{(Q^N_{SS})}^{-1}}_{\infty} \\& \leq \sqrt{d} \vertiii{ {(\bar{Q}_{SS})}^{-1}[\bar{Q}_{SS} - Q^N_{SS}]{(Q^N_{SS})}^{-1}}_2
    \\ & \leq \sqrt{d} \vertiii{ {(\bar{Q}_{SS})}^{-1}}_2\vertiii{\bar{Q}_{SS} - Q^N_{SS}}_2\vertiii{{(Q^N_{SS})}^{-1}}_2
    \\ & \leq \frac{\sqrt{d}}{C_{\min}} \vertiii{\bar{Q}_{SS} - Q^N_{SS}}_2 \vertiii{{(Q^N_{SS})}^{-1}}_2
\end{align*}
Using the bound (\ref{usebd1}) in the proof of Lemma \ref{mineigen}, we get
$$\mathbb{P}[\vertiii{{(Q^N_{SS})}^{-1}}_2 \geq \frac{2}{C_{\min}} ] \leq 2 \exp \left(  -\frac{\delta^2nK}{32 d^2}+2 \log (d)\right),$$ and
$$\mathbb{P}[\vertiii{Q^N_{SS} - \bar{Q}_{SS}}_2 \geq \delta/\sqrt{d}] \leq 2 \exp \left( -\frac{\delta^2nK}{32 d^3}+2 \log (d)\right).$$
So finally we have $$\mathbb{P}\left( \vertiii{{(Q^N_{SS})}^{-1} - {(\bar{Q}_{SS})}^{-1}}_{\infty}  \geq \delta\right) \leq 4 \exp \left( -B_1 \frac{nK\delta^2}{d^3} + B_2 \log (d)\right),$$ where $B_1, B_2$ are some positive constants.
\end{proof}

\subsection{Proof of Lemma \ref{2lmlm}}\label{pf2lmlm}
\begin{proof}
By the definition of the $\ell_{\infty}$-matrix norm, and using the $Z^{(K+1)}_{jl}$ defined in Section \ref{2pfmineigen}, we have 
\begin{align*}
    \mathbb{P}[\vertiii{Q^{(K+1),S}_{{[S^{(K+1)}]}^cS^{(K+1)}} - \bar{Q}^{(K+1),S}_{{[S^{(K+1)}]}^cS^{(K+1)}}}_{\infty} \geq \delta] & = \mathbb{P} \big[ \max_{j\in ({[S^{(K+1)}]}^c\cap S)}  \sum_{k \in S^{(K+1)}} |Z^{(K+1)}_{jl}| \geq \delta \big]
    \\& \leq d  \mathbb{P} \big[  \sum_{l \in S^{(K+1)}} |Z^{(K+1)}_{jl}| \geq \delta \big],
\end{align*}
where the final inequality uses a union bound and the fact that $|({[S^{(K+1)}]}^c\cap S)| \leq d$.
\begin{align*}
    \mathbb{P} \big[  \sum_{l \in S^{(K+1)}} |Z^{(K+1)}_{jl}| \geq \delta \big] & \leq \mathbb{P} [  \exists k \in S^{(K+1)} ||Z^{(K+1)}_{jl}| \geq \delta/d]
    \\ & \leq \mathbb{P} [  \exists k \in |S^{(K+1)}| ||Z^{(K+1)}_{jl}| \geq \delta/d]
    \\& \leq |S^{(K+1)}| \mathbb{P} [ |Z^{(K+1)}_{jl}| \geq \delta/d ]
    \\& \leq d \mathbb{P} [ |Z^{(K+1)}_{jl}| \geq \delta/d ].
\end{align*}

We then obtain (\ref{2eq:subeq1}) by setting $\varepsilon = \delta/d$ in the Hoeffding's bound (\ref{2QZ}),
\begin{align*}
    \mathbb{P}[\vertiii{Q^{(K+1),S}_{{[S^{(K+1)}]}^cS^{(K+1)}} - \bar{Q}^{(K+1),S}_{{[S^{(K+1)}]}^cS^{(K+1)}}}_{\infty} \geq \delta] &\leq d^2  \mathbb{P} [ |Z^{(K+1)}_{jl}| \geq \delta/d ]
    \\& \leq 2 \exp \left( -\frac{\varepsilon^2n^{(K+1)}}{32{d}^2} +2 \log (d) \right).
\end{align*}

Analogously for (\ref{2eq:subeq2}), we have
\begin{align*}
    \mathbb{P}[\vertiii{Q^{(K+1)}_{S^{(K+1)}S^{(K+1)}} - \bar{Q}^{(K+1)}_{S^{(K+1)}S^{(K+1)}}}_{\infty} \geq \delta] & = \mathbb{P} \big[ \max_{j\in S^{(K+1)}} \sum_{k \in S^{(K+1)}} |Z^{(K+1)}_{jl}| \geq \delta \big]
    \\& \leq d  \mathbb{P} \big[  \sum_{k \in S^{(K+1)}} |Z^{(K+1)}_{jl}| \geq \delta \big]
    \\& \leq d \mathbb{P} [  \exists k \in S^{(K+1)} ||Z^{(K+1)}_{jl}| \geq \delta/d ]
    \\& \leq {d}^2 \mathbb{P} [ |Z^{(K+1)}_{jl}| \geq \delta/d ]
    \\& \leq 2 \exp \left( -\frac{\delta^2n^{(K+1)}}{32{d}^2} + 2\log({d}) \right).
\end{align*}
To prove (\ref{2eq:subeq3}), we have

\begin{align*}
& \vertiii{{(Q^{(K+1)}_{S^{(K+1)}S^{(K+1)}})}^{-1} - {(\bar{Q}^{(K+1)}_{S^{(K+1)}S^{(K+1)}})}^{-1}}_{\infty}
\\& =\vertiii{ {(\bar{Q}^{(K+1)}_{S^{(K+1)}S^{(K+1)}})}^{-1}[\bar{Q}^{(K+1)}_{S^{(K+1)}S^{(K+1)}} - Q^{(K+1)}_{S^{(K+1)}S^{(K+1)}}]{(Q^{(K+1)}_{S^{(K+1)}S^{(K+1)}})}^{-1}}_{\infty} \\ &
   \leq \sqrt{d} \vertiii{ {(\bar{Q}^{(K+1)}_{S^{(K+1)}S^{(K+1)}})}^{-1}[\bar{Q}^{(K+1)}_{S^{(K+1)}S^{(K+1)}} - Q^{(K+1)}_{S^{(K+1)}S^{(K+1)}}]{(Q^{(K+1)}_{S^{(K+1)}S^{(K+1)}})}^{-1}}_2
    \\ & \leq \frac{\sqrt{d}}{C^{(K+1)}_{\min}} \vertiii{\bar{Q}^{(K+1)}_{S^{(K+1)}S^{(K+1)}} - Q^{(K+1)}_{S^{(K+1)}S^{(K+1)}}}_2 \vertiii{{(Q^{(K+1)}_{S^{(K+1)}S^{(K+1)}})}^{-1}}_2,
\end{align*}

where the sub-multiplicative property $\vertiii{AB}_2 \leq \vertiii{A}_2\vertiii{B}_2$ for matrices $A, B$ is used for the last line, and Assumption \ref{novela1} is also applied.Then using the bound (\ref{usebd2}) in the proof of Lemma \ref{2mineigen}, we get
$$\mathbb{P}[\vertiii{{(Q^{(K+1)}_{S^{(K+1)}S^{(K+1)}})}^{-1}}_2 \geq \frac{2}{C^{(K+1)}_{\min}} ] \leq 2 \exp \left(  -\frac{\delta^2n^{(K+1)}}{32 {d}^2}+2 \log ({d})\right),$$
and 
$$\mathbb{P}[\vertiii{Q^{(K+1)}_{S^{(K+1)}S^{(K+1)}} - \bar{Q}^{(K+1)}_{S^{(K+1)}S^{(K+1)}}}_2 \geq \delta/\sqrt{{d}}] \leq 2 \exp \left( -\frac{\delta^2n^{(K+1)}}{32 {d}^3}+2\log(d)\right)$$
So we have $$\mathbb{P}\left( \vertiii{{(Q^{(K+1)}_{S^{(K+1)}S^{(K+1)}})}^{-1} - {(\bar{Q}^{(K+1)}_{S^{(K+1)}S^{(K+1)}})}^{-1}}_{\infty}  \geq \delta\right) \leq 4 \exp \left( -B_1\frac{n^{(K+1)}\delta^2}{d^3} + B_2 \log ({d})\right),$$ where $B_1, B_2$ are some positive constants.
\end{proof}
\end{document}